\documentclass[12pt,leqno,a4paper,article]{amsart}
\usepackage{amsmath, amsthm, verbatim, amsfonts, amssymb, color}
\usepackage{mathrsfs}
\usepackage{dsfont}
\usepackage[a4paper,margin=3cm]{geometry}

\usepackage{graphicx}

\usepackage{algorithm,algorithmic}
\usepackage{multirow}
\usepackage{float}

\usepackage{bm}
\usepackage{bbm}

\usepackage{mathtools}
\usepackage{mathrsfs}
\usepackage{bm}
\usepackage{marginnote}
\usepackage{extarrows}
\usepackage{xcolor}
\usepackage{comment}

\usepackage{enumitem}
\usepackage{hyperref}
\hypersetup{hypertex=true,
	colorlinks=true,
	linkcolor=blue,
	anchorcolor=blue,
	citecolor=green,
	pdftitle={Overleaf Example},
	pdfpagemode=FullScreen}
\makeatletter
\newcommand{\runum}[1]{\mathrm{\romannumeral #1}}
\newcommand{\Rmnum}[1]{\mathrm{\expandafter\@slowromancap\romannumeral #1@}}
\makeatother

\newcommand\zero{\mathbf{0}}

\newcommand\dd{ \mathop{}\!\mathrm{d} }
\newcommand\md{ \mathop{}\!\mathrm{D} }
\newcommand{\defas}{:=}

\newcommand{\operatorP}{\operatorname{P}}
\newcommand{\operatorQ}{\operatorname{Q}}

\newcommand{\abs}[1]{\left|#1\right|}
\newcommand*\norm[1]{\left\lVert#1\right\rVert}
\newcommand{\lipnorm}[1]{\left\Vert #1 \right\Vert_{\operatorname{Lip}}}
\newcommand{\opnorm}[1]{\left\Vert #1 \right\Vert_{\operatorname{op}}}
\newcommand{\infnorm}[1]{\left\Vert #1 \right\Vert_{ \infty}}

\newcommand{\hsnorm}[1]{\left\Vert #1 \right\Vert_{\operatorname{HS}}}
\newcommand{\inprod}[1]{\left\langle #1 \right\rangle}
\newcommand*\sca[1]{\left\langle#1\right\rangle}

\DeclareMathOperator{\Lip}{Lip}
\DeclareMathOperator{\tr}{tr}

\newcommand{\bbR}{\mathbb{R}}
\newcommand{\bbP}{\mathbb{P}}
\newcommand{\bbE}{\mathbb{E}}
\newcommand{\bbN}{\mathbb{N}}

\newcommand{\bfb}{\mathbf{b}}
\newcommand{\bfg}{\mathbf{g}}
\newcommand{\bfh}{\mathbf{h}}
\newcommand{\bfm}{\mathbf{m}}
\newcommand{\bfu}{\mathbf{u}}
\newcommand{\bfv}{\mathbf{v}}
\newcommand{\bfw}{\mathbf{w}}
\newcommand{\bfx}{\mathbf{x}}
\newcommand{\bfy}{\mathbf{y}}
\newcommand{\bfz}{\mathbf{z}}

\newcommand{\bfA}{\mathbf{A}}
\newcommand{\bfB}{\mathbf{B}}
\newcommand{\bfC}{\mathbf{C}}
\newcommand{\bfD}{\mathbf{D}}
\newcommand{\bfE}{\mathbf{E}}
\newcommand{\bfF}{\mathbf{F}}
\newcommand{\bfG}{\mathbf{G}}
\newcommand{\bfH}{\mathbf{H}}
\newcommand{\bfL}{\mathbf{L}}
\newcommand{\bfN}{\mathbf{N}}
\newcommand{\bfR}{\mathbf{R}}
\newcommand{\bfS}{\mathbf{S}}
\newcommand{\bfU}{\mathbf{U}}

\newcommand{\bfW}{\mathbf{W}}
\newcommand{\bfY}{\mathbf{Y}}

\newcommand{\rme}{\mathrm{e}}

\newcommand{\conmomentum}{\mathbf{M}}
\newcommand{\conposition}{\mathbf{X}}

\newcommand{\dismomentum}{\widetilde{\mathbf{M}}}
\newcommand{\disposition}{\widetilde{\mathbf{X}}}

\newcommand{\hatmomentum}{\widehat{\mathbf{M}}}
\newcommand{\hatposition}{\widehat{\mathbf{X}}}

\newcommand{\BM}{\mathbf{B}}

\newcommand{\pheq}{\mathrel{\phantom{=}}}

\theoremstyle{plain}
\newtheorem{theorem}{Theorem}
\newtheorem{corollary}[theorem]{Corollary}
\newtheorem{lemma}{Lemma}[section]

\newtheorem{assumption}{Assumption}

\theoremstyle{definition}
\newtheorem{remark}[theorem]{Remark}

\numberwithin{equation}{section}
\renewcommand\labelenumi{\textup{\alph{enumi})}}
\renewcommand\theenumi\labelenumi

\usepackage{url}

\usepackage{marginnote}

\newcommand\Keywords[1]{\textbf{Keywords}: #1}

\begin{document}
	
\title[]{Error estimates between SGD with momentum and underdamped Langevin diffusion \footnote{equal contributions}} 
	
\allowdisplaybreaks[4]
	
	\author[A. Guillin]{Arnaud Guillin}
	\address[A. Guillin]{Laboratoire de Mathematiques Blaise Pascal,
CNRS-UMR 6620, Universite Clermont-Auvergne, avenue des landais, 63177, Aubiere, cedex}
	\email{arnaud.guillin@uca.fr}
	
	\author[Y. Wang]{Yu Wang}
	\address[Y. Wang]{1. Department of Mathematics, Faculty of Science and Technology, University of Macau, Taipa, Macau, 999078, China; 2. UM Zhuhai Research Institute, Zhuhai, Guangdong, 519000, China.}
	\email{yc17447@um.edu.mo}
	
	\author[L.~Xu]{Lihu Xu}
	\address[L.~Xu]{Department of Mathematics, Faculty of Science and Technology, University of Macau, Macau S.A.R., China. }
	\email{lihuxu@um.edu.mo}
	
	\author[H.~Yang]{Haoran Yang} 	
	\address[H.~Yang]{1. School of Mathematical Sciences, Peking University, Beijing, China. 2. Beijing International Center for Mathematical Research (BICMR), Peking University, Beijing, China.  }
	\email{yanghr@pku.edu.cn}

	\maketitle

	\begin{abstract}
		Stochastic gradient descent with momentum is a popular variant of stochastic gradient descent, which has recently been reported to have a close relationship with the underdamped Langevin diffusion. In this paper, we establish a {quantitative } error estimate between them in the $1$-Wasserstein and total variation distances.
		
		\noindent \Keywords{Stochastic Gradient Descent with momentum (SGDm), Underdamped Langevin Diffusion, $1$-Wasserstein Distance, Total Variation Distance, Variant Time Step, Malliavin Calculus} 
		
		\noindent {\bf MSC2020 subject classification}:60J20, 60H35, 60H30, 60F99. 
		
	\end{abstract}
		
\tableofcontents
	
\noindent
	
\section{Introduction}
\label{sec:sec1}
	
	Many tasks in machine learning and statistics can be formulated as an optimization problem as follows
	\begin{equation}\label{eq:ini-min}
		\text{ minimize} \quad \mathbb E_{\xi}  F({\bf x}, \xi),
	\end{equation}
	where $\xi$ is random and its distribution is not known, and ${\bf x} \in \mathbb{R}^d$. In practice, one needs to replace the mean $\mathbb E_{\xi} F({\bf x}, \xi)$ with its sample mean as 
	\begin{equation} \label{eq:emp-min}
		\text{ minimize} \quad \frac 1N \sum_{i=1}^N F({\bf x}, \xi^i),
	\end{equation}
	where $\xi^1,...,\xi^N$ are i.i.d.\ and have the same distribution as $\xi$. For the further use, we denote 
	\begin{equation}
		\label{eq:t-func}
		f(\bfx)=\mathbb E_{\xi} F(\bf x, \xi).
	\end{equation}
	
	To solve the minimization problem \eqref{eq:emp-min}, one often uses online stochastic gradient descent (SGD)  or online SGD with momentum (SGDm), and the latter usually converges faster or works more efficient than the former. The relationship between SGD and stochastic differential equation (SDE) has been intensively studied recently,  see e.g. \cite{li2017stochastic,li2019stochastic,chen2023probability,chen2020statistical,fontaine2021convergence}. Although there have been papers reporting the connections between stochastic gradient descent with momentum (SGDm) and underdamped Langevin diffusion, see for instance \cite{gao2022global}, their relationship has not been well understood, particularly when the time tends to infinity. The primary goal of this paper is to quantify their error bound in the 1-Wasserstein and total variation distances.  
	
	The noised online SGDm has two variables $(\bfm_k, \bfx_k)$,  called the moment and position respectively, satisfying
	\begin{equation}
		\label{sgd_m}
		\begin{dcases}
			\bfm_{k+1} = \bfm_k - \gamma \eta_{k+1} \bfm_k - \frac{\eta_{k+1}}{N} \sum_{i = 1}^N \nabla F(\bfx_k, \xi_{k+1}^i) + \beta \sqrt{\eta_{k+1}} \bm{\zeta}_{k+1},  \ & k \geqslant 0, \\
			\bfx_{k+1} = \bfx_k + \eta_{k+1} \bfm_k, \  & k \geqslant 0,
		\end{dcases}		
	\end{equation}
	with initial value $ (\bfm_{0},\bfx_{0}) \in \bbR^{2d} $, where the constant $\gamma > 0$ is the friction-coefficient, $\beta>0$ can be regarded as the temperature, 
	$\{ \eta_k, {k \geqslant 1} \}$ denote the step size, 
	 $\{ \xi_k^i, k \geqslant 1, \, 1 \leqslant i \leqslant N \}$ are i.i.d.\ copies of $\xi$, and $\{ \bm{\zeta}_{k}, {k \geqslant 1} \}$ are i.i.d.\ with standard $d$-dimensional normal distribution. Note that for each $k\geqslant 1$, $\nabla^k F(\bfx, \xi)$ is the $k$-th order derivative of $F$ with respect to $\bfx$ throughout this paper. 
	
	 We will compare the algorithm \eqref{sgd_m} with the underdamped Langevin diffusion $(\conmomentum_t, \conposition_t)_{t\geqslant 0}$ with state space $\bbR^{2d} $
	\begin{equation}
		\label{SDE1}
		\left\{
			\begin{aligned}
				\dd \conmomentum_t &= -\gamma \conmomentum_t \dd t - \nabla f(\conposition_t) \dd t + \beta \dd \BM_t,  \\
				\dd \conposition_t &= \conmomentum_t \dd t,
			\end{aligned}	
		\right.	
	\end{equation}
	where $(\bfB_t)_{t \geqslant 0}$ is a $d$-dimensional standard Brownian motion. The two processes possess the same initial point $(\conmomentum_0,\conposition_0) = (\bfm_0,\bfx_0)$. 
	 In order to compare the algorithm \eqref{sgd_m} with SDE \eqref{SDE1}, we introduce the following intermediate stochastic system: for $t \in [t_k,t_{k+1})$ with $k \ge 0$, 
	\begin{equation}\label{SDE2}
		\left\{
			\begin{aligned}
				\dd \dismomentum_t &= -\gamma \dismomentum_{t_k} \dd t - \frac{1}{N} \sum_{i = 1}^N \nabla F (\disposition_{t_k}, \xi_{k+1}^i) \dd t + \beta \dd \BM_t,  \\
				\dd \disposition_{t} &= \dismomentum_{t_k} \dd t, \\
			\end{aligned}	
		\right.		
	\end{equation}
	 where $t_k = \sum_{j=1}^k \eta_j$ and $t_0 = 0$. 
	It is obvious that $\{( \bfm_k, \bfx_k), {k \geqslant 0}\}$ and $\{(\dismomentum_{t_k}, \disposition_{t_k}), {k \geqslant 0}\}$ have the same distribution as long as $(\dismomentum_{t_0}, \disposition_{t_0})=( \bfm_0, \bfx_0)$.


\subsection{Literature Review}
Comparisons between stochastic algorithms and stochastic continuous dynamics have been intensively studied recently, e.g.\ \cite{gelfand1991recursive, pmlr-v65-raginsky17a, chen2023probability}. In particular, 
 the online stochastic gradient Langevin descent (SGLD) related to the minimization problem \eqref{eq:emp-min} reads as 
\begin{equation*}
	\bfx_{k+1} = \bfx_k - \eta \nabla F(\bfx_k, \xi) +  \sqrt{ 2 \beta^{-1} \eta} \bm{ \zeta_{k+1} }\ , \quad k \geqslant 0, 
\end{equation*}
with a constant learning rate $\eta>0$. 
 \cite{pmlr-v65-raginsky17a} studied the connection between this SGLD and the overdamped Langevin diffusion process $\conposition = (\conposition_t)_{t\geqslant0}$
\begin{equation*}
	\dd \conposition_t = -\nabla f(\conposition_t) \dd t + \sqrt{2 \beta^{-1} } \dd \bfB_t,
\end{equation*}
see also \cite{MR3568041,eberle2011reflection,chiang1987diffusion,MR3288096} for related distribution sampling problems.  Besides, \cite{pmlr-v65-raginsky17a}  showed that the $2$-Wasserstein distance between the SGLD with constant step size $\eta$ and the continuous time Langevin diffusion can be bounded by $\mathcal{O}(k\eta(\delta^{1/4}+\eta^{1/4}))$ at the $k$-th iteration. It is obvious that as $k$ has an order higher than $[{\eta (\delta^{1/4}+\eta^{1/4})}]^{-1}$, this bound is not useful. By the Lindeberg technique, \cite{chen2023probability} obtained a uniform bound  $\mathcal{O}((1+\log\abs{\eta})\eta)$ in 1-Wasserstein distance with respect to the time. 
Furthermore, \cite{pages2023unadjusted} studied the convergence of the Euler-Maruyama (EM) scheme for the SDEs driven by multiplicative noises in the 1-Wasserstein and total variation distances, the steps of the EM scheme are decreasing.

The underdamped Langevin diffusion \eqref{SDE1} has been extensively studied in recent years, see for instance \cite{Andreas2019,MR3288096,MR2731396,MR3568348,wu2001large,MR4757507,MR2858447,pmlr-v75-cheng18a}. An important advantage of the underdamped diffusion is that it often converges to the stationary faster than the overdamped diffusion, see for instance \cite{Andreas2019, CLW23,AAMN24} either in Wasserstein distance or in $L^2$ when the spectral gap is small by appropriately choosing friction coefficient.   \cite{wu2001large} showed that the system converges to its equilibrium measure with exponential rate by constructing appropriate Lyapunov test functions. \cite{Andreas2019, MR4757507} established the contraction property in the Wasserstein type distances for two solutions of \eqref{SDE1} with different initial data, and \cite{schuh2024convergence} studied the EM scheme of \eqref{SDE1} and its convergence in the 1-Wasserstein distance. 
\cite{cheng2018sharp} considered the problem of sampling from distribution $p^*(\bfx) \propto \rme^{-f(\bfx)}$ where $f: \bbR^d \to \bbR$ is $L$-smooth everywhere and $m$-strongly convex outside a ball of finite radius $R$. Given a tolerance error $\varepsilon$, they proved that the complexity of the underdamped Langevin Markov chain Monte Carlo (MCMC) is $\mathcal{O}(\sqrt{d}/\varepsilon)$ if the step size is $\eta = \mathcal{O}(\varepsilon/\sqrt{d})$. What's more, \cite{chau2022stochastic,gao2022global} considered the stochastic gradient Hamilton Monte Carlo (SGHMC) for sampling a similar distribution and found a complexity bound $\mathcal{O}( (k\eta)^{3/2} \sqrt{\log k\eta}(\delta^{1/4}+\eta^{1/4}) + k\eta\sqrt{\eta})$ in the 2-Wasserstein distance when the step size is a constant $\eta$.  
For other applications related to underdamped Langevin diffusion, we refer the reader to \cite{chen2015convergence,chau2022stochastic,gao2022global} for stochastic gradient Hamilton MC and to \cite{pmlr-v75-cheng18a,cheng2018sharp} for underdamped Langevin MCMC.

For other SGD algorithms such as Mini-batch SGD and Nesterov SGD, we refer the reader to   \cite{konevcny2015mini,li2014efficient, chee2020understanding,loizou2020momentum,botev2017nesterov,muehlebach2019dynamical,kim2016optimized} and the references therein. These algorithms differ in their approaches to updating the parameters during training neural networks and can be useful in improving convergence and reducing overfitting.

\subsection{Contributions and Methods}
Our contributions and methods are summarised as the below: 

	(1) Although there have been several papers studying the connections between the accelerated algorithms with continuous dynamics, see for instance \cite{chau2022stochastic,gao2022global}, but most of them only provide qualitative limit without a quantitative error bound, particularly when the time is large. To the best of our knowledge, this paper first provides error bounds between the SGDm and the underdamped Langevin diffusion uniformly with respect to the time, these bounds depend on the dimension $d$ polynomially and reveal convergence rates. Our results can be formally formulated as the following: 
	{ \[
		d_{\mathcal{W}_1}( \mathcal{L}(\conmomentum_{t_n}, \conposition_{t_n}) , \mathcal{L}(\bfm_n , \bfx_n)) \leqslant C \sqrt{d}(1+1/N)  \sqrt{\eta_n} ,
	\]
	\[
		d_{\mathrm{TV}}( \mathcal{L}(\conmomentum_{t_n}, \conposition_{t_n}) , \mathcal{L}(\bfm_n , \bfx_n)) \leqslant C d^{4}  (\sqrt{\eta_n} + 1 / \sqrt{ N } ).
	\]}
We clearly see a significant difference between the above two bounds, one being of an order $O(\sqrt{\eta_n}+{\sqrt{\eta_n}}/{N})$ in the 1-Wasserstein distance and the other $O(\sqrt{\eta_n}+{1}/{\sqrt{N}})$ in the total variation distance. We think that the rate $O({1}/{\sqrt{N}})$ in the total variation distance is due to the singularity of this distance and that it is hard to improve the $O({1}/{\sqrt{N}})$ without paying a price on the rate $\sqrt{\eta_n}$ . 

To the best of our knowledge, most of the known results related to underdamped Langevin diffusion samplings and discretization schemes have a constant step size $\eta$, while our SGDm has the non-increasing step sizes $\eta_k$ which include the constant step size case.

	(2) There are two systems of noises, $\xi^i_k$ and $\bm{\zeta}_{k+1}$, in SGDm \eqref{sgd_m}, whose  interplays make the SGDm much more complex than the underdamped Langevin sampling by discretising SDE \eqref{SDE1}.  
These complexities can be clearly seen in bounding the total variation distance, where we need to estimate a Malliavin matrix related to $\bm{\zeta}_{k+1}$ by splitting the regimes of $\xi^i_k$. What's more, $\xi^i_k$ can be heavy tailed. More precisely, we only need to assume that $\xi^i_k$ has second and fourth moments to bound $d_{\mathcal{W}_1}( \mathcal{L}(\conmomentum_{t_n}, \conposition_{t_n})$ and $d_{\mathrm{TV}}( \mathcal{L}(\conmomentum_{t_n}, \conposition_{t_n}) , \mathcal{L}(\bfm_n , \bfx_n))$ respectively.

(3) The term  ${N}^{-1} \sum_{i = 1}^N \nabla F (x, \xi_{k+1}^i)$ is an unbiased estimator of $\nabla f(x)$. According to the law of large number, it converges to
	$\nabla f(x)$ almost surely as $N \rightarrow \infty$. This, together with the exponential ergodicity of $(\conmomentum_{t}, \conposition_{t})$, immediately implies an error bound in the order $O(\sqrt{\eta_n})$ for the underdamped Langevin sampling as $n$ is large. The rate $O(\sqrt{\eta_n})$ may be not optimal due to the heavy tail effect of $\xi^i_k$, we conjecture that the rate $O(\sqrt{\eta_n})$ can be improved by assuming that 
$\xi^i_k$ has a high order moment. We leave this possible improvement to the future research.   
	
	(4) Although the approach of our proofs is still via the classical Lindeberg principle, in contrast to the overdamped Langevin diffusion, the underdamped Langevin diffusion is degenerate and the following difficulties naturally arise: (i) the regularity problems are very involved and we need Malliavin calculus to handle them; (ii) there are interplays between the randomnesses of $\xi^i_k$ and $\bm{\zeta}_{k+1}$, it is much more difficult for us to estimate the Malliavin matrix and figure out the polynomial dependence on the dimension $d$; (iii) the related Lyapunov function is much more subtle and not intuitive.

\subsection{Structure of The Paper} 
 	In the next section, we will introduce our assumptions and main theorems. Section \ref{sec:sec3} will provide all auxiliary lemmas, which include estimates of $p$-th moments, the contraction property of the $1$-Wasserstein distance, and additional estimates related to the total variation distance. The proofs for the main results and their corollaries will be provided in Section \ref{sec:pfs-Thm}.
	In Appendix \ref{Appendix:A}, we give the supporting lemmas related to the Lyapunov function. 
	At last, the lemmas in Section \ref{sec:sec3} will be proved in Appendix \ref{Appendix:B}.

{\bf Acknowledgements}: We would like to thank Professor Feng-Yu Wang for very helpful discussion and pointing out the references \cite{MR4757507} and \cite{schuh2024convergence} to us. A. Guillin is benefited from a government grant managed by the Agence Nationale de la Recherche under the France 2030 investment plan “ANR-23-EXMA-0001"n and under the grant ANR-23-CE40-0003. L. Xu is supported by the National Natural Science Foundation of China No. 12071499, the Science and Technology Development Fund (FDCT) of Macau S.A.R. FDCT 0074/2023/RIA2, and the University of Macau grants MYRG2020-00039-FST, MYRG-GRG2023-00088-FST.

\section{Preliminary and Main Results}
\label{sec:sec2}
	
	\subsection{Notations and Assumptions}	
	We use normal font for scalars (e.g.\ $a,A, \dots$) and boldface for vectors and matrices (e.g.\ $\bfx,\bfy, \bfA,\bfB, \dots$).
	For any vectors $\bfu = (u_1, \dotsc, u_d)$, $\bfv = (v_1, \dotsc, v_d)$ in $ \bbR^d$, their standard inner product is denoted by $\langle \bfu, \bfv \rangle = \sum_{i=1}^d u_i v_i$, and denote the corresponding Euclidean norm as $\abs{\bfu} = \big({\sum_{i=1}^d u_i^2} \big)^{1/2}$.   The Hilbert–Schmidt inner product for matrices $\bfA = (A_{ij})_{d\times d}, \bfB=(B_{ij})_{d\times d} \in \bbR^{d\times d}$ is denoted by $\langle \bfA, \bfB \rangle_{\mathrm{HS}} = \sum_{i,j=1}^d A_{ij}B_{ij}$, the  Hilbert–Schmidt norm is defined as $\hsnorm{\bfA} = ({\sum_{i,j=1}^d A_{ij}^2})^{1/2}$ .  Besides, the operator norm of $\bfA$ is denoted by $\opnorm{\bfA} = \sup_{\abs{\bfv}= 1}\abs{ \bfA \bfv }$, which has the following relationship with $\hsnorm{\bfA}$
	\[
		\opnorm{\bfA} \leqslant \hsnorm{\bfA} \leqslant \sqrt{d} \opnorm{\bfA}.
	\]
	If a matrix $\bfA$ is positive semi-definite, we define $\lambda_{\max}(\bfA)$ as its maximal eigenvalue and $\lambda_{\min}(\bfA)$ as its minimal eigenvalue. 

	
	Let $\mathcal{C}(\bbR^d, \bbR)$ denote the set of all continuous functions from $\bbR^d$ to $\bbR$ and $\mathcal{C}_b(\bbR^d, \bbR)$ denote the set of all bounded continuous functions from $\bbR^d$ to $\bbR$. For $k \ge 0$, denote by $\mathcal{C}^k(\bbR^d, \bbR)$ the set of functions  from $\bbR^d$ to $\bbR$ which has continuous $0$-th,...,$k$-th order derivatives, further denote by $\mathcal{C}_b^k(\bbR^d, \bbR)$ the set of functions  from $\bbR^d$ to $\bbR$ which has bounded continuous $0$-th,...,$k$-th order derivatives and by $\mathcal{C}_p^k(\bbR^d, \bbR)$ the set of functions from $\bbR^d$ to $\bbR$ whose $0$-th,...,$k$-th order derivatives have polynomial growth. For $g \in \mathcal{C}^3(\bbR^d, \bbR)$ and $\bfv, \bfv_1, \bfv_2, \bfv_3, \bfx \in \bbR^d$, we denote that 
	\begin{align*}
		\nabla_{\bfv} g(\bfx) &= \lim_{\varepsilon \to 0} \frac{ g(\bfx+\varepsilon \bfv) - g(\bfx) }{ \varepsilon }, \\
		\nabla_{\bfv_2}\nabla_{\bfv_1} g(\bfx) &= \lim_{\varepsilon \to 0} \frac{ \nabla_{\bfv_1}g(\bfx+\varepsilon \bfv_2) - \nabla_{\bfv_1}g(\bfx) }{ \varepsilon }, \\
		\nabla_{\bfv_3}\nabla_{\bfv_2}\nabla_{\bfv_1} g(\bfx) &= \lim_{\varepsilon \to 0} \frac{ \nabla_{\bfv_2}\nabla_{\bfv_1}g(\bfx+\varepsilon \bfv_3) - \nabla_{\bfv_2}\nabla_{\bfv_1}g(\bfx) }{ \varepsilon },
 	\end{align*} 
 	as the directional derivatives of $g$. We know $\nabla g(\bfx) \in \bbR^d$, $\nabla^2 g(\bfx) \in \bbR^{d \times d}$, $\nabla^3 g(\bfx) \in \bbR^{d\times d \times d}$.  Moreover, we define the operator norm of $\nabla^k g (\bfx)$, $k = 2, 3$ with
 	\begin{align*}
 		\opnorm{ \nabla^k g(\bfx) } = \sup_{\abs{\bfv_i} = 1, \, i = 1, \dots, k} \abs{ \nabla_{\bfv_k} \dots \nabla_{\bfv_1} g(\bfx) }.
 	\end{align*}
 	If $g$ is bounded, we denote its infinity norm by $\infnorm{g} = \sup_{\bfx}\abs{g(\bfx)}$. 	
	If $g$ is a Lipschitz function, its Lipschitz norm is denoted by $\lipnorm{ g } = \sup_{\bfx \neq \bfy} \abs{g(\bfx)-g(\bfy)} / \abs{\bfx-\bfy} $. Additionally, for a function $h \colon (\bfm, \bfx) \in \bbR^{2d} \mapsto h(\bfm, \bfx) \in \bbR$, we denote its derivatives with respect to $\bfm$ and $\bfx$ as $\nabla_{\bfm} h$ and $\nabla_{\bfx} h$ respectively. 
	Whenever we need to emphasize the initial value $(\conmomentum_0, \conposition_0) = (\bfm, \bfx)$ for given $(\bfm, \bfx) \in \bbR^{2d}$, we will write $(\conmomentum_t^{\bfm}, \conposition_t^{\bfx})$. Similarly, we use the notation $(\dismomentum_t^{\bfm}, \disposition_t^{\bfx})$. 
	The operator semigroup induced by $(\conmomentum_t^{\bfm}, \conposition_t^{\bfx})_{t\geqslant0}$ from \eqref{SDE1} is given by
	\[
		\operatorP_t h(\bfm, \bfx) = \bbE h(\conmomentum_t^{\bfm}, \conposition_t^{\bfx}), \quad h \in \mathcal{C}_b(\bbR^{2d}, \bbR), \ t>0.
	\]	
	Its infinitesimal generator $\mathscr{A}$ is defined as
	\begin{equation*}
		\mathscr{A} h(\bfm , \bfx) := -\inprod{ \nabla_{\bfm} h, \gamma \bfm + \nabla f(\bfx) } + \inprod{ \nabla_{\bfx} h , \bfm} + \frac12 \beta^2\  \Delta_{\bfm} h,
	\end{equation*}
	for any $h \in \mathcal{C}_p^2( \bbR^{2d}, \bbR)$. The exact form for the domain of $\mathscr{A}$ is not necessarily figured out in this paper.
	
	 Let $\mathscr{P} (\bbR^{2d})$ denote the space of probability distributions on $\bbR^{2d}$. For any $\mu, \nu \in \mathscr{P} (\bbR^{2d})$, denote their 1-Wasserstein distance by
	\begin{equation*}
		d_{\mathcal{W}_1} (\mu, \nu) = \inf_{\pi \in \Pi (\mu, \nu)} \int_{\bbR^{2d} \times \bbR^{2d}} \abs{\bfz^1 - \bfz^2} \dd \pi (\bfz^1, \bfz^2), \quad \mu, \nu \in \mathscr{P} (\bbR^{2d}),
	\end{equation*}
	where $\Pi (\mu, \nu)$ denotes the set of probability distributions on $\bbR^{2d} \times \bbR^{2d}$ with marginal distributions $\mu$ and $\nu$. Kantorovich-Rubinstein Theorem tells us that
	\begin{equation*}
		d_{\mathcal{W}_1} (\mu, \nu) = \sup_{\norm{h}_{\Lip} \leqslant  1} \int_{\bbR^{2d}} h (\bfz) \, \dd (\mu - \nu) (\bfz), \quad \mu, \nu \in \mathscr{P} (\bbR^{2d}).
	\end{equation*}
	 Besides, the total variation distance between $\mu$ and $\nu$ can be defined by
	\begin{equation*}
		d_{\mathrm{TV}}( \mu , \nu ) =  \sup_{\infnorm{g}\leqslant 1} \int_{\bbR^{2d}} g(\bfz) \, \dd (\mu-\nu)(\bfz), \quad \mu, \nu \in \mathscr{P} (\bbR^{2d}).
	\end{equation*}

\vskip 2mm	
	
Throughout this paper, we propose the following assumptions.  	
	\begin{assumption}
		\label{Assump.1}
		The function $f(\bfx)$, defined by \eqref{eq:t-func}, is non-negative and satisfies the following conditions:
		\begin{itemize}
			\item[$(\mathrm{\runum{1}})$] There exist constants $A, B \geqslant 0$ such that
			\begin{equation}
				\label{A1-1}
				\abs{f(\zero)} \leqslant A, \quad \abs{\nabla f(\zero)} \leqslant  B.
			\end{equation}
			\item[$(\runum{2})$] $f$ is $L$-smooth for some constant $L>0$, that is,
			\begin{equation}
				\label{A1-2}
				\abs{\nabla f(\bfx) - \nabla f(\bfy)} \leqslant L \abs{\bfx-\bfy}, \quad \forall\  \bfx,\bfy \in \bbR^d.
			\end{equation}
			\item[$(\runum{3})$] For any $\bfx, \bfy \in \bbR^d$, there exist positive constants $a$ and $b$ such that
			\begin{equation}
				\label{A1-3}
				\inprod{\bfx - \bfy, \nabla f(\bfx) - \nabla f(\bfy)} \geqslant a \abs{\bfx - \bfy}^2 -  b .
			\end{equation} 
		\end{itemize}
	\end{assumption}
	\begin{assumption}
		\label{Assump.2}
		Let $\mathcal{U} $ be a measurable space. $\xi$ is a $\mathcal{U}$-valued random variable such that $\nabla F(\bfx, \xi) $  satisfies the following conditions:
		\begin{itemize}
			\item[$(\runum{1})$] For any $\bfx \in \bbR^d$, 
			\begin{equation*}
				\bbE_{\xi} [ \nabla F(\bfx, \xi) ] = \nabla f(\bfx).
			\end{equation*}
			\item[$(\runum{2})$] There exists a constant $A_0> 0$ such that for some $q\geqslant 2$,
			\begin{equation}
				\label{A2-2}
				\sup_{\bfx \in \bbR^d} \big\{ \bbE_{\xi} [ \abs{\nabla f(\bfx) - \nabla F(\bfx, \xi)}^{q} ] \big\}^{ 1/q }\leqslant A_0 .
			\end{equation}
		\end{itemize}
		Here, $F(\bfx, \xi)$ and $f(\bfx)$ are in \eqref{eq:ini-min} and \eqref{eq:t-func} respectively. 
	\end{assumption}

	In practical applications, the choice of step size often varies with each iteration $k$. It is necessary and important to add an additional assumption as following to restrict the behavior of $\eta_k$.
	
	\begin{assumption}
		\label{Assump.3}
		Let $(\eta_k)_{k \geqslant 1}$ be a non-increasing and positive sequence and let $t_n=\sum_{k=1}^n \eta_k$, which satisfy the following conditions:
		\begin{itemize}
			\item[$(\runum{1})$] 
			$\lim_{n \rightarrow \infty} t_n = + \infty$;
			\item[$(\runum{2})$] There exists a constant $\omega \in (0, 2\theta )$ such that
			\begin{equation} \label{A3-2}
				{\eta_k \leqslant \frac{2\theta - \omega}{2 \theta^2} } \quad \text{ and } \quad
				\eta_{k-1} - \eta_{k} \leqslant \omega \eta_{k}^2, \quad \forall \ k \geqslant 1,
			\end{equation}	
			where the constant $\theta$ will be given in Lemma \ref{lem3-6}.					
		\end{itemize}
		
	\end{assumption}
	A typical example is $\eta_k = \eta/k^{\alpha}$ for some constants $\eta>0$ and $\alpha \in (0,1)$, then condition \eqref{A3-2} holds for sufficiently large $k$.
	
\subsection{Main Results}
	
	We are now in the position to state our main results, whose proofs will be given in Section \ref{sec:pfs-Thm}. Define the Lyapunov function $\mathcal{V}(\bfm, \bfx): \bbR^{2d}  \to \bbR$ of SDE \eqref{SDE1} as
	\begin{equation} \label{def-Lya.}
		\mathcal{V}(\bfm, \bfx) := f(\bfx) + \frac{\gamma^2}{4} \left( \Big|\bfx+\frac1{\gamma} \bfm \Big|^2 + \Big| \frac1{\gamma}\bfm \Big|^2 - \lambda\abs{\bfx}^2 \right),
	\end{equation}
 where constant $0 < \lambda \leqslant \frac14 \land \frac{a}{4L + \gamma^2}$. More details about this Lyapunov function can be found in Appendix \ref{Appendix:A}. 
  
	 We denote by $\mathcal{L}(\conmomentum_{t_n}, \conposition_{t_n})$ and $ \mathcal{L}(\bfm_n, \bfx_n)$ the laws of $(\conmomentum_{t_n}, \conposition_{t_n})$ from \eqref{SDE1} and $(\bfm_n, \bfx_n)$ from \eqref{sgd_m} respectively.

Throughout this paper, we shall use the letters $c, C, C_{1}, C_2, C_3$ to denote positive numbers which may depend on the parameters $A_0, A, B, L, a, b, q, \theta, \omega$ in Assumptions \ref{Assump.1}, \ref{Assump.2}, and \ref{Assump.3} above, and $A_0^{\prime}, q^{\prime}, B_1, B_2$ in Assumption \ref{Assump.4} below. Their values may vary from line to line, but do NOT depend on the dimension $d$ and the sample size $N$. In some cases, they may depend on other parameters such as $p$ and we will stress this in a way like '$C$ also depends on $p$' if necessary.  
	\begin{theorem}
		\label{Thm-1}
		Under Assumptions \ref{Assump.1}, \ref{Assump.2}, and \ref{Assump.3}, we assume that $\eta_1 \leqslant c$ for some positive constant $c$,  and $\gamma > \sqrt{2}(2L+a)/\sqrt{a}$ additionally. Then we have
		\begin{equation*}
			d_{\mathcal{W}_1}{ ( \mathcal{L}(\conmomentum_{t_n}, \conposition_{t_n}), \mathcal{L}(\bfm_n, \bfx_n) )}
			\leqslant C {  \sqrt{ \mathcal{V}(\bfm_0,\bfx_0) + d} \left(1+\frac1N\right) }\sqrt{\eta_n} , \quad  \forall n \geqslant 0 .
		\end{equation*}
	\end{theorem}

	To estimate the total variation distance, we rely on the following assumptions regarding the derivatives of the function $F(\bfx, \xi)$, which are stronger than Assumption \ref{Assump.2}.

	\begin{assumption}
		\label{Assump.4}
		Functions $F(\bfx, \xi)$ and $f(\bfx)$ are in \eqref{eq:ini-min} and \eqref{eq:t-func} respectively. We assume that
		\begin{itemize}
			\item[$(\runum{1})$] For any $\bfx \in \bbR^d$,
				\begin{equation*}
					\bbE_{\xi} \left[ \nabla F(\bfx, \xi) \right] = \nabla f(\bfx) 
					\quad \text{and} \quad 
					\bbE_{\xi} \left[ \nabla^2 F(\bfx, \xi) \right] = \nabla^2 f(\bfx) .
				\end{equation*}
			\item[$(\runum{2})$] There exists a constant $A_0^{\prime} > 0$ such that for some $q^{\prime} \geqslant 4$,
				\begin{equation} \label{A4-2}
					\sup_{\bfx \in \bbR^d} \Big\{ \bbE_{\xi} \big[ \abs{\nabla f(\bfx) - \nabla F(\bfx, \xi)}^{q^\prime} \big] \Big\}^{ 1/q^{\prime} }\leqslant A_0^{\prime} .
				\end{equation}
			\item[$(\runum{3})$] There exist some positive constants $B_1$ and $B_2$ such that
			\begin{gather*}
				\left\{ \bbE \left[ \sup_{ \bfx \in \bbR^d } \opnorm{ \nabla^2 F (\bfx, \xi) }^8 \right] \right\}^{{1}/{8}} \leqslant B_1 , \\
				\max \left\{  \sup_{ \bfx \in \bbR^d } \opnorm{\nabla^3 f (\bfx)},\  \sup_{ \bfx \in \bbR^d } \left( \bbE \opnorm{ \nabla^3 F (\bfx, \xi) }^4 \right)^{{1}/{4}} \right\} \leqslant B_2.
			\end{gather*}
		\end{itemize}
	\end{assumption}
	
	Note that \eqref{A4-2} covers the \eqref{A2-2}. We have the following result about total variation distance.
	
	\begin{theorem}
		\label{Thm-2}
		Under Assumptions \ref{Assump.1}, \ref{Assump.2}, \ref{Assump.3} and \ref{Assump.4}, we assume that $\eta_1 \leqslant c d^{-2}$ for some constant $c$, and $\gamma > \sqrt{2}(2L+a)/\sqrt{a}$ additionally. Then we have
			\begin{equation*}
				d_{\mathrm{TV}}{ ( \mathcal{L}(\conmomentum_{t_n}, \conposition_{t_n}), \mathcal{L}(\bfm_n, \bfx_n) )}
				\leqslant C {d^{7/2} \sqrt{\mathcal{V} (\bfm_0, \bfx_0) + d}  \left( \sqrt{\eta_n} + \frac{1}{\sqrt{N}} \right) }, \quad \forall \ n \geqslant 0.
			\end{equation*}
	\end{theorem} 
	
	Above two results yield the following corollaries for the case of $\eta_k = \eta / k^{\alpha}$  for some constant $\eta>0$ and $\alpha \in (0, 1)$.  
	
	\begin{corollary}
		\label{Cor-W1}
		Under the same conditions in Theorem \ref{Thm-1}, let $\eta$ be a (small) positive constant and $\eta_k = \eta/k^\alpha$ with $\alpha \in (0,1)$. Then, there exists a constant $C$ such that for all $n\geqslant0$
		\begin{align*}
			d_{\mathcal{W}_1}( {\mathcal{L}(\conmomentum_{t_n}, \conposition_{t_n}) , \mathcal{L}(\bfm_n, \bfx_n)} )
			\leqslant C {\sqrt{ \mathcal{V}(\bfm_0,\bfx_0) + d } \left(1+\frac1N\right) } \sqrt{\frac{\eta}{n^{\alpha}}} . 
		\end{align*}
	\end{corollary}
	
	\begin{corollary}
		\label{Cor-TV}
		Under the same conditions in Theorem \ref{Thm-2}, let $\eta$ be a (small) positive constant and $\eta_k = \eta/k^\alpha$ with $\alpha \in (0,1)$. Then, there exists a constant $C$ such that for all $n\geqslant0$
		\begin{align*}
			d_{\mathrm{TV}}( {\mathcal{L}(\conmomentum_{t_n}, \conposition_{t_n}) , \mathcal{L}(\bfm_n, \bfx_n)} )
			&\leqslant C {  d^{7/2}  \sqrt{\mathcal{V}(\bfm_0,\bfx_0) + d}  \left( \sqrt{\frac{\eta}{n^\alpha}} + \frac{1}{ \sqrt{N} } \right) }  . 
		\end{align*}
	\end{corollary}
	
	With regard to the original minimization problem, we have the following generalization error bound.
	\begin{corollary}
		\label{Cor-Error}
		Under the same conditions in Theorem \ref{Thm-1}, let $\bfx^* = \arg \min_{\bfx\in\bbR^d} f(\bfx)$ be the minimizer of $f(\bfx)$ in \eqref{eq:t-func}. It holds
		\begin{equation*}
			\bbE F (\bfx_n, \xi) - f(\bfx^*)
			\leqslant \left\{ C \left[ {  \left( 1+\frac1N \right)} \sqrt{\eta_n} \right]^{ \frac{q-2}{q-1}}  + B_e  \rme^{-\kappa t_n} \right\} \big( \mathcal{V}(\bfm_0, \bfx_0) +  d \big) + J  d,
		\end{equation*}
		where $B_e$ and $\kappa$ are the constants in \eqref{eq:ExpErg}, $q$ is in Assumption \ref{Assump.2},  and $J$ has the form of
		\begin{equation*}
			J = \frac{\beta^2}{4} \log\left[ \frac{2 \rme L}{a} \left( \frac{2ab + B^2}{d a \beta^2} + 1 \right) \right].
		\end{equation*}		
	\end{corollary}
	
	That is, the SGDm described by \eqref{sgd_m} is expected to approach the minimum point of the function $f(\mathbf{x})$ under appropriate parameters by choosing sufficiently large time $n$ and
	large sample size $N$,  and sufficiently small $\beta$. Moreover, the larger $q$ in Assumption \ref{Assump.2} is, the faster convergence we can obtain.
%
\begin{remark}
(i). Note that these results above hold for constant learning rate $\eta_n \equiv \eta$ for all $n \in \mathbb{N}$ as long as the $\eta$ is a sufficiently small number.

(ii). In the algorithm \eqref{sgd_m}, the term  ${N}^{-1} \sum_{i = 1}^N \nabla F (x, \xi_{k+1}^i)$ is an unbiased estimator of $\nabla f(x)$. According to the law of large number, it converges to
	$\nabla f(x)$ almost surely as $N \rightarrow \infty$. This, together with the exponential ergodicity of $(\conmomentum_{t}, \conposition_{t})$, immediately implies an error bound in the order $O(\sqrt{\eta_n})$ in the 1-Wasserstein and total variation distance for the underdamped Langevin sampling as $n$ is large.  The rate $O(\sqrt{\eta_n})$ may be not optimal due to the heavy tail effect of $\xi^i_k$, we conjecture that the rate $O(\sqrt{\eta_n})$ can be improved by assuming that 
$\xi^i_k$ has a high order moment. We leave this possible improvement to the future research.
\end{remark}
\section{Auxiliary Lemmas}
\label{sec:sec3}

	In this section, we list auxiliary lemmas which will be used to prove our main results, and their proofs will be given in Appendix \ref{Appendix:B}.
	
	Combining \eqref{A1-1} and \eqref{A1-3} implies the following dissipation property of $f$
	\begin{equation*}
		\inprod{\bfx, \nabla f(\bfx)} \geqslant \frac{a}{2} \abs{\bfx}^2 - K, \quad \forall\  \bfx \in \bbR^d ,
	\end{equation*}
	where $K = b + B^2 / (2a)$. By \eqref{A1-2}, the following linear growth condition holds
	\begin{equation}
		\label{eq:f-linear}
		\abs{ \nabla f(\bfx) } \leqslant L \abs{\bfx}+B, \quad \forall\  \bfx\in\bbR^d.
	\end{equation}	 	
	Recall the Lyapunov function \eqref{def-Lya.}.  
	Since $\lambda \leqslant 1 / 4$ and $f$ is nonnegative, it is easy to see 
	\begin{equation}
		\label{eq:bounds}
		0 \leqslant \max \left\{ \frac{1-2\lambda}{4(1-\lambda)} \abs{\bfm}^2\ ,\ \frac{\gamma^2}{8}(1-2\lambda) \abs{\bfx}^2 \right\} \leqslant \mathcal{V}(\bfm,\bfx), \quad \forall \ \bfm, \bfx \in \bbR^d.
	\end{equation}
	One can verify that
	\begin{equation}
		\label{eq:inf-Lya}
		\mathscr{A} \mathcal{V}(\bfm, \bfx) \leqslant -\lambda \gamma \mathcal{V}(\bfm, \bfx) + \frac12 \left( \gamma \mathring{A} +d \beta^2 \right),
	\end{equation}
	where constants $\lambda$ and $\mathring{A}$ satisfy
	\begin{equation*}
		0 < \lambda \leqslant \min \left\{\frac14,  \frac{a}{4L+\gamma^2} \right\}, \quad
		\mathring{A} \geqslant K+2\lambda \left( \frac{B^2}{2L}+A \right).  
	\end{equation*}
More details about the Lyapunov function can be found in Appendix \ref{Appendix:A}.

In order to estimate SDE \eqref{SDE2}, we need to introduce the following auxiliary one step SDE: for $t \in [0,\eta]$ with $\eta$ being a step size,
\begin{equation}\label{SDE2-0}
		\left\{
			\begin{aligned}
				\dd \dismomentum_t &= -\gamma \dismomentum_{0} \dd t - \frac{1}{N} \sum_{i = 1}^N \nabla F (\disposition_{0}, \xi^i) \dd t + \beta \dd \BM_t,  \\
				\dd \disposition_{t} &= \dismomentum_{0} \dd t, \\
			\end{aligned}	
		\right.		
	\end{equation}	
where $(\dismomentum_0,\disposition_{0})=(\bfm, \bfx)$, and $\xi^1,...,\xi^N$ are i.i.d. and satisfy Assumption \ref{Assump.2}.
%
%
%

\subsection{Moments Estimates}
We give in this subsection the moments estimates for	
$\conmomentum_t$, $\conposition_t$, $\dismomentum_{t_n}$ and $\disposition_{t_n}$. 
	
	\begin{lemma}
		\label{lem3-1}
		Under Assumption \ref{Assump.1}, for each $p \geqslant 1$,  there exists a positive number $C$, which also depends on $p$, such that for any  $t \geqslant 0$, and initial value $(\bfm,\bfx)$, 
		\begin{align*}
			\bbE \left[ \mathcal{V}(\conmomentum_t^{\bfm}, \conposition_t^{\bfx})^p \right] &\leqslant \rme^{-\lambda \gamma t} \mathcal{V}(\bfm,\bfx)^p + C d^p.			
		\end{align*}	
	\end{lemma}	
	
	Combining this lemma with \eqref{eq:bounds}, for each $p\geqslant1$, there exists a positive number $C$, which also depends on $p$, such that for all $t\geqslant 0$ and initial value $(\bfm,\bfx)$
	\begin{equation}
		\label{eq:2pm}
		\bbE \abs{\conmomentum_t^{\bfm}}^{2p}  +  \bbE  \abs{\conposition_t^{\bfx}}^{2p}   \leqslant C \left( \rme^{-\lambda \gamma t} \mathcal{V}(\bfm, \bfx)^p + d^p \right). 
	\end{equation}
	Then, using \eqref{eq:f-linear} and \eqref{eq:2pm}, we can obtain the following estimates.
	
	\begin{lemma}
		\label{lem3-2}
		Under Assumption \ref{Assump.1}, for each $p \geqslant 1$, there exists a positive number $C$, which also depends on $p$,  such that for all $t \in [0, 1]$ and initial value $(\bfm,\bfx)$:
		\begin{equation*}
			\bbE \abs{\conmomentum_t^{\bfm}-\bfm}^{2p} + \bbE \abs{\conposition_t^{\bfx}-\bfx}^{2p} \leqslant C t^p ( \mathcal{V}(\bfm,\bfx)^p + d^p ). 
		\end{equation*}
	\end{lemma}

	What's more,  we have the similar consequences for SDE \eqref{SDE2-0},
	
	\begin{lemma}
		\label{lem3-3} Consider SDE \eqref{SDE2-0}. 
		Under Assumptions \ref{Assump.1} and \ref{Assump.2}, for each $1 \leqslant p\leqslant q/2$, 
		there exists some positive number $C$, which also depends on $p$,  such that,
		\begin{align} \label{eq:lem3-3}
			 \bbE \big|{\dismomentum_{t}^{\bfm} - \bfm} \big|^{2p} + \bbE \big|{\disposition_{t}^{\bfx} - \bfx}\big|^{2p}  \leqslant  C t^p (\mathcal{V}(\bfm,\bfx)^p + d^p), 
		\end{align}
		for all $t \in [0, \eta]$ with $\eta$ being a step size and initial value $(\bfm,\bfx)$. If the condition \eqref{A4-2} in Assumption \ref{Assump.4} holds additionally, \eqref{eq:lem3-3} holds for any $1 \leqslant p \leqslant q^{\prime}/2$. 
	\end{lemma}
	
	On the other hand, using the estimate for SDE \eqref{SDE2-0} in an inductive way, SDE \eqref{SDE2} has moments estimates similar to Lemma \ref{lem3-1}.
	\begin{lemma} 
		\label{lem3-4}
		Under Assumptions \ref{Assump.1}, \ref{Assump.2} and \ref{Assump.3}, let $\eta_1 \leqslant c$ for some positive constant $c$, and $t_n = \sum_{k=1}^n \eta_k$. For each $1 \leqslant p\leqslant q/2$, there exists a positive number $C$, which also depends on $p$,  such that 
		\begin{align}
			\bbE\left[ \mathcal{V} \big(\dismomentum_{t_n}^{\bfm}, \disposition_{t_n}^{\bfx} \big)^p \right] &\leqslant  \rme^{- \frac{\lambda \gamma}{2} t_n} \mathcal{V}(\bfm, \bfx)^p + Cd^p,   \label{eq:lem3-4}
		\end{align}
		for all $n \in \mathbb{N}$ and initial value $(\bfm, \bfx)$.  If the condition \eqref{A4-2} in Assumption \ref{Assump.4} holds additionally, \eqref{eq:lem3-4} holds for any $1 \leqslant p \leqslant q^{\prime}/2$. 
	\end{lemma}

\subsection{Auxiliary Lemmas for $1$-Wasserstein Distance}

	The following lemma provides a bound for the difference between $(\conmomentum_{\eta}^{\bfm} , \conposition_{\eta}^{\bfx})$ and $(\dismomentum_{\eta}^{\bfm} , \disposition_{\eta}^{\bfx})$.

	\begin{lemma} \label{lem3-5}
		Consider SDE \eqref{SDE2-0}. Under Assumptions \ref{Assump.1} and \ref{Assump.2}, 
		there exists a positive number $C > 0$ 
		such that for any $\eta\in(0,1)$,
		\begin{equation}
			\label{eq:W1-one}
			d_{\mathcal{W}_1} \big( \mathcal{L}( \conmomentum_{\eta}^{\bfm} , \conposition_\eta^{\bfx} ) , \mathcal{L}( \dismomentum_{\eta}^{\bfm} , \disposition_\eta^{\bfx} ) \big) 
			\leqslant C \eta^{3/2} { \left( 1+\frac1N \right)\sqrt{ \mathcal{V}(\bfm, \bfx) + d } }.
		\end{equation} 
		Let the condition \eqref{A4-2} in Assumption \ref{Assump.4} hold additionally, we have 
		\begin{equation}
			\label{eq:4thM}
			\bbE \big|{\conmomentum_{\eta}^{\bfm} - \dismomentum_{\eta}^{\bfm}}\big|^4 + \bbE \big|{\conposition_{\eta}^{\bfx} - \disposition_{\eta}^{\bfx}}\big|^4
			\leqslant C \eta^6 ( \mathcal{V}(\bfm , \bfx)^2 + d^2 ) + C { \frac{\eta^4}{N^2} }.
		\end{equation} 	
	\end{lemma}

	\cite{MR4757507} obtained a global contractivity for Langevin dynamics which will help us to prove the main result. Conveniently, we briefly introduce it here.
	
	Recall the condition \eqref{A1-3}
	\[
		\inprod{\bfx - \bfy, \nabla f(\bfx) - \nabla f(\bfy)} \geqslant a \abs{\bfx - \bfy}^2 - b , \quad  \forall \ \bfx, \bfy \in \bbR^d,
	\]
	then it holds
	\[
		\inprod{\bfx - \bfy, \nabla f(\bfx) - \nabla f(\bfy)} \geqslant \frac{a}{2} \abs{ \bfx - \bfy }^2, \quad \forall \ \bfx, \bfy \in \bbR^d \ \text{such that} \ \abs{ \bfx - \bfy }^2 \geqslant \frac{2b}{a} .
	\]
	Together with $\abs{ \nabla f(\bfx) - \nabla f(\bfy)} \leqslant L$ for all $\bfx, \bfy \in \bbR^d$,  $f$ is a potential function with a $L$-Lipschitz continuous gradient and that is $a/2$-strongly convex outside a Euclidean ball of radius $ \sqrt{2b/a}$. 
	Then, \cite[Theorem 5]{MR4757507} immediately implies the following lemma when $\gamma$ is sufficiently large such that 
	\[
		\gamma^2 > \frac{2(2L+a)^2}{a}.
	\] 
	
	\begin{lemma}
		\label{lem3-6}
		Suppose that $(\conmomentum_t^1, \conposition_t^1)_{t \geqslant 0}$ and $(\conmomentum_t^2, \conposition_t^2)_{t \geqslant 0}$ satisfy SDE \eqref{SDE1} with different initial values $(\bfm^1, \bfx^1)$ and $(\bfm^2, \bfx^2)$ respectively. Let Assumption \ref{Assump.1} hold and assume $\gamma > \sqrt{2}(2L+a) / \sqrt{a}$ additionally.
		Then there exist positive constants $C$ and {$\theta$} ($\theta$ does not depend on $d$) such that 
		\begin{equation*}
			d_{\mathcal{W}_1} (\mathcal{L} (\conmomentum_t^1, \conposition_t^1), \mathcal{L} (\conmomentum_t^2, \conposition_t^2)) \leqslant C \rme^{-\theta t}  \big(|\bfm^1-\bfx^1|+|\bfm^2- \bfx^2| \big),
		\end{equation*} 
		for all $t \geqslant 0$.
	\end{lemma}

\subsection{Auxiliary Lemmas for Total Variation Distance}
The following lemma is a direct application of Zhang et al.\ \cite{MR2673982}, and we will give the details in Appendix \ref{C3:total variation distance}.
	\begin{lemma}
		\label{lem3-7}
		For any fixed $T>0$, Let process $\bfY_t = (\conmomentum_t, \conposition_t)_{t\in[0,T]}$ come from SDE \eqref{SDE1} and let Assumption \ref{Assump.1} hold. Then, for any function $\phi \in C_b^1(\bbR^{2d}, \bbR) $, there exists a positive number $C>0$ such that
		\begin{equation*}
			\abs{ \nabla \bbE \phi(\bfY_{T} ) } \leqslant C \infnorm{ \phi } \left( T^{{3}/{2}} \lor T^{-{3}/{2}} \right).
		\end{equation*}
	\end{lemma}

	To estimate the total variation distance, we need to use Malliavin calculus. Let us briefly introduce its preliminary in our setting. More details can be found in \cite{MR2200233}. 
	
		Denote $\mathcal{H}= L^2([0,T]; \bbR^{d})$. Let $\bfW = (W_t^1, \dotsc, W_t^d)_{t\geqslant0}$ be a $d$-dimensional Wiener process (a.k.a. Brownian motion) on a probability space $(\Omega, \mathcal{F}, \bbP)$, where $\mathcal{F}$ is the $\sigma$-field generated by $\bfW$. For any $\bfh = (h_1, \dotsc, h_d) \in \mathcal{H}$, define the Wiener integral as
	\[
		\bfW(\bfh)  = \sum_{i=1}^{d} \int_{0}^T h_i(t) \dd W_t^i .
	\]
	
	We denote by $\mathcal{C}_p^{\infty}(\bbR^m, \bbR)$ the set of infinitely differentiable functions $g:\bbR^m \to \bbR$ such that $g$ and all of its partial derivatives have polynomial growth. Denote by $\mathcal S$ the set of random variables in $L^2(\Omega)$ with the form:
	\begin{equation*}
		G = g(\bfW(\bfh_1), \dotsc, \bfW(\bfh_{m})) ,
	\end{equation*}
	where $\bfh_i \in \mathcal{H}$, $i=1, \dotsc, m$ for $m \in \mathbb N$. Then the first order Malliavin derivative of $G$ is the $\mathcal{H}$-valued random variable given by
	\[
	\md_t G = \sum_{j=1}^m \partial_j g( \bfW(\bfh_1), \dotsc, \bfW(\bfh_{m}) ) \bfh_j(t), \quad  0 \leqslant t \leqslant T. 
	\]
So $\md_t G \in L^2(\Omega,\mathcal H)$.  
	We can further define the second order Malliavin derivative of $G$ as the following
\[
	\md_{t_1} \md_{t_2} G = \sum_{j_1=1}^m  \sum_{j_2=1}^m \partial_{j_1} \partial_{j_2}  g( \bfW(\bfh_1), \dotsc, \bfW(\bfh_{m}) ) \bfh_{j_1}(t_1)  \bfh_{j_2}(t_2), \quad  0 \leqslant t_1, t_2 \leqslant T. 
	\]
So $\md_{t_1} \md_{t_2} G \in L^2(\Omega,\mathcal H\otimes \mathcal H)$. Inductively, we can define the $k$-th order Malliavin derivative $D_{t_1}...D_{t_k} G$, which is located in $L^2(\Omega,\mathcal H^{\otimes k})$. 
Define the following norm 
\[
	\norm{G}_{k,p}:  = \left[ \bbE \abs{G}^p + \sum_{j=1}^k\bbE \norm{\md_{t_1} \cdots \md_{t_j} G}_{\mathcal{H}^{\otimes j}}^p \right]^{1/p}.
	\]
Under this norm, the operator $\md$ can extended from $\mathcal S$ to its domain, denoted by $\mathbb{D}^{k,p}$.   Specially,  $\mathbb{D}^{1,2}$ is also a Hilbert space with product
	\[
	\inprod{ G , H }_{1,2}  =  \bbE (GH) + \bbE \inprod{\md G, \md H}_{\mathcal{H}},  \quad \forall \ G, H \in \mathbb{D}^{1,2}.
	\]
The following relation is called integration by parts in Malliavin calculus: 
\begin{equation}\label{eq:dual}
		\bbE\left[G \delta(\bfh) \right] = \bbE\left[ \inprod{ \md G , \bfh }_{\mathcal{H}} \right] , \quad \forall \ G\in \mathbb{D}^{1,2}, \quad \bfh \in \mathcal H,  
	\end{equation}	
where $\delta(\bfh)$ is called Skorohod integral. 

 Given $\bfh \in \mathcal{H}$, we can define the Malliavin derivative along the direction $\bfh$, denoted by $\md^{\bfh} G$, as the following: 
	\[
		\md^{\bfh} G = \inprod{\md G , \bfh}_{\mathcal{H}}
		= \int_{0}^T \inprod{ \bfh(s) , \md_{s} G} \dd s .
 	\]
For  $\bfG = (G^1, \dotsc, G^d )^{\top}$ with each $G^i \in \mathbb{D}^{1,2}$, we define $\md_t \bfG = ( \md_t G^1, \dotsc, \md_t G^d )$. And, the norm
	\[
		\norm{ \md_{t_1}\cdots \md_{t_k} \bfG}_{\mathcal{H}^{\otimes k}}^2 = \sum_{i=1}^d \norm{ \md_{t_1}\cdots \md_{t_k} G^i }_{\mathcal{H}^{\otimes k}}^2\ ,  \quad
		\norm{ \bfG }_{k,p}^p = \sum_{i=1}^d \norm{G^i }_{k,p}^p.
	\]
	The associated Malliavin matrix of $\bfG$ is defined as the following random semi-definite symmetric matrix
	\[
	\bm{\Gamma}(\bfG) = \left( \inprod{ \md G^i , \md G^j }_{\mathcal{H}} \right)_{1\leqslant i ,j \leqslant d}.
	\]

\vskip 2mm
	
	The following abstract lemma will play an important role in the proof of Theorem \ref{Thm-2}.  We leave its proof in Appendix \ref{C3:total variation distance}.
	\begin{lemma}
		\label{lem3-8}
		Let $\bfF = (F^1, \dotsc, F^d) $ be a random vector such that all of its components $F^i \in \mathbb{D}^{2,8}$, $i = 1, \dotsc, d$, and its Malliavin matrix $\bm{\Gamma}({\bfF})$ is invertible a.s.\ with $ \big( \det \bm{\Gamma}({\bfF}) \big)^{-1} \in \bigcap_{p\geqslant 1} L^p(\Omega)$. Let ${\bf G} = (G^1, \dotsc, G^d)$ be another random vector with $G^i \in \mathbb{D}^{1,4} $ for all $1\leqslant i \leqslant d$, and $g:\bbR^d \to \bbR$ be a function in $\mathcal{C}_b^1(\bbR^{d}, \bbR)$. Then there exists a positive number $C>0$ such that 
		\begin{equation} \label{eq:MalB}
			\big| \bbE \inprod{ \nabla g(\bfF), \bfG} \big|  
			\leqslant C \infnorm{g} \norm{ \bfG }_{1,4} \big\{ \bbE [\mathcal{K}_1 \mathcal{K}_2 \mathcal{K}_3 ] \big\}^{{1}/{4}},
		\end{equation}
		where
		\begin{equation*}
			\mathcal{K}_1 = 1 + \norm{ \md \bfF }_{ \mathcal{H} }^8, \qquad
			\mathcal{K}_2 = 1 + \norm{ \md^2 \bfF }_{\mathcal{H}\otimes \mathcal{H}}^4, \qquad
			\mathcal{K}_3 = 1 + \hsnorm{ \bm{\Gamma}({\bfF})^{-1} }^8.
		\end{equation*}
	\end{lemma}
	

	In order to apply Lemma \ref{lem3-8}, we need to introduce a series of events to estimate $\bbE [\mathcal{K}_1 \mathcal{K}_2 \mathcal{K}_3 ]$. To this end, let us first recall that for each step $i \in \bbN$ of SGDm \eqref{sgd_m}, $\xi_i^r$, $r = 1, \dotsc, N$ are independent copies of random variable $\xi$ satisfying Assumption \ref{Assump.2}, to make notations simple, we denote $\bm{\xi}_i = (\xi_i^1, \dotsc, \xi_i^N)$, $i \in \mathbb{N}$, and define
	\begin{equation} \label{e:Phi}
		\varphi (\bm{\xi}_i) := \frac{1}{N} \sum_{r=1}^N \sup_{ \bfx\in\bbR^d } \opnorm{\nabla^2 F(\bfx, \xi_i^r)} .
	\end{equation} 	
Define the following events: for $j \leqslant k$ and $\tau = 1, 2$,
	\begin{equation} \label{eq:Ejk}
		E_{j,k}^\tau := \left\{ \sum_{i = j + 1}^k \eta_i ( \varphi (\bm{\xi}_i)^\tau - \bbE [\varphi (\bm{\xi}_i)^\tau] ) > t_k - t_j \right\}
		.
	\end{equation}
	
	\begin{lemma}
		\label{lem3-9}
		Given Assumptions \ref{Assump.3} and \ref{Assump.4}, let $t_k = \sum_{i = 1}^k \eta_i$. There exist  positive constants $c$ and $C$ such that
		\begin{equation*}
			\bbP (E_{j,k}^\tau) \leqslant \frac{C \rme^{c (t_k - t_j)}}{(t_k - t_j)^2} \eta_k^2, \quad \tau = 1, 2.
		\end{equation*}
	\end{lemma}

For notations simplicity, we denote 
\begin{equation} \label{e:YTilY}
	\bfY_t = (\conmomentum_t, \conposition_t), \quad \widetilde{\bfY}_t = (\dismomentum_t, \disposition_t) , \quad \forall \ t \geqslant 0,
\end{equation}
where $(\conmomentum_t, \conposition_t)$ and $(\dismomentum_t, \disposition_t)$ are defined by SDEs \eqref{SDE1} and  \eqref{SDE2} respectively. 
 For any $0\leqslant s \leqslant t$, denote $\bfY_{s,t}^{\bfy}$ or $\bfY_{s,t}(\bfy)$ as the value of the process $(\bfY_t)_{t\geqslant 0}$ at the time $t$ given $\bfY_s = \bfy \in \bbR^{2d}$. Similarly for $\widetilde{\bfY}_{s,t}^{\bfy}$. When $s=0$, we will drop the subscript $s$ if there is no ambiguity.

	\begin{lemma}
		\label{lem3-10}
		Let Assumptions  \ref{Assump.1}, \ref{Assump.2}, \ref{Assump.3},  and \ref{Assump.4} hold and recall $t_k = \sum_{i = 1}^k \eta_i$. There exists a positive number $C>0$ such that for any $\bfz \in \bbR^{2d}$ and any $\ell < k$ such that $t_k - t_\ell \leqslant 1 / [5 (B_1 + \gamma + 2)]$, the followings hold:
		
		\noindent (i) For all $t \in [t_{k-1}, t_k]$
		\begin{equation*}
			\bbE \norm{ \md \widetilde{\bfY}_{t_\ell , t_k}^{\bfz} }_{\mathcal{H}}^8 \leqslant C d^4, \quad
			\bbE \norm{ \md \bfY_{t_{k-1},t} ( \widetilde{\bfY}_{t_\ell , t_{k-1}}^{\bfz} ) }_{\mathcal{H}}^8 \leqslant C d^4 .
		\end{equation*}
		
		\noindent (ii) 
		We have 
		\[
			\lVert \md \widetilde{\bfY}_{t_\ell , t_k}^{\bfz} \rVert_{\mathcal{H}} \leqslant C \sqrt{d} \quad \text{on the event}\quad  (E_{\ell,k}^1)^c.
		\]
	
	\end{lemma}

	\begin{lemma}
		\label{lem3-11}
		Under the same setting as in Lemma \ref{lem3-10},  there exists a positive number $C>0$ such that for any $\bfz \in \bbR^{2d}$ and any $\ell < k$ such that $t_k - t_\ell \leqslant 1 / [5 (B_1 + \gamma + 2)]$, there exists a positive number $C>0$ such that the followings hold:
		\begin{equation*}
			\bbE \norm{ \md^2 \widetilde{\bfY}_{t_\ell , t_k}^{\bfz} }_{\mathcal{H}\otimes \mathcal{H}}^4 \leqslant C d^4, \quad
			\bbE \norm{ \md^2 \bfY_{t_{k-1},t_k} ( \widetilde{\bfY}_{t_\ell , t_{k-1}}^{\bfz} ) }_{\mathcal{H}\otimes \mathcal{H}}^4 \leqslant Cd^4.
		\end{equation*}
	\end{lemma}

	\begin{lemma}
		\label{lem3-12}
		Under the same setting as in Lemma \ref{lem3-10},   there exists a positive number $C>0$ such that for any $\bfz \in \bbR^{2d}$ and any $\ell < k$ such that $t_k - t_\ell \leqslant 1 / [5 (B_1 + \gamma + 2)]$, there exists a positive number $C>0$ such that the followings hold:
		
		\noindent (i)  We have
		\begin{equation*}
			\bbE \norm{ \md \bfY_{t_{k-1}, t_k} ( \widetilde{\bfY}_{t_\ell , t_{k-1}}^{\bfz} ) - \md \widetilde{\bfY}_{t_\ell , t_k}^{\bfz} }_{ \mathcal{H} }^4 \leqslant { C d^2 \eta_k^6 ( \mathcal{V} (\bfz)^2 + d^2 ) + C \frac{d^4 \eta_k^4}{N^2} }.
		\end{equation*} 
		
		\noindent (ii) We have
		$$\lVert \md \bfY_{t_{k-1}, t_k} ( \widetilde{\bfY}_{t_\ell , t_{k-1}}^{\bfz} ) - \md \widetilde{\bfY}_{t_\ell , t_k}^{\bfz} \rVert_{\mathcal{H}} \leqslant C \sqrt{d \eta_k} \quad \text{on the event} \quad (E_{\ell,k}^1 \cup E_{\ell,k}^2)^c.$$
	\end{lemma}

	\begin{lemma}
		\label{lem3-13}
		Under the same setting as in Lemma \ref{lem3-10}, for any $\bfz \in \bbR^{2d}$ and any $\ell < k$ such that $t_k - t_\ell \leqslant 1 / [5 (B_1 + \gamma + 2)]$, there exists a positive number $C>0$ such that
		\begin{equation*}
			\lambda_{\min} \left( \bm{\Gamma} \big( \widetilde{\bfY}_{t_\ell , t_k}^{\bfz} \big) \right) \geqslant C (t_k - t_\ell)^3 \quad \text{ on the event } \quad (E^1_{\ell,k} \cup E^1_{m,k})^c, 
		\end{equation*}
		where $m = \min\{ j \colon t_k - t_j \leqslant {(t_k - t_\ell)}/{[10 (B_1 + \gamma + 2)]}\}$ and $\lambda_{\min} \big( \bm{\Gamma} \big( \widetilde{\bfY}_{t_\ell , t_k}^{\bfz} \big)\big)$ is the smallest eigenvalue of the Malliavin matrix $\bm{\Gamma} \big( \widetilde{\bfY}_{t_\ell , t_k}^{\bfz} \big)$.
	\end{lemma}

	
\section{Proofs of the Main Results} \label{sec:pfs-Thm}
	
In this section, we provide the proofs of Theorems \ref{Thm-1}, \ref{Thm-2}, and Corollaries \ref{Cor-W1}, \ref{Cor-Error} by using the lemmas mentioned above. Employing the Lindeberg principle, we decompose $\bbE[g({\bfY}_{t_k})] - \bbE[ g(\widetilde{\bfY}_{t_k})]$ into $k$ terms and subsequently analyze each term individually, where ${\bfY}_t$ and $\widetilde{\bfY}_t$ are defined by \eqref{e:YTilY}. 

 For any $0\leqslant s \leqslant t$, recall the notations $\bfY_{s,t}^{\bfy}$, $\bfY_{s,t}(\bfy)$ and $\widetilde{\bfY}_{s,t}^{\bfy}$ immediately below \eqref{e:YTilY}.

\subsection{Proof of Theorem \ref{Thm-1}}	
	
	
	\begin{proof}[Proof of Theorem \ref{Thm-1}]
		Applying the Lindeberg technique, we have the following decomposition for any Lipschitz function $h \colon \bbR^{2d} \to \bbR$,
		\begin{equation}
			\label{eq:lindeberg W1}
			\bbE h(\bfY_{t_n}^{\bfy}) - \bbE h( \widetilde{\bfY}_{t_n}^{\bfy})
			= \sum_{k = 1}^n \bbE \left[ h \left( \bfY_{t_{k-1}, t_n} ( \widetilde{\bfY}_{t_{k-1}}^{\bfy} ) \right) - h \left( \bfY_{t_k, t_n} ( \widetilde{\bfY}_{t_k}^{\bfy} ) \right) \right]
		\end{equation}
		 with initial value $\bfy = (\bfm_0,\bfx_0)$.
		
		Denote $\bfz=\bfY_{t_{k-1}, t_k} ( \widetilde{\bfY}_{t_{k-1}}^{\bfy} )$ and $\tilde{\bfz}=\widetilde{\bfY}_{t_k}^{\bfy}$. We apply Lemma \ref{lem3-6} on $[t_k, t_n]$ to get 
		\begin{align*}
			&\pheq
				\left|\bbE \left[ h \left( \bfY_{t_{k-1}, t_n} ( \widetilde{\bfY}_{t_{k-1}}^{\bfy} ) \right) - h \left( \bfY_{t_k, t_n} ( \widetilde{\bfY}_{t_k}^{\bfy} ) \right) \right]\right| \\
				&=\big|\bbE \left[ h \left( \bfY_{t_{k}, t_n} (\bfz) \right) - h \left( \bfY_{t_k, t_n} (\tilde{\bfz})\right) \right]\big| \\
				&=\big|\bbE \left[ h \left( \bfY_{t_{k}, t_n} ({\bfz}^*) \right) - h \left( \bfY_{t_k, t_n} ({\tilde{\bfz}}^*)\right) \right]\big| \\
				&
				\leqslant C \norm{h}_{\Lip} \rme^{-\theta(t_n - t_k)} \bbE |{\bfz}^*-{\tilde \bfz}^*|,
		\end{align*}
		where $({\bfz}^*, {\tilde \bfz}^*)$ is a random vector whose first and second marginal distributions are the same as ${\bfz}$ and ${\tilde \bfz}$ respectively. Because this inequality also holds for any $\bfz^*$ and ${\tilde{\bfz}}^*$ with this property, we can choose $\bfz^*$ and ${\tilde{\bfz}}^*$ which satisfies $\bbE |\bfz^*-{\tilde \bfz}^*|=d_{\mathcal{W}_1}(\mathcal{L}(\bfz),\mathcal{L}({\tilde{\bfz}}^*))$ (It is well known that the minimum of the coupling in the 1-Wasserstein can be realized). Hence, 
	\begin{align*}
\left|\bbE \left[ h \left( \bfY_{t_{k-1}, t_n} ( \widetilde{\bfY}_{t_{k-1}}^{\bfy} ) \right) - h \left( \bfY_{t_k, t_n} ( \widetilde{\bfY}_{t_k}^{\bfy} ) \right) \right]\right| 
				\leqslant C \norm{h}_{\Lip} \rme^{-\theta(t_n - t_k)} d_{\mathcal{W}_1}(\mathcal{L}(\bfz),\mathcal{L}({\tilde{\bfz}})),
		\end{align*}		
		Applying Lemma \ref{lem3-4}, it holds
		\[
			\bbE \left[ \mathcal{V}(\widetilde{\bfY}_{t_{k-1}}^{\bfy}) \right] \leqslant \rme^{-\lambda \gamma t_{k-1}} \mathcal{V}(\bfy) + C d.
		\]
		Furthermore, combining this with \eqref{eq:W1-one} in Lemma \ref{lem3-5} immediately yields that
		\begin{align*}
			d_{\mathcal{W}_1}(\mathcal{L}(\bfz),\mathcal{L}({\tilde{\bfz}}))
			&\leqslant C \eta_k^{3/2} \left(1+\frac1N\right) \bbE\left[  \sqrt{\mathcal{ V }(\widetilde{\bfY}_{t_{k-1}}^{\bfy}) + d} \right] \\ 
			&\leqslant C \eta_{k}^{3/2} { \left(1+\frac1N\right) \sqrt{\mathcal{V}(\bfy) + d} }.
		\end{align*} 
		Hence, 
		\begin{equation} \label{eq:simgleEST}
			\begin{aligned}
				&\pheq
				\abs{\bbE \left[ h \left( \bfY_{t_{k-1}, t_n} ( \widetilde{\bfY}_{t_{k-1}}^{\bfy} ) \right) - h \left( \bfY_{t_k, t_n} ( \widetilde{\bfY}_{t_k}^{\bfy} ) \right) \right]} \\
				&\leqslant C \norm{h}_{\Lip} \eta_k^{{3}/{2}} \rme^{-\theta (t_n - t_k)} {  \left(1+\frac1N\right)\sqrt{\mathcal{V}(\bfy) + d} } .
			\end{aligned}
		\end{equation}
			Substituting this formula in \eqref{eq:lindeberg W1}, we have
		\begin{equation} \label{eq:W1 sum} 
			\abs{ \bbE h(\bfY_{t_n}^{\bfy}) - \bbE h( \widetilde{\bfY}_{t_n}^{\bfy}) }
			\leqslant C \norm{h}_{\Lip} { \left(1+\frac1N\right) \sqrt{\mathcal{V}(\bfy) + d} } \sum_{k=1}^n \eta_k^{{3}/{2}} \rme^{-\theta (t_n - t_k)}.
		\end{equation}	
		For the sum term, applying Lemma \ref{lem-A-3} with $\varepsilon = 1/2$ therein, we have
		\[
			 \sum_{k=1}^n \eta_k^{{3}/{2}}  \rme^{-\theta (t_n - t_k)}
			 \leqslant 
			 \frac{4}{2\theta - \omega } \sqrt{\eta_n}  . 
		\]		
		Consequently, the desired result holds since $(\bfm_n, \bfx_n)_{n\geqslant 0}$ and $(\dismomentum_{t_n},\disposition_{t_n})_{n\geqslant 0}$ have the same distribution.	
	\end{proof}

\subsection{Proof of Theorem \ref{Thm-2}}
 We denote the operator semigroups of $(\bfY_t)_{t\geqslant 0}$ and $(\widetilde{\bfY}_t)_{t\geqslant 0}$ by $\operatorP_{s, t} $ and $\operatorQ_{s,t}$ respectively, defined by
\[
	\operatorP_{s, t} h(\bfy) = \bbE \left[ h(\bfY_t) | \bfY_s = \bfy \right], \quad  \operatorQ_{s,t} h(\bfy) = \bbE [ h(\widetilde{\bfY}_t) | \widetilde{\bfY}_s = \bfy ] , \quad  1\leqslant s < t,
\]
for any $h\in \mathcal{C}_b(\bbR^{2d}, \bbR)$ and $\bfy \in \bbR^{2d}$. As $s=0$, we will drop the subscript $s$ if no confusions arise. 
 	
 	\begin{proof}[Proof of Theorem \ref{Thm-2}]
 	For any $h\in \mathcal{C}_b(\bbR^{2d}, \bbR)$, by the Lindeberg principle in the form of semigroup, we have the following decomposition
 	\begin{equation*}
 		\bbE h(\bfY_{t_n}^{\bfy}) - \bbE h( \widetilde{\bfY}_{t_n}^{\bfy})
 		= \left( \operatorP_{0,t_n} - \operatorQ_{0,t_n} \right)h(\bfy)
 		= \sum_{k = 1}^n \operatorQ_{0,t_{k-1} } \circ \left( \operatorP_{t_{k-1},t_k} - \operatorQ_{t_{k-1}, t_k} \right) \circ \operatorP_{t_k,t_n} h(\bfy).
 	\end{equation*}
 	Define index $k_n$  as 
 	\begin{equation*}
 		k_n = \inf \{ k \colon t_n-t_k < 1 \}.
 	\end{equation*}
 	Then, the above equation can be divided into the following two parts:
 	\begin{equation} \label{eq:dec-TV}
 		\begin{split}
 			\mathcal{J}_1 &= \sum_{k = 1}^{ {k_n}-1} \operatorQ_{0,t_{k-1} } \circ \left( \operatorP_{t_{k-1},t_k} - \operatorQ_{t_{k-1}, t_k} \right) \circ \operatorP_{t_k,t_n} h(\bfy) , \\
 			\mathcal{J}_2 &= \sum_{k = k_n }^{n} \operatorQ_{0,t_{k-1} } \circ \left( \operatorP_{t_{k-1},t_k} - \operatorQ_{t_{k-1}, t_k} \right) \circ \operatorP_{t_k,t_n} h(\bfy). 
 		\end{split}
 	\end{equation}
 	
 	For $\mathcal{J}_1$ and each $1\leqslant k < k_n$, we let $\operatorP_{t_k,t_n} h(\bfy) = \operatorP_{t_k,t_{k_n} } \circ \operatorP_{t_{k_n},t_n} h(\bfy)$, and define function $g_1(\bfy) := \operatorP_{t_{k_n},t_n} h(\bfy) $. 
 	Then $g_1$ is a Lipschitz function according to Lemma \ref{lem3-7} with $T = t_n-t_{k_n} \in [1/2,1]$ therein, and it holds
 	\begin{equation*}
 		\lipnorm{ g_1 } = \abs{\nabla \operatorP_{t_{k_n},t_n} h(\bfy)}
 		\leqslant C \infnorm{ h }.
 	\end{equation*}
 	Similar to obtaining \eqref{eq:simgleEST} in the proof of Theorem \ref{Thm-1}, we have
 	\begin{equation*}
 		\begin{aligned}
 			&\pheq
 			\abs{ \operatorQ_{0,t_{k-1} } \circ \left( \operatorP_{t_{k-1},t_k} - \operatorQ_{t_{k-1}, t_k} \right) \circ \operatorP_{t_k,t_{k_n}} g(\bfy) } \\
 			&\leqslant C \lipnorm{ g_1 } \eta_k^{{3}/{2}} \rme^{-\theta (t_{k_n} - t_k)} { \left(1+\frac1N\right) \sqrt{\mathcal{V}(\bfy) + d} }.
 		\end{aligned}
 	\end{equation*}
 	Since $t_{k_n} > t_n - 1$, we obtain the following estimate of $\mathcal{J}_1$:
 	\begin{equation}
 		\label{eq:J1-TV}
 		\abs{\mathcal{J}_1}
 		\leqslant C  \infnorm{ h } {  \sqrt{\mathcal{V}(\bfy) + d} \left(1+\frac1N\right) } \sum_{k = 1}^{k_n-1} \eta_k^{ {3}/{2}}  \rme^{-\theta (t_n - t_k)}.
 	\end{equation}
	
 For the term $\mathcal{J}_2$, we claim 
 \begin{equation} \label{eq:J2-TV}
 		\abs{\mathcal{J}_2}
 		\leqslant C \infnorm{h} { d^{7/2} \sqrt{\mathcal{V} (\bfy) + d} } \sum_{k = k_n}^{n} \left( \eta_k^{{3}/{2}} + \frac{\eta_k}{\sqrt{N}} \right) \rme^{-\theta (t_n - t_k)}.
 	\end{equation}
 	Combining \eqref{eq:dec-TV}, \eqref{eq:J1-TV} and \eqref{eq:J2-TV} implies
 	\begin{equation} \label{eq:TV sum}
 		\abs{\bbE h(\bfY_{t_n}^{\bfy}) - \bbE h( \widetilde{\bfY}_{t_n}^{\bfy})}
 		\leqslant C \infnorm{h} { d^{7/2}  \sqrt{\mathcal{V} (\bfy) + d} } \sum_{k = 1}^{n} \left( \eta_k^{{3}/{2}} + \frac{\eta_k}{\sqrt{N}} \right) \rme^{-\theta (t_n - t_k)}.
 	\end{equation}
 	Then, applying Lemma \ref{lem-A-3} yields the desired result.
 	
 \vskip 2mm

	It remains to prove \eqref{eq:J2-TV}. For any fixed $k \in \{k_n, \dots, n\}$ in $\mathcal{J}_2$, we can find time $t_{\ell}$ such that 
 	\begin{equation}
 		\label{def:tl}
 		 \frac{1}{6 (B_1 + \gamma + 2)} \leqslant  t_k-t_{\ell} \leqslant \frac{1}{5 (B_1 + \gamma + 2)}.
 	\end{equation}
 	Then we have the following decomposition,
 		\begin{align*}
 			&\pheq 
 			\operatorQ_{0,t_{k-1} } \circ \left( \operatorP_{t_{k-1},t_k} - \operatorQ_{t_{k-1}, t_k} \right) \circ \operatorP_{t_k,t_n} h(\bfy) \\
 			&= \operatorQ_{0, t_{\ell} } \circ \operatorQ_{t_{\ell}, t_{k-1} } \circ \left( \operatorP_{t_{k-1},t_k} - \operatorQ_{t_{k-1}, t_k} \right) g_2(\bfy) \\
 			&= \bbE \left\{ \left[ g_2 \left( \bfY_{t_{k-1}, t_k} \left( \widetilde{\bfY}_{t_{\ell} , t_{k-1}}^{\bfz} \right)  \right) - g_2\left( \widetilde{\bfY}_{t_{\ell} , t_k}^{\bfz} \right) \right] \left( \mathbf{1}_{E_{\ell,k}^1 \cup E_{m,k}^1 \cup E_{\ell,k}^2} + \mathbf{1}_{(E_{\ell,k}^1 \cup E_{m,k}^1 \cup E_{\ell,k}^2)^c} \right) \right\} \\
 			&= \bbE \left\{ \left[ g_2 \left( \bfY_{t_{k-1}, t_k} \left( \widetilde{\bfY}_{t_{\ell} , t_{k-1}}^{\bfz} \right)  \right) - g_2\left( \widetilde{\bfY}_{t_{\ell} , t_k}^{\bfz} \right) \right] \mathbf{1}_{E_{\ell,k}^1 \cup E_{m,k}^1 \cup E_{\ell,k}^2} \right\} \\
 			&\quad + \bbE \left\{ \bbE \left[
 			g_2\left( \bfY_{t_{k-1}, t_k} \left( \widetilde{\bfY}_{t_{\ell} , t_{k-1}}^{\bfz} \right)  \right) - g_2\left( \widetilde{\bfY}_{t_{\ell} , t_k}^{\bfz} \right) \middle| \mathscr{F}_{\ell, k} \right] \mathbf{1}_{(E_{\ell,k}^1 \cup E_{m,k}^1 \cup E_{\ell,k}^2)^c} \right\} \\
 			&=: I_{k,1} + I_{k,2},
 		\end{align*}
 	where $g_2(\bfy) = \operatorP_{t_k,t_n} h(\bfy)$, $\bfz = \widetilde{\bfY}_{t_{\ell}}^{\bfy}$, $\mathscr{F}_{\ell, k} = \sigma (\widetilde{\bfY}_{t_{\ell}}^{\bfy}, \, (\bm{\xi}_i)_{\ell + 1 \leqslant i \leqslant k})$, and $E_{\ell,k}^1$, $E_{m,k}^1$, $E_{\ell,k}^2$ are defined in \eqref{eq:Ejk} with
 	\begin{equation*}
 		m = \min \left\{ j \colon t_k - t_j \leqslant \frac{t_k - t_\ell}{10 (B_1 + \gamma + 2)} \right\}.
 	\end{equation*} 	
Let us estimate $I_{k,1}$ and $I_{k,2}$. 

For $I_{k,1}$, since $C_1 \leqslant t_k - t_m \leqslant t_k - t_\ell \leqslant C_2$ holds for some constants $C_1, C_2 > 0$, by Lemma \ref{lem3-9}, we have
 	\begin{equation} \label{eq:Ik1}
        \begin{split} 
 		\abs{I_{k,1}} & \leqslant 2 \infnorm{g_2} \bbP (E_{\ell,k}^1 \cup E_{m,k}^1 \cup E_{\ell,k}^2) \\
 		&  \leqslant 2 \infnorm{h} ( \bbP (E_{\ell,k}^1) + \bbP (E_{m,k}^1) + \bbP (E_{\ell,k}^2) ) 
 		\leqslant C \infnorm{h} \eta_k^2.
        \end{split}
 	\end{equation}
	
 	 For $I_{k,2}$, observe that
 	\begin{equation}\label{eq:Ik2}
 		I_{k,2}
 		= \int_{0}^1 \bbE \left\{ \bbE \left[ \inprod{ \nabla g_2 ( \widetilde{\bfY}_{t_\ell , t_k}^{\bfz} + r \bm{\Xi}_{t_\ell, t_k} ) , \bm{\Xi}_{t_\ell, t_k} } \middle| \mathscr{F}_{\ell, k} \right] \mathbf{1}_{(E_{\ell,k}^1 \cup E_{m,k}^1 \cup E_{\ell,k}^2)^c} \right\} \dd r,
 	\end{equation}
 	where $\bm{\Xi}_{t_\ell, t_k} = \bfY_{t_{k-1}, t_k} ( \widetilde{\bfY}_{t_\ell , t_{k-1}}^{\bfz} ) - \widetilde{\bfY}_{t_\ell , t_k}^{\bfz}$. Denote $\bfG = \bm{\Xi}_{t_\ell, t_k}$ and $\bfF = \widetilde{\bfY}_{t_\ell , t_k}^{\bfz} + r \bm{\Xi}_{t_\ell, t_k}$, and we shall apply Lemma \ref{lem3-8} to estimate  $\bbE \left[ \inprod{ \nabla g_2 ( \widetilde{\bfY}_{t_\ell , t_k}^{\bfz} + r \bm{\Xi}_{t_\ell, t_k} ) , \bm{\Xi}_{t_\ell, t_k} } \middle| \mathscr{F}_{\ell, k} \right]$ on the event $(E_{\ell,k}^1 \cup E_{m,k}^1 \cup E_{\ell,k}^2)^c$. To this end, 
let us consider SDE \eqref{SDE2} on $[t_\ell, t_k]$. Note that $\mathscr{F}_{\ell, k}$ is independent of $(\bfB_t)_{t \in [t_\ell, t_k]}$ and that the Malliavin calculus in Lemma \ref{lem3-8} is associated to $(\bfB_t)_{t \in [t_\ell, t_k]}$ and has nothing to do with $\mathscr{F}_{\ell, k}$. An advantage is that we can easily bound $\mathcal{K}_1, \mathcal{K}_2, \mathcal{K}_3$ on the right hand of \eqref{eq:MalB} on the event $(E_{\ell,k}^1 \cup E_{m,k}^1 \cup E_{\ell,k}^2)^c$. 
	
	According to \eqref{eq:4thM} in Lemma \ref{lem3-5} and Lemma \ref{lem3-12}(i), $\norm{\bfG}_{1,4}$ satisfies
 	\begin{equation*}
 		\norm{\bfG}_{1,4}=\left( \bbE \abs{ \bfG }^4 + \bbE \norm{ \md \bfG }^4_{ \mathcal{H} } \right)^{{1}/{4}}
 		\leqslant C \sqrt{d (\mathcal{V}(\bfz) + d)} \left( \eta_k^{{3}/{2}} + \frac{\eta_k}{\sqrt{N}} \right).
 	\end{equation*} 
 	Next we estimate $\bbE[\mathcal{K}_1\mathcal{K}_2\mathcal{K}_3]$. By Lemmas \ref{lem3-10}(ii) and \ref{lem3-12}(ii), 
 	\[
 		\mathcal{K}_1 = 1 + \norm{ \md \bfF }_{ \mathcal{H} }^8 \leqslant C d^4 \quad \text{on the event} \quad (E_{\ell,k}^1 \cup E_{\ell,k}^2)^c.
 	\] 
 	By Lemma \ref{lem3-11}, 
	\[
		\bbE \mathcal{K}_2 = 1 + \bbE \lVert \md^2 \bfF \rVert_{\mathcal{H}\otimes \mathcal{H}}^4 \leqslant C d^4.
	\]
	Let us now bound $\mathcal{K}_3 = 1 + \hsnorm{ \bm{\Gamma} ({\bfF})^{-1} }^8$. 
	Since $\hsnorm{ \bm{\Gamma}({\bfF})^{-1} } \leqslant \sqrt{2d} \, [\lambda_{\min} (\bm{\Gamma} ({\bfF}) )]^{-1}$, we first show
	\[
		\lambda_{\min} ( \bm{\Gamma} ({\bfF}) ) \geqslant C > 0 \quad  \text{ on the event}\quad  (E_{\ell,k}^1 \cup E_{m,k}^1 \cup E_{\ell,k}^2)^c.
	\] 
	Indeed, observe $\bm{\Gamma}( {\bfF} ) = \bm{\Gamma} ( \widetilde{\bfY}_{t_\ell , t_k}^{\bfz} ) + r^2 \bm{\Gamma} ( \bm{\Xi}_{t_\ell, t_k} ) + r \bfN$, where the entries of matrix $\bfN$ are given by
 	\begin{equation*}
 		N_{ij} =  \inprod{ \md ( \widetilde{\bfY}_{t_\ell , t_k}^{\bfz} )^i , \md ( \bm{\Xi}_{t_\ell, t_k} )^j }_{ \mathcal{H} } +  \inprod{ \md ( \widetilde{\bfY}_{t_\ell , t_k}^{\bfz} )^j , \md ( \bm{\Xi}_{t_\ell, t_k} )^i }_{ \mathcal{H} }, \quad i, j = 1, \dots, 2d.
 	\end{equation*}
 	By Lemmas \ref{lem3-10}(ii) and \ref{lem3-12}(ii), on the event $ (E_{\ell,k}^1 \cup E_{\ell,k}^2)^c$, we have $\lVert \md \widetilde{\bfY}_{t_\ell , t_k}^{\bfz} \rVert_{\mathcal{H}} \leqslant C \sqrt{d}$ and $\lVert \md \bm{\Xi}_{t_\ell, t_k} \rVert_{\mathcal{H}} \leqslant C \sqrt{ d \eta_k}$. Thus
 	\begin{equation*}
	\begin{split}
 		\hsnorm{\bfN}
 		& \leqslant 2 \left\{ \sum_{i, j = 1}^{2d} \norm{ \md ( \widetilde{\bfY}_{t_\ell , t_k}^{\bfz} )^i }_{ \mathcal{H} }^2 \norm{ \md ( \bm{\Xi}_{t_\ell, t_k} )^j }_{ \mathcal{H} }^2 \right\}^{1/2} \\
 		& \leqslant 2 \norm{ \md \widetilde{\bfY}_{t_\ell , t_k}^{\bfz} }_{ \mathcal{H} } \norm{ \md \bm{\Xi}_{t_\ell, t_k} }_{ \mathcal{H} }
 		\leqslant C d \sqrt{\eta_k}, \quad \text{on the event} \quad (E_{\ell,k}^1 \cup E_{\ell,k}^2)^c .
	\end{split}
 	\end{equation*}
 	This, together with Lemma \ref{lem3-13} and \eqref{def:tl}, implies
 	\begin{equation*}
         \begin{split}
 		\lambda_{\min} ( \bm{\Gamma} ({\bfF}) )
 		& \geqslant \lambda_{\min} ( \bm{\Gamma} (\widetilde{\bfY}_{t_\ell , t_k}^{\bfz}) ) + r^2 \lambda_{\min} ( \bm{\Gamma} (\bm{\Xi}_{t_\ell, t_k}) ) - r \hsnorm{\bfN} \\
		& 
 		\geqslant C + 0 - C d \sqrt{\eta_k} \\
		& > C/2, \quad \text{ on the event}\quad  (E_{\ell,k}^1 \cup E_{m,k}^1 \cup E_{\ell,k}^2)^c,
	\end{split}
 	\end{equation*}
 	for any $r \in [0, 1]$ and $\eta_k$ sufficiently small such that $\eta_k \leqslant c d^{-2}$ for some positive constant $c$. Then, we have
 	\begin{equation*}
 		\mathcal{K}_3 = 1 + \hsnorm{ \bm{\Gamma} ({\bfF})^{-1} }^8 \leqslant C d^4, \quad \text{ on the event } \quad (E_{\ell,k}^1 \cup E_{m,k}^1 \cup E_{\ell,k}^2)^c.
 	\end{equation*} 
	Hence, the following holds on event $(E_{\ell,k}^1 \cup E_{m,k}^1 \cup E_{\ell,k}^2)^c$
	\begin{equation} \label{e:ConEst}
		\bbE \left[ \inprod{ \nabla g_2 ( \widetilde{\bfY}_{t_\ell , t_k}^{\bfz} + r \bm{\Xi}_{t_\ell, t_k} ) , \bm{\Xi}_{t_\ell, t_k} } \middle| \mathscr{F}_{\ell, k} \right]
		\leqslant C \infnorm{g_2} { d^{7/2} \sqrt{\mathcal{V}(\bfz) + d} \left( \eta_k^{{3}/{2}} + \frac{\eta_k}{\sqrt{N}} \right)}.
	\end{equation}
	Recalling $\bfz=\widetilde{\bfY}_{t_\ell}^{\bfy}$, Lemma \ref{lem3-4} yields that
 	\begin{equation*}
 		\bbE \left[ \sqrt{{\mathcal{V}(\widetilde{\bfY}_{t_\ell}^{\bfy})} + d} \right]
 		\leqslant \sqrt{\bbE \left[ {\mathcal{V}(\widetilde{\bfY}_{t_\ell}^{\bfy})} \right] + d}
 		\leqslant C \sqrt{\mathcal{V} (\bfy) + d}.
 	\end{equation*} 
	This, together with \eqref{eq:Ik2} and \eqref{e:ConEst}, implies
 	\begin{equation} \label{eq:thm4pr2}
 		\abs{I_{k,2}}
 		\leqslant C \infnorm{g_2} d^{7/2} \sqrt{\mathcal{V}(\bfy) + d} \left( \eta_k^{{3}/{2}} + \frac{\eta_k}{\sqrt{N}} \right).
 	\end{equation}
 	It follows from \eqref{eq:Ik1}, \eqref{eq:thm4pr2} and $\infnorm{g_2} \leqslant \infnorm{h}$ that
 	\begin{equation*}
 		\begin{split}
 			\abs{\mathcal{J}_2}
 		&\leqslant C \infnorm{h} d^{7/2} \sqrt{\mathcal{V} (\bfy) + d}  \sum_{k = k_n}^{n} \left( \eta_k^{{3}/{2}} + \frac{\eta_k}{\sqrt{N}} \right) \\
 		&\leqslant C \rme^{\theta} \infnorm{h} d^{7/2} \sqrt{\mathcal{V} (\bfy) + d}  \sum_{k = k_n}^{n} \left( \eta_k^{{3}/{2}} + \frac{\eta_k}{\sqrt{N}} \right) \rme^{-\theta (t_n - t_k)},
 		\end{split}
 	\end{equation*}
 	where in the second inequality we use $t_n - t_k < 1$  such that $\rme^{-\theta(t_n - t_k)} \geqslant \rme^{-\theta}$ for each $k_n \leqslant k \leqslant n$. Hence, \eqref{eq:J2-TV} is verified.
 	
	We complete the proof.
 	\end{proof}
 
 \subsection{Proofs of Corollaries \ref{Cor-W1} and \ref{Cor-TV}}
 
 \begin{proof} 
 	[Proof of Corollary \ref{Cor-W1} ]	
 	Actually, by \eqref{eq:W1 sum}, we only need to analyze 
 	\[
 		\sum_{k=1}^n \eta_k^{3/2} \rme^{-\theta(t_n - t_k)} , \quad \eta_k = \frac{\eta}{ k^{\alpha} },\ \alpha \in (0, 1) .
 	\]
 	
 	Observe that
 	\begin{equation*}
 		\eta_{k-1} - \eta_{k} = \eta \cdot \frac{k^{\alpha} - (k-1)^{\alpha} }{ k^{\alpha} (k-1)^{\alpha} }.
 	\end{equation*}
 	Notice that $k^{\alpha} - (k-1)^{\alpha} \to 0$ and $\sqrt{\eta_k} \rme^{\theta t_k} \to +\infty$ as $k \to \infty$, so there exists $n_0 \in \mathbb{N}$ such that for all $k \geqslant n_0$
 	\begin{equation*}
 		\eta_k \leqslant \frac{2\theta - \omega}{2 \theta^2}, \quad \eta_{k-1} - \eta_{k} \leqslant \omega \eta_k^{2} \quad \text{and} \quad
 		\sqrt{\eta_k} \rme^{\theta t_k} \geqslant 1 .
 	\end{equation*}	
 	It is obvious that the desired result holds for $n \leqslant n_0$. And for $n > n_0$, we split the sum term in \eqref{eq:W1 sum} into two parts:
 	\begin{equation} \label{eq:sum1<1}
 			\begin{split}
 				\sum_{k=1}^{n_0} \eta_k^{{3}/{2}} \rme^{-\theta (t_n - t_k)}
 			&\leqslant C \rme^{-\theta t_n}
 			\leqslant C \sqrt{\eta_n}
 			= C \frac{\eta^{{1}/{2}} }{n^{ {\alpha}/{2}} }, \\
 			\sum_{k = n_0 + 1}^n \eta_k^{{3}/{2}} \rme^{-\theta (t_n - t_k)}
 			&\leqslant C \sqrt{\eta_n} 
 			= C \frac{\eta^{{1}/{2}} }{n^{ {\alpha}/{2}} },
 			\end{split}
 	\end{equation} 
 	where we use Lemma \ref{lem-A-3} to estimate the second part. 	
 	Substituting \eqref{eq:sum1<1} into \eqref{eq:W1 sum} implies the desired result.
 \end{proof}
 
  \begin{proof} 
 	[Proof of Corollary \ref{Cor-TV}]	
 	By \eqref{eq:TV sum}, the proof is similar to the proof of Corollary \ref{Cor-W1}. We only need to analyze the following sum additionally,
 	\[
 		 \sum_{k=1}^n \eta_k \rme^{-\theta(t_n - t_k)}, \quad \eta_k = \frac{\eta}{ k^{\alpha} }, \ \alpha \in (0, 1) .
 	\]
 	
 	Using the same notations in the proof of Corollary \ref{Cor-W1}, we have
 	\begin{equation} \label{eq:sum2<1 TV}
 		\sum_{k=1}^n \eta_k \rme^{-\theta (t_n - t_k)} = \sum_{k=1}^{n_0} \eta_k \rme^{-\theta (t_n - t_k)} + \sum_{k = n_0 + 1}^n \eta_k \rme^{-\theta (t_n - t_k)} \leqslant C.
 	\end{equation} 	
 	Substituting \eqref{eq:sum1<1} and \eqref{eq:sum2<1 TV} into \eqref{eq:TV sum} implies the desired result.
 \end{proof}
 
 \subsection{Proof of Corollary \ref{Cor-Error}}
 
 \begin{proof}
 	[Proof of Corollary \ref{Cor-Error}]
 	Recall $f(\bfx)=\bbE F (\bfx, \xi)$ for any $\bfx \in \bbR^d$. 
 	As $\xi$ is independent of $(\disposition_{t_n})_{n \geqslant 1}$, we have the following decomposition:
 	\begin{align*}
 		&\mathrel{\phantom{=}} \bbE F (\disposition_{t_n}, \xi) - \min_{\bfx\in\bbR^d} \bbE F (\bfx, \xi) \\
 		&= [ \bbE f(\disposition_{t_n}) - \bbE f(\conposition_{t_n}) ]+ [ \bbE f(\conposition_{t_n}) - \bbE_{\bm{\pi}} f(\conposition) ] + [ \bbE_{\bm{\pi}} f(\conposition) -  f(\bfx^*) ],
 	\end{align*}
 	where $\bfx^* = \arg\min f(\bfx)$. According to the conclusion of exponential ergodicity \cite{MR1234295}, the second term on the right-hand side can be estimated by
 	\begin{equation*}
 		\abs{ \bbE f(\conposition_t) - \bbE_{\bm{\pi}} f(\conposition) } \leqslant B_e ( \mathcal{V} (\bfm_0,\bfx_0) + 1 ) \rme^{-\kappa t} ,
 	\end{equation*}
 	for some positive constants $B_e$ and $\kappa$. 
 	As for the last term, Proposition 11 in \cite{pmlr-v65-raginsky17a} yields that
 	\begin{equation*}
 		\int_{\bbR^d} f(\bm{\omega}) \bm{\pi}_{\bfx} (\dd \bm{\omega}) - \min_{\bfx\in\bbR^d} f(\bfx) \leqslant \frac{d\beta^2}{4} \log\left[ \frac{2\rme L}{a} \left( \frac{2K}{d\beta^2} + 1 \right) \right] 
 	\end{equation*}
 	as $\beta^2\leqslant a/2$, where $\bm{\pi}_{\bfx} $ is the marginal distribution of $\bm{\pi}$. So it remains to estimate the first term by Theorem \ref{Thm-1}.
 	
 	Pick a localization function $\psi \in C^\infty (\bbR^d, \bbR)$ satisfying $0 \leqslant \psi (\bfx) \leqslant 1$ and
 	\begin{equation*}
 		\psi (\bfx) = \begin{cases}
 			1, &\abs{\bfx} \leqslant 1, \\
 			0, &\abs{\bfx} \geqslant 2.
 		\end{cases}
 	\end{equation*}
 	For $n \geqslant 1$, denote
 	\begin{equation*}
 		g_n (\bfx) = f(\bfx) \psi \left( \frac{\bfx}{R_n} \right),
 	\end{equation*}
 	with some constants $R_n > 0$ determined later, then we have
 	\begin{align*}
 		\abs{\nabla g_n (\bfx)}
 		&= \abs{ \psi \left( \frac{\bfx}{R_n} \right) \nabla f (\bfx) + \frac{1}{R_n} f (\bfx) \nabla \psi \left( \frac{\bfx}{R_n} \right) } \\
 		&\leqslant \left( \abs{\nabla f (\bfx)} + \frac{\norm{\nabla \psi}_\infty}{R_n} \abs{f (\bfx)} \right) \mathbf{1}_{[0, 2 R_n)} (\abs{\bfx}).
 	\end{align*}
 	Recall that $\abs{f(\bfx)} \leqslant C \abs{\bfx}^2$ and $\abs{\nabla f(\bfx)} \leqslant C \abs{\bfx}$ for all $\bfx \in \bbR^d$ satisfying $\abs{\bfx} \geqslant 1$, so $\abs{\nabla g_n (\bfx)} \leqslant C R_n$ for any $\bfx \in \bbR^d$, i.e., $\norm{g_n}_{\Lip} \leqslant C R_n$. It follows from Lemma \ref{lem3-4} that 
 	\begin{align*}
 		\abs{\bbE f(\disposition_{t_n}) - \bbE g_n (\disposition_{t_n})}
 		&\leqslant \bbE \abs{ f(\disposition_{t_n}) \mathbf{1}_{(R_n,+\infty)} \big(\big|\disposition_{t_n}\big|\big)} \\
 		&\leqslant C \bbE \left[ \big|\disposition_{t_n}\big|^2\mathbf{1}_{(R_n,+\infty)} \big(\big|\disposition_{t_n}\big|\big) \right] \\
 		&\leqslant C R_n^{2 - 2p} \bbE \left[ \big|{\disposition_{t_n}}\big|^{2p} \mathbf{1}_{(R_n,+\infty)} \big(\big|\disposition_{t_n}\big|\big) \right] \\
 		&\leqslant C R_n^{2 - 2p} \big( \mathcal{V}(\bfm_0, \bfx_0)^p + d^p \big).
 	\end{align*}
 	Similarly, the following holds by Lemma \ref{lem3-1}
 	\begin{equation*}
 		\abs{\bbE f(\conposition_{t_n}) - \bbE g_n (\conposition_{t_n})}
 		\leqslant C R_n^{2 - 2p} \big( \mathcal{V}(\bfm_0, \bfx_0)^p + d^p \big).
 	\end{equation*}
 	Thus, by Theorem \ref{Thm-1} we have 
 	\begin{align*}
 		&\pheq
 		\abs{\bbE f(\disposition_{t_n}) - \bbE f(\conposition_{t_n})} \\
 		&\leqslant \abs{\bbE f(\disposition_{t_n}) - \bbE g_n (\disposition_{t_n})} + \abs{\bbE g_n (\disposition_{t_n}) - \bbE g_n (\conposition_{t_n})} + \abs{\bbE g_n (\conposition_{t_n}) - \bbE f(\conposition_{t_n})} \\
 		&\leqslant C R_n^{2 - 2p} \big( \mathcal{V}(\bfm_0, \bfx_0)^p + d^p \big) + C \norm{g_n}_{\Lip} { \left(1+\frac1N\right) \sqrt{\eta_n ( \mathcal{V}(\bfm_0,\bfx_0) + d )} }\\
 		&\leqslant C R_n^{2 - 2p} \big( \mathcal{V}(\bfm_0, \bfx_0)^p + d^p \big) + C R_n { \left(1+\frac1N\right) \sqrt{\eta_n ( \mathcal{V}(\bfm_0,\bfx_0) + d )} }.
 	\end{align*}
 	Since the above equation holds with arbitrary $R_n>0$, we can take 
 	\[
 		R_n =  \eta_n^{1 / (2 - 4p)} \sqrt{\mathcal{V}(\bfm_0,\bfx_0) + d} .
 	\]
 	Let $2p = q$, we obtain the desired. 
 \end{proof}

\newpage
	
	\bibliographystyle{alpha}
	\newcommand{\etalchar}[1]{$^{#1}$}

\newpage

\section*{Appendix}

\appendix

\section{Supporting Lemmas}
\label{Appendix:A}
The first lemma shows that $f(\bfx)$ admits lower and upper bounds that are quadratic functions.
	
	\begin{lemma}{\cite[Lemma 2]{pmlr-v65-raginsky17a} }
		If function $f(\bfx)$ satisfies Assumption \ref{A1-1}, then the following quadratic bounds hold,
		\[
			\frac{a}{2} \abs{\bfx}^2 - \frac{K}{2} \log{3} \leqslant f(\bfx) \leqslant L \abs{\bfx}^2 + \frac{B^2}{2L} + A, \quad \forall \ \bfx \in \bbR^d, 
		\]
		where $K = b + B^2 / (2a)$.
	\end{lemma}

	Applying Assumption \ref{Assump.1} $(\runum{3})$, it is easy to verify that
		\begin{equation}
		\label{eq:dissipation-f}
		\inprod{\bfx, \nabla f(\bfx)} \geqslant
		\frac{a}{2} \abs{\bfx}^2 - K
		\geqslant 2\lambda \left( f(\bfx)+\frac14 \gamma^2 \abs{\bfx}^2 \right) - \mathring{A}\ ,\quad \forall\ \bfx\in\bbR^d,
		\end{equation}
		where the constants $\lambda$ and $\mathring{A}$ satisfy
		\begin{equation} \label{def:lambda A}
		0 < \lambda \leqslant \min \left\{\frac14,  \frac{a}{4L+\gamma^2} \right\}, \quad
		\mathring{A} \geqslant K+2\lambda \left( \frac{B^2}{2L}+A \right).  
		\end{equation}
	Applying \eqref{eq:dissipation-f}, the next lemma shows that the Lyapunov function $\mathcal{V}(\bfm, \bfx)$ defined in \eqref{def-Lya.} satisfies \eqref{eq:inf-Lya}. 
	
	\begin{lemma}\label{lem-A-2}
		Under Assumption \ref{A1-1}, and $\lambda $ and $\mathring{A}$ are in \eqref{def:lambda A}, we have that 
		\[
		\mathscr{A} \mathcal{V}(\bfm , \bfx) \leqslant -\lambda \gamma \mathcal{V}(\bfm, \bfx) +  ( \gamma \mathring{A} + d \beta^2 )/2.
		\]
	\end{lemma}
	\begin{proof}
		Recall that
		\begin{equation*} 
			\mathcal{V}(\bfm, \bfx) := f(\bfx) + \frac{\gamma^2}{4} \left( \Big|\bfx+\frac1{\gamma} \bfm \Big|^2 + \Big| \frac1{\gamma}\bfm \Big|^2 - \lambda\abs{\bfx}^2 \right),
		\end{equation*}
		and we have
		\begin{align*}
		\nabla_{\bfm} \mathcal{V}(\bfm, \bfx) &= \bfm +\frac{\gamma}{2} \bfx , \quad \Delta_{\bfm} \mathcal{V}(\bfm, \bfx) = d ,\\
		\nabla_{\bfx} \mathcal{V}(\bfm, \bfx) &=  \nabla f(\bfx) + \frac{\gamma^2(1-\lambda)}{2} \bfx + \frac{\gamma}{2} \bfm .
		\end{align*}
		Hence, by inequality \eqref{eq:dissipation-f}, we have
		\begin{equation*}
		\begin{split}
		\mathscr{A} \mathcal{V}(\bfm, \bfx) &= -\inprod{ \nabla_{\bfm} \mathcal{V}(\bfm, \bfx), \gamma \bfm + \nabla f(\bfx) } + \inprod{ \nabla_{\bfx}\mathcal{V}(\bfm, \bfx) , \bfm} + \frac12 \beta^2 \Delta_{\bfm} \mathcal{V}(\bfm, \bfx) \\
		&\leqslant - \left[ \lambda \gamma \big( f(\bfx) + \frac{\gamma^2}{4} \abs{ \bfx }^2 \big) \right] + \frac{\gamma}{2} \mathring{A} - \frac{\lambda \gamma^2}{2} \inprod{\bfm, \bfx} - \frac{\gamma}{2} \abs{\bfm}^2 + \frac12 d \beta^2 \\
		&= -\lambda \gamma \mathcal{V}(\bfm, \bfx) - \frac{\lambda^2\gamma^3}{4} \abs{\bfx}^2 - \frac{(1-\lambda)\gamma}{2} \abs{\bfm}^2 + \frac12\left( \gamma \mathring{A} + d \beta^2 \right) ,
		\end{split}
		\end{equation*}
		where the inequality comes from \eqref{eq:dissipation-f}. Since $\lambda \leqslant 1 / 4$, we complete the proof.
	\end{proof}
	
		The process $\left( \conmomentum_t, \conposition_t \right)_{t \geqslant 0} $ admits a unique stationary distribution as
		\begin{equation*}
		\bm{\pi}(\dd \bfm, \dd \bfx) = \frac{1}{\,Z\,} \exp\left\{ -\frac{\gamma}{\beta^2} \left(  \abs{\bfm}^2 + 2 f(\bfx) \right) \right\} \dd\bfm \dd\bfx,
		\end{equation*} 
		where $Z$ is a normalized constant given by
		\begin{equation*}
		Z = \left( \frac{\beta^2 \pi}{\gamma} \right)^{{d}/{2}} \int_{\bbR^d} \rme^{-{2} \gamma f(\bfx) / {\beta^2}} \dd\bfx, 
		\end{equation*}
		see e.g. \cite{MR3288096}. Thanks to \cite{MR1234295}, Lemma \ref{lem-A-2} yields that  process $\left( \conmomentum_t, \conposition_t \right)_{t \geqslant 0} $ is exponentially ergodic. That is, there exist some positive constants $B_e$ and $\kappa$, such that for all $ \bfm, \bfx \in \bbR^d$,
		\begin{equation*}
		\sup_{g \leqslant \mathcal{V}+1} \Big| \bbE \left[ g(\conmomentum_t^{\bfm}, \conposition_t^{\bfx}) \right] - \bbE_{\bm{\pi}} \left[g(\conmomentum,\conposition)\right] \Big| \leqslant B_e ( \mathcal{V}(\bfm, \bfx) + 1 ) \rme^{-\kappa t} , \quad \forall \ t \geqslant 0,  
		\end{equation*}
		where $(\conmomentum, \conposition)$ has the distribution $\bm{\pi}$. By the definition of $\mathcal{V}(\bfm,\bfx)$, this immediately implies that for all $\bfx \in \bbR^d$,
		\begin{equation}
		\label{eq:ExpErg}
		\abs{ \bbE f(\conposition_t^{\bfx}) - \bbE_{\bm{\pi}}[f(\conposition)] } \leqslant B_e ( \mathcal{V}(\bfm, \bfx) + 1 )  \rme^{-\kappa t}, \quad \forall t \geqslant 0.
		\end{equation}
	
	\vskip 3mm

	\begin{lemma}
			\label{lem-A-3}
			Let step size $(\eta_k)_{k \geqslant 1}$ satisfy Assumption \ref{Assump.3} and $t_n = \sum_{i = 1}^n \eta_i$. For any $\varepsilon \in [0,1/2]$, the following holds
			\[
			\sum_{k = 1}^n \eta_k^{1+\varepsilon} \rme^{-\theta ( t_n - t_k) } \leqslant \frac{4}{2  \theta - (4 \varepsilon - 1) \omega} \eta_n^{\varepsilon} .
			\]
		\end{lemma}
		
		\begin{proof}
			Notice that for each $k\leqslant n$
			\begin{equation} \label{eq:1}
			\eta_k^{1+\varepsilon} \rme^{-\theta (t_n - t_k)}
			= \rme^{-\theta t_n} \cdot \eta_k^{1+\varepsilon} \frac{ \rme^{\theta t_k}-\rme^{\theta t_{k-1}} }{ 1-\rme^{-\theta \eta_k} } 
			\leqslant \frac{4  \rme^{-\theta t_n}}{2 \theta +  \omega} \cdot \eta_k^{\varepsilon} \left( \rme^{\theta t_k}-\rme^{\theta t_{k-1}} \right),
			\end{equation}
			where in the last inequality, we use the fact
			\begin{equation*}
			\rme^{-x} \leqslant 1-x+ \frac{x^2}{2} \leqslant 1 - x + \frac{2  \theta -  \omega}{4  \theta} x = 1 - \frac{2 \theta +  \omega}{4  \theta} x , \quad \forall \ 0 \leqslant x \leqslant 1 - \frac{ \omega}{2 \theta} ,
			\end{equation*}
			to obtain that 
			\[
			1 - \rme^{-\theta \eta_k} \geqslant \frac{2  \theta +  \omega }{4}\eta_k , \quad \forall\ \eta_k \leqslant \frac{ 2  \theta -  \omega }{2  \theta^2} .
			\]
			Besides, according to condition $\eta_{k-1} - \eta_k \leqslant \omega \eta_k^2$ and Bernoulli's inequality, it holds that
			\[
			\eta_{k-1}^{\varepsilon} - \eta_k^{\varepsilon} = \eta_k^{\varepsilon} \left[ \Big[ \frac{\eta_{k-1}}{\eta_k} \Big]^{\varepsilon} - 1 \right] 
			\leqslant \varepsilon \eta_k^{\varepsilon} \left[ \frac{\eta_{k-1}}{\eta_k} - 1\right] 
			\leqslant \varepsilon \omega \eta_k^{1+\varepsilon} ,
			\] 
			which derives that
			\begin{equation}\label{eq:2}
			\begin{split}
			\sum_{k = 1}^n \eta_k^{\varepsilon} \left( \rme^{\theta t_k}-\rme^{\theta t_{k-1}} \right)
			&= \sum_{k = 1}^n \left( \eta_k^{\varepsilon} \rme^{\theta t_k} - \eta_{k-1}^{\varepsilon} \rme^{\theta t_{k-1}} \right) + \sum_{k = 1}^n \left( \eta_{k-1}^{\varepsilon} - \eta_k^{\varepsilon} \right) \rme^{\theta t_{k-1}} \\
			&\leqslant \eta_n^{\varepsilon} \rme^{\theta t_n} + \varepsilon {\omega} \sum_{k = 1}^n \eta_k^{1+\varepsilon} \rme^{\theta t_k}.
			\end{split}
			\end{equation}
			Combining \eqref{eq:1} and \eqref{eq:2} shows that
			\begin{equation*}
			\sum_{k=1}^n \eta_k^{1+\varepsilon} \rme^{-\theta (t_n - t_k)}
			\leqslant \frac{4}{2  \theta + \omega} \eta_n^{\varepsilon} + \frac{ 4 \varepsilon \omega}{2  \theta +  \omega} \sum_{k=1}^n \eta_k^{1+\varepsilon} \rme^{-\theta (t_n - t_k)},
			\end{equation*}
			and a simply calculation yields the desired.
		\end{proof}

	\section{Proofs of Auxiliary Lemmas}	
	\label{Appendix:B}
	
	\begin{lemma} \label{lem-B-1}
		Let $\bfU^i$, $i = 1, \dots, N$ be i.i.d.\ random vectors in $\bbR^d$ satisfying $\bbE \bfU^1 = \zero$ and $\bbE [ \abs{\bfU^1}^{2p} ] < + \infty$ for some $p \in \bbN$. Then there exists a constant $C > 0$ only depending on $p$, such that
		\begin{equation*}
			\bbE \bigg[ \bigg| {\frac{1}{N} \sum_{i = 1}^N \bfU^i} \bigg|^{2p} \bigg]
			\leqslant \frac{C}{N^p} \bbE \left[ \abs{\bfU^1}^{2p} \right].
		\end{equation*}
	\end{lemma}
	
	\begin{proof}
		Notice that
		\begin{equation*}
			{ \abs{\sum_{i = 1}^N \bfU^i}^{2p}
			= \prod_{j = 1}^p \sca{\sum_{i_{2j-1} = 1}^N \bfU^{i_{2j-1}}, \sum_{i_{2j} = 1}^N \bfU^{i_{2j}}}
			= \sum_{i_1, \dots, i_{2p} = 1}^N \prod_{j = 1}^p \sca{\bfU^{i_{2j-1}}, \bfU^{i_{2j}}}, }
		\end{equation*}		
		which implies that
		\begin{equation*}
			{ \bbE \bigg[ \bigg|{\frac{1}{N} \sum_{i = 1}^N \bfU^i}\bigg|^{2p}\bigg]
			= \frac{1}{N^{2p}} \sum_{i_1, \dots, i_{2p} = 1}^N \bbE \bigg[ \prod_{j = 1}^p \sca{\bfU^{i_{2j-1}}, \bfU^{i_{2j}}} \bigg]. }
		\end{equation*}
		Since $\bfU^1, \dots, \bfU^N$ are independent and $\bbE \bfU^1 = \zero$, we have $\bbE [ \prod_{j = 1}^p \sca{\bfU^{i_{2j-1}}, \bfU^{i_{2j}}} ] = 0$ whenever some $i \in \{1, \dots, N\}$ appears only once among $i_1, \dots, i_{2p}$. For the rest terms, H\"older's inequality shows that
		\begin{equation*}
			{ \bbE \bigg[ \prod_{j = 1}^p \sca{\bfU^{i_{2j-1}}, \bfU^{i_{2j}}} \bigg] \leqslant \bbE \bigg[ \prod_{j = 1}^{2p} \abs{\bfU^{i_j}} \bigg] \leqslant \prod_{j = 1}^{2p} \left\{ \bbE \left[ \abs{\bfU^{i_j}}^{2p} \right] \right\}^{1 / 2p} = \bbE \left[ \abs{\bfU^1}^{2p} \right]. }
		\end{equation*}
		Furthermore, for $l = 1, \dots, p$, the index set $\{1, \dots, 2p\}$ can be decomposed into $l$ parts such that each part consists of at least two indices, and $c_l$ denotes the number of such partitions. Then we have
		\begin{equation*}
			{ \sum_{i_1, \dots, i_{2p} = 1}^N \bbE \bigg[ \prod_{j = 1}^p \sca{\bfU^{i_{2j-1}}, \bfU^{i_{2j}}} \bigg]
			\leqslant \sum_{l = 1}^p c_l N^l \bbE \left[ \abs{\bfU^1}^{2p} \right].}
		\end{equation*}
		The desired result follows from
		\begin{equation*}
			{ \bbE \bigg[ \bigg| {\frac{1}{N} \sum_{i = 1}^N \bfU^i} \bigg|^{2p}\bigg]
			\leqslant \frac{1}{N^{2p}} \sum_{l = 1}^p c_l N^l \bbE \left[ \abs{\bfU^1}^{2p} \right]
			\leqslant \frac{C}{N^p} \bbE \left[ \abs{\bfU^1}^{2p}\right].  }
		\end{equation*}
	\end{proof}
	
	\subsection{Proofs of Lemmas for Moments Estimations}
	\label{Appen-B-1}
	
	\begin{proof}[\textbf{Proof of Lemma \ref{lem3-1}}]
		Recalling the Lyapunov function $\mathcal{V}(\bfm,\bfx)$ in \eqref{def-Lya.}, we define function $F_t$ as 
			\begin{equation*}
			F_t := \bbE \mathcal{V}(\conmomentum_{t}^{\bfm}, \conposition_{t}^{\bfx}) =
			\bbE \left[ f(\conposition_t^{\bfx}) + \frac{\gamma^2}{4} \left( \abs{\conposition_t^{\bfx}+\frac1{\gamma} \conmomentum_t^{\bfm}}^2 +\abs{\frac1{\gamma}\conmomentum_t^{\bfm}}^2 - \lambda\abs{\conposition_t^{\bfx}}^2 \right) \right].
			\end{equation*}
			It\^o's formula implies that $F_t$ satisfies the following differential inequality
			\begin{equation*}
			\frac{\dd F_t}{\dd t} \leqslant -\lambda \gamma \, F_t + \frac12 \left( \gamma \mathring{A} +d \beta^2 \right),
			\end{equation*}
			with initial condition $F_0 =  \mathcal{V}(\bfm,\bfx)$. Hence, we can obtain that
			\begin{equation}
			\label{eq:bound of F_t}
			F_t \leqslant \rme^{-\lambda \gamma t} F_0 + \frac1{ 2 \lambda \gamma} \left( \gamma \mathring{A} +d \beta^2 \right), 
			\end{equation}
		which yields the result with respect to of $p=1$ immediately. As for the case of $p>1$, we can apply similar argument. By It\^o's formula, it holds
		\begin{equation*}
		\begin{split}
		\mathscr{A} \left( \mathcal{V}(\bfm,\bfx)^p \right)& \leqslant    p \mathcal{V}(\bfm,\bfx)^{p-1} \mathscr{A} \mathcal{V}(\bfm,\bfx) + \frac{ p(p-1)\beta^2 }{2} \mathcal{V}(\bfm,\bfx)^{p-2} \abs{ \bfm + \frac{ \gamma}{2}\bfx }^2\\
		&\leqslant -p \lambda\gamma\mathcal{V}(\bfm,\bfx)^p + C \mathcal{V}(\bfm,\bfx)^{p-1},
		\end{split}
		\end{equation*}
		where the last inequality is because of \eqref{eq:bounds} and  \eqref{eq:inf-Lya}. 
		Then, by Young's inequality, the function $G_t := \bbE \left[ \mathcal{V}(\conmomentum_t^{\bfm}, \conposition_t^{\bfx})^p \right]$ satisfies 
		\begin{equation*}
		\frac{\dd G_t}{\dd t} \leqslant -\lambda \gamma G_t + C, 
		\end{equation*}
		with initial value $G_0 = \mathcal{V}(\bfm,\bfx)^p$, where $C>0$ here also depends on $p$. This implies the desired immediately.		
	\end{proof}
	
	\begin{proof}[\textbf{Proof of Lemma \ref{lem3-2} }]
		 Under Assumption \ref{Assump.1}, we have
			\[
			\begin{split}
			&\mathrel{\phantom{\le}}
			\bbE \abs{\conmomentum_t^{\bfm}-\bfm}^{2p}
			= \bbE \abs{ \int_{0}^t \left[ -\gamma \conmomentum_s - \nabla f(\conposition_s) \right] \dd s +  \beta \bfB_t  }^{2p} \\
			&\leqslant C t^{2p-1} \bbE  \int_{0}^t \abs{ \gamma \conmomentum_s + \nabla f(\conposition_s) }^{2p} \dd s  + C \bbE \abs{\BM_t}^{2p} \\ 
			&\leqslant C t^{2p-1} \int_0^t \left( \gamma^{2p} \bbE\abs{\conmomentum_s}^{2p} + L^{2p} \bbE\abs{\conposition_s}^{2p} + B^{2p} \right) \dd s + C t^p d^p, \\
			&\leqslant C t^{2p} ( \mathcal{V}(\bfm, \bfx)^p + d^p ) + C t^p d^p ,
			\end{split}
			\]
			where the first inequality is by the H\"older's inequality, the second inequality is by \eqref{eq:f-linear} and the last one is due to \eqref{eq:2pm}. Similarly, it holds 
			\[
			\bbE \abs{\conposition_t^{\bfx} - \bfx}^{2p} = \bbE \abs{ \int_0^t \conmomentum_s \dd s }^{2p} \leqslant t^{2p-1} \int_0^t \bbE \abs{\conmomentum_s}^{2p} \dd s \leqslant C t^{2p} ( \mathcal{V}(\bfm, \bfx)^p + d^p ).
			\]
			Combining above two inequalities, we can obtain the desired.
			 
	\end{proof}
	
	\begin{proof}[\textbf{Proof of Lemma \ref{lem3-3} }]
		Recall SDE \eqref{SDE2-0} and rewrite it as 
		\begin{equation*}
		\left\{
		\begin{aligned}
		\dismomentum_{t}^{\bfm} - \bfm 
		&= -\gamma t \bfm - t \nabla f(\bfx) + t \bigg( \nabla f(\bfx ) - \frac1N \sum_{i = 1}^N \nabla F(\bfx, \xi^i) \bigg) + \beta \bfB_{t}\ , \\
		\disposition_{t}^{\bfx} - \bfx &= t \bfm\ .
		\end{aligned}
		\right.
		\end{equation*}
		for all $t\in[0,\eta]$ with $\eta$ being a step size. To make notations simple, denote 
		\begin{equation}  \label{e:gi}
		\bm{\Lambda} (\bfx, \xi^i) = \nabla f(\bfx) - \nabla F(\bfx,\xi^i),  \quad i =1, \dotsc, N,
		\end{equation} 
		which are i.i.d.\ random vectors. Then, Assumption \ref{Assump.2} and Lemma \ref{lem-B-1} imply that 
		\begin{equation} \label{e:underA2}
			\bbE \abs{ \frac1N \sum_{i = 1}^N \bm{\Lambda} (\bfx, \xi^i) }^{2p} \leqslant \frac{ C }{N^p} \bbE \abs{\bm{\Lambda} (\bfx, \xi^i)}^{2p}
			\leqslant \frac{ C }{N^p}, \quad \forall \ 1 \leqslant p \leqslant \frac{q}{2}.
		\end{equation}
		If the condition \eqref{A4-2} in Assumption \ref{Assump.4} holds additionally, \eqref{e:underA2} holds for all $1 \leqslant p \leqslant q^{\prime}/2$. 
		Thus, we have that 
		\begin{equation*}
		\begin{split}
		&\pheq
		\bbE \abs{ \dismomentum_{t}^{\bfm} - \bfm }^{2p} \\
		&\leqslant C \Big[  t^{2p} \bbE \big|{ \gamma\bfm + \nabla f(\bfx) }\big|^{2p} + 
		t^{2p} \bbE \big|{  N^{-1} \sum\nolimits_{i = 1}^d \bm{\Lambda} (\bfx,\xi^i) }\big|^{2p} 
		+ \beta^{2p} \bbE \abs{\bfB_t}^{2p} \Big] \\
		&\leqslant C \left[ \abs{ \bfm }^{2p} + \abs{ \bfx }^{2p} + 1 \right] t^{2p} + C t^p d^p .
		\end{split}
		\end{equation*}
		Using \eqref{eq:bounds} , we can obtain 
		\begin{equation*}
		\bbE \abs{ \dismomentum_{t}^{\bfm} - \bfm }^{2p} \leqslant C t^p ( \mathcal{V}(\bfm,\bfx)^p + d^p ).
		\end{equation*}
		Similarly, we have 
		\begin{equation*}
		\bbE \abs{\disposition_{t}^{\bfx} - \bfx}^{2p} = t^{2p} \bbE \abs{\bfm}^{2p} \leqslant C t^{2p} \mathcal{V}(\bfm,\bfx)^p. 
		\end{equation*}
		Hence, by $0\leqslant t \leqslant \eta <1$, combining all above yields the desired.		
	\end{proof}
	
	\begin{proof}[\textbf{Proof of Lemma \ref{lem3-4}} ]
		Consider SDE \eqref{SDE2}, let us prove the case of $p=1$ firstly. 
		By It\^o's formula, $(\mathcal{V} (\dismomentum_t, \disposition_t))_{0 \leqslant  t \leqslant  \eta}$ satisfies
		\begin{equation}
			\label{eq:decom.d V}
			\begin{split}
				\dd \mathcal{V} (\dismomentum_t, \disposition_t) 
				&=  - \frac{\gamma}{2} \sca{\dismomentum_t, \bfm}\! \dd t - \frac{\lambda \gamma^2}{2} \sca{\disposition_t, \bfm}\! \dd t
				+ \sca{\nabla f(\disposition_t), \bfm} \dd t + \frac{d \beta^2}{2}  \dd t \\
				&\pheq -  \sca{ \dismomentum_t + \frac{\gamma}{2} \disposition_t, \frac1N \sum\nolimits_{i = 1}^N\nabla F(\bfx, \xi^i )}\! \dd t 
				+ \beta\inprod{  \dismomentum_t + \frac{\gamma}{2} \disposition_t  , \dd \bfB_t }   \\
				&=:  [ I_1 + I_2 ] \dd t + \beta\inprod{  \dismomentum_t + \frac{\gamma}{2} \disposition_t  , \dd \bfB_t } ,
			\end{split}
		\end{equation}
		where $I_1$ and $I_2$ are given by
		\begin{equation*}
		\begin{split}
		I_1
		&= - \frac{\gamma}{2} \abs{\dismomentum_t}^2 - \frac{\lambda \gamma^2}{2} \sca{\disposition_t, \dismomentum_t} - \frac{\gamma}{2} \sca{\disposition_t, \nabla f(\disposition_t)} + \frac{d \beta^2}{2}, 
		\\
		I_2
		&= \frac{\gamma}{2} \sca{\dismomentum_t  + \lambda \gamma \disposition_t, \dismomentum_t - \bfm} - \sca{\nabla f(\disposition_t), \dismomentum_t - \bfm} \\
		&\pheq + \sca{ \dismomentum_t + \frac{\gamma}{2} \disposition_t, \nabla f(\disposition_t) - \nabla f(\bfx)  } 
		+ \sca{ \dismomentum_t + \frac{\gamma}{2} \disposition_t , \frac1N \sum\nolimits_{i = 1}^N \bm{\Lambda} (\bfx,\xi^i) }.
		\end{split}
		\end{equation*}
		Here, $\bm{\Lambda} (\bfx, \xi^i)$, $i = 1, \dotsc, N$ are defined by \eqref{e:gi}. 
		
		For $I_1$, by the definition of $\mathcal{V}$ in \eqref{def-Lya.} and \eqref{eq:dissipation-f} , we have 
		\begin{equation}
			\label{ineq:I1}
			\begin{split}
				I_1 
				&\leqslant - \left[ \lambda \gamma \Big( f( \disposition_t ) + \frac{\gamma^2}{4} | \disposition_t|^2  \Big) \right] - \frac{\lambda \gamma^2 }{2} \inprod{ \dismomentum_t , \disposition_t } - \frac{\gamma}{2}| \dismomentum_t |^2 + \frac12 (\gamma \mathring{A} + d \beta^2) \\ 
				&\leqslant    -\lambda \gamma \mathcal{V} (\dismomentum_t, \disposition_t) + \frac12 \left( \gamma \mathring{A} +d \beta^2 \right) . 
			\end{split}
		\end{equation}
		
		For $I_2$, the Young's inequality yields that for any $\varepsilon > 0$
		\begin{align*}
		I_2
		&\leqslant  \varepsilon \left[ \frac{\gamma}{2}  \abs{ \dismomentum_t + \lambda \gamma \disposition_t }^2  + 2 \left( L^2 \abs{\disposition_t}^2 + B^2 \right) + 2 \abs{ \dismomentum_t + \frac{\gamma}{2} \disposition_t }^2  \right] \\
		&\pheq +  \frac{1}{4\varepsilon}  \left[ \frac{\gamma+2}{2} \abs{ \dismomentum_t - \bfm }^2 + L^2  \abs{ \disposition_t - \bfx }^2  \right]
		+   \frac{1}{4\varepsilon} \abs{ \frac1N \sum\nolimits_{i = 1}^N \bm{\Lambda} (\bfx,\xi^i) }^2  
		.
		\end{align*}
		Thus, by choosing a special $\varepsilon$ and \eqref{eq:bounds}, we can obtain
		\begin{equation}
			\label{ineq:I2}
			I_2 \leqslant \frac{\lambda \gamma}{3} \mathcal{V} (\dismomentum_t, \disposition_t) + C \left[ \abs{ \dismomentum_t - \bfm }^2 +  \abs{ \disposition_t - \bfx }^2  \right] + C \abs{ \frac1N \sum\nolimits_{i = 1}^N \bm{\Lambda} (\bfx,\xi^i) }^2  + C.
		\end{equation}
		
		Combining \eqref{eq:decom.d V}, \eqref{ineq:I1} and \eqref{ineq:I2}, we have
		\begin{equation}
			\label{ineq:decom.dV}
			\begin{split}
				\dd \mathcal{V} (\dismomentum_t, \disposition_t)  
			&\leqslant -\frac{2 \lambda \gamma }{3} \mathcal{V} (\dismomentum_t, \disposition_t) \dd t 
			+ \beta \inprod{  \dismomentum_t + \frac{\gamma}{2} \disposition_t  , \dd \bfB_t } \\
			&\pheq + C \left\{ \left[ \abs{ \dismomentum_t - \bfm }^2 +  \abs{ \disposition_t - \bfx }^2  \right] + \abs{ \frac1N \sum\nolimits_{i = 1}^N \bm{\Lambda} (\bfx,\xi^i) }^2  + d \right\} \dd t
			\end{split}
		\end{equation}
		Then, by Lemmas \ref{lem3-3} and \ref{lem-B-1}, we have the following differential inequality
		\begin{equation*}
			\frac{\dd}{\dd t} \bbE \mathcal{V} (\dismomentum_t, \disposition_t)
			\leqslant - \frac{2\lambda \gamma}{3} \bbE \mathcal{V} (\dismomentum_t, \disposition_t) + C t \mathcal{V}(\bfm, \bfx)  
		+ C d ,
		\end{equation*}
		for all $0\leqslant t \leqslant \eta < 1$. 
		Then, solving above inequality, we have
		\begin{equation*}
		\bbE \mathcal{V} (\dismomentum_\eta, \disposition_\eta)\leqslant \rme^{- 2 \lambda \gamma \eta/ 3 } \mathcal{V} (\bfm, \bfx) + C \eta^2 \mathcal{V}(\bfm,\bfx) + C d \eta.
		\end{equation*}
		Furthermore, the fact $1-x \leqslant \rme^{-x} \leqslant 1-x + x^2/2$ for all $x \in \bbR$ yields that one can find a positive constant $\eta_0$ 
		such that the following inequality holds for all $\eta \leqslant \eta_0$,
		\begin{equation}
		\label{eq:bound of EV for one step}
		\bbE \mathcal{V} (\dismomentum_\eta, \disposition_\eta) \leqslant \rme^{-{\gamma \lambda } \eta / 2 } \mathcal{V}(\bfm,\bfx) + C d \eta .
		\end{equation}		
		By Markov property, \eqref{eq:bound of EV for one step} implies the following recursive inequality,
		\begin{equation*}
		\bbE \mathcal{V} (\dismomentum_{t_k}, \disposition_{t_k})
		\leqslant  \rme^{- {\lambda \gamma} \eta_k / 2 } \bbE \mathcal{V} (\dismomentum_{t_{k - 1}}, \disposition_{t_{k - 1}}) + C d \eta_k,
		\end{equation*}
		as long as $\eta_k \leqslant \eta_0$.  Then, \eqref{eq:lem3-4} with $p=1$  follows from that
		\begin{equation}
			\label{ineq:recursion}
			\begin{split}
		&\pheq
		\bbE \mathcal{V} (\dismomentum_{t_n}, \disposition_{t_n}) 
		- \rme^{- {\lambda \gamma}t_n / 2} \mathcal{V}(\bfm, \bfx)  \\
		&= 
		\sum_{k = 1}^n \rme^{- {\lambda \gamma} (t_n - t_k) / {2}} \left[ \bbE \mathcal{V} (\dismomentum_{t_k}, \disposition_{t_k}) - \rme^{- {\lambda \gamma} \eta_k / 2 } \bbE \mathcal{V} (\dismomentum_{t_{k - 1}}, \disposition_{t_{k - 1}}) \right] \\
		&\leqslant   C d \sum_{k = 1}^n \eta_k \rme^{-{\lambda \gamma} (t_n - t_k)/ 2 } \\
		&=  C d  \sum_{k = 1}^n \frac{\eta_k}{1 - \rme^{-{\lambda \gamma}\eta_k/2}} \left[ \rme^{-{\lambda \gamma} (t_n - t_k) / 2 } - \rme^{-{\lambda \gamma} (t_n - t_{k - 1})/2} \right] 
		\leqslant   C d ,
		\end{split}
		\end{equation}
		where the last inequality holds as $\eta_k \leqslant 2 / (\gamma\lambda)$.
		
		\vskip 3mm
		
		As for $p>1$, by It\^o's formula and \eqref{eq:decom.d V}, we have that for $0 \leqslant  t \leqslant  \eta $,
		\begin{equation*}
		\begin{split}
		\dd \mathcal{V} (\dismomentum_t, \disposition_t)^p
		&= p \mathcal{V} (\dismomentum_t, \disposition_t)^{p-1} [I_1 + I_2] \dd t + p \beta \mathcal{V} (\dismomentum_t, \disposition_t)^{p-1} \inprod{  \dismomentum_t + \frac{\gamma}{2} \disposition_t  , \dd \bfB_t } \\
		&\pheq + \frac{p(p-1) \beta^2}{2} \abs{ \dismomentum_t + \frac{\gamma}{2}\disposition_t }^2 \mathcal{V} (\dismomentum_t , \disposition_t)^{p-2} \dd t .				
		\end{split}
		\end{equation*}
		This, together with \eqref{ineq:I1}, \eqref{ineq:I2} and $\abs{ \bfm + \gamma \bfx / 2}^2 \leqslant C \mathcal{ V }(\bfm, \bfx)$ by \eqref{eq:bounds}, it holds that
		\begin{equation} \label{ineq:decom.dVp}
			\begin{split}
				\dd \mathcal{V} (\dismomentum_t, \disposition_t)^p
				&\leqslant  \left[ p \mathcal{V} (\dismomentum_t, \disposition_t)^{p-1} \left( -\frac{2  \lambda \gamma}{3} \mathcal{V} (\dismomentum_t, \disposition_t) + C \mathcal{R}_0  \right) \right] \dd t \\
				&\pheq + p\beta \mathcal{V} (\dismomentum_t, \disposition_t)^{p-1} \inprod{  \dismomentum_t + \frac{\gamma}{2} \disposition_t  , \dd \bfB_t },  
			\end{split}
		\end{equation}
		where the reserved item $\mathcal{R}_0$ is of the form
		\begin{equation*}
			\mathcal{R}_0 =  \abs{\dismomentum_t - \bfm}^2 + \abs{\disposition_t - \bfx}^2  +   \abs{ \frac1N \sum\nolimits_{i = 1}^N \bm{\Lambda} (\bfx,\xi^i) }^2 + d,
		\end{equation*}
		which satisfies the following by Lemmas \ref{lem3-3} and \ref{lem-B-1}:
		\begin{equation} \label{e:Rp}
				\bbE \left[ \mathcal{R}_0^p \right] 
				\leqslant C t^p \mathcal{V}(\bfm,\bfx)^p + C d^p.
		\end{equation}
		By Young's inequality again, we can obtain that
		\begin{equation*}
			\mathcal{V}(\dismomentum_t, \disposition_t)^{p-1} \cdot (C \mathcal{R}_0)
			\leqslant \left(1-\frac1p\right) \frac{2\lambda\gamma}{3} \mathcal{V}(\dismomentum_t, \disposition_t)^p + C \mathcal{R}_0^{p}.
		\end{equation*}
		Substituting this into \eqref{ineq:decom.dVp}, there is
		\[
			\dd \mathcal{V} (\dismomentum_t, \disposition_t)^p
			\leqslant -\frac{2 \lambda \gamma }{3}  \mathcal{V}(\dismomentum_t, \disposition_t)^p \dd t 
			+ C \mathcal{R}_0^{p} \dd t 
			+ p \beta \mathcal{V} (\dismomentum_t, \disposition_t)^{p-1} \inprod{  \dismomentum_t + \frac{\gamma}{2} \disposition_t  , \dd \bfB_t } .
		\]		
		Combining \eqref{e:Rp}, this implies the follow differential inequality
		\begin{equation*}
		\frac{\dd}{\dd t} \bbE  \mathcal{V} (\dismomentum_t, \disposition_t)^p
		\leqslant -\frac{2 \lambda \gamma}{3} \bbE \mathcal{V} (\dismomentum_t, \disposition_t)^p 
		 + C t^p \mathcal{V}(\bfm,\bfx)^p +  C d^p .
		\end{equation*}
		By Gronwall's inequality, we have
		\begin{equation*}
		\bbE \mathcal{V} (\dismomentum_{\eta }, \disposition_{\eta })^p
		\leqslant  \rme^{- {2\lambda \gamma}\eta / 3} \mathcal{V} (\bfm, \bfx)^p + C \eta^p \mathcal{V} (\bfm, \bfx)^p + C d^p {\eta } .
		\end{equation*}
		And one can also find a positive constant $\eta_0^{\prime}$ 
		such that for all $\eta \leqslant \eta_0^{\prime}$, the following holds,
		\begin{equation*}
		\bbE \mathcal{V} (\dismomentum_{\eta }, \disposition_{\eta })^p \leqslant \rme^{-{\lambda\gamma}{\eta }/2} \mathcal{V}(\bfm,\bfx)^p + C d^p {\eta }.
		\end{equation*}		
		Hence, for any $k \geqslant 1$, as $\eta_k \leqslant \eta_0^{\prime}$, by Markov property, we have
		\begin{equation*}
		\bbE \mathcal{V} (\dismomentum_{t_k}, \disposition_{t_k})^p
		\leqslant  \rme^{- {\lambda \gamma} \eta_k / 2} \bbE \mathcal{V} (\dismomentum_{t_{k - 1}}, \disposition_{t_{k - 1}})^p + C d^p \eta_k.
		\end{equation*}
		which implies \eqref{eq:lem3-4}  for $1<p \leqslant q/2$ under Assumption \ref{Assump.2} ( or $1<p\leqslant q^{\prime}/2$ if the condition \eqref{A4-2} in Assumption \ref{Assump.4} holds  ) by recursion similar to \eqref{ineq:recursion}.
		
		The proof is complete. 
	\end{proof}

	\subsection{Proofs of Auxiliary Lemmas for $1$-Wasserstein Distance}
	\label{C2:1-Wasserstein distance}

	\begin{proof}[\textbf{Proof of Lemma \ref{lem3-5}}]
		
		\noindent(i) Suppose that $(\hatmomentum_t, \hatposition_t)_{0\leqslant t\leqslant \eta}$ is the solution of SDE
		\begin{equation*}
			\left\{
				\begin{aligned}
					\dd \hatmomentum_t &= -\gamma \bfm \, \dd t - \nabla f(\bfx)  \dd t + \beta  \dd \bfB_t, \\
					\dd \hatposition_t &= \bfm  \dd t,
				\end{aligned}
			\right.
		\end{equation*} 
		with initial value $(\hatmomentum_0, \hatposition_0) = (\bfm, \bfx)$.
		We claim that 
		\begin{align}
			\label{eq:W_1 of mu and mu_hat}
			d_{\mathcal{W}_1} ( \mathcal{L} (\conmomentum_{\eta}, \conposition_{\eta}), \mathcal{L} (\hatmomentum_{\eta}, \hatposition_{\eta}) ) 
			&\leqslant C \eta^{3/2} \sqrt{\mathcal{V}(\bfm,\bfx) + d} , \\
			\label{eq:W_1 of mu_tilde and mu_hat}
			d_{\mathcal{W}_1} ( \mathcal{L} (\hatmomentum_{\eta}, \hatposition_{\eta}), \mathcal{L} (\dismomentum_{\eta}, \disposition_{\eta}) ) 
			&\leqslant C {  \frac{\sqrt{d}}{N}  \eta^{3/2},}
		\end{align}
		which immediately implies the desired result by the triangle inequality.

		Let us now show \eqref{eq:W_1 of mu and mu_hat}. Observe that
		$(\conmomentum_t - \hatmomentum_t, \conposition_t - \hatposition_t)_{0 \leqslant  t \leqslant  \eta}$ satisfies
		\begin{equation*}
			\left\{ 
				\begin{aligned}
					\dd (\conmomentum_t - \hatmomentum_t) &= -\gamma (\conmomentum_t - \bfm) \, \dd t - (\nabla f(\conposition_t) - \nabla f(\bfx)) \, \dd t, \\
					\dd (\conposition_t - \hatposition_t) &= (\conmomentum_t - \bfm) \, \dd t,
				\end{aligned}
			\right.
		\end{equation*}
		with initial value $(\conmomentum_t - \hatmomentum_t, \conposition_t - \hatposition_t) = (\zero, \zero)$. Then, it holds that
		\begin{align*}
			\bbE \abs{\conmomentum_t - \hatmomentum_t}^2
			&\leqslant  2 \gamma^2 \bbE \abs{\int_0^t (\conmomentum_s - \bfm) \, \dd s}^2 + 2 \bbE \abs{\int_0^t (\nabla f(\conposition_s) - \nabla f(\bfx)) \, \dd s}^2 \\
			&\leqslant  2 \gamma^2 t \bbE \int_0^t \abs{\conmomentum_s - \bfm}^2 \dd s + 2 t \bbE\int_0^t \abs{\nabla f(\conposition_s) - \nabla f(\bfx)}^2 \dd s , \\
			\bbE \abs{\conposition_t - \hatposition_t}^2
			&= \bbE \abs{\int_0^t (\conmomentum_s - \bfm) \, \dd s}^2
			\leqslant  t \int_0^t \bbE \abs{\conmomentum_s - \bfm}^2 \dd s.
		\end{align*}
		These, together with Lemma \ref{lem3-2}, we have
		\[
			\bbE \abs{\conmomentum_\eta - \hatmomentum_\eta}^2 + \bbE \abs{\conposition_{\eta} - \hatposition_{\eta}}^2 
			\leqslant C \eta^3 ( \mathcal{V}(\bfm, \bfx) + d ) .
		\]
	By definition of $1$-Wasserstein distance, we immediately know \eqref{eq:W_1 of mu and mu_hat} holds true.
		
		It remains to prove \eqref{eq:W_1 of mu_tilde and mu_hat}. Due to $\hatposition_{\eta} = \disposition_{\eta} = \bfx+ \eta \bfm$, it is equivalent to prove 
		\begin{equation*}
			d_{\mathcal{W}_1} ( \mathcal{L} (\hatmomentum_{\eta}), \mathcal{L} (\dismomentum_{\eta}) ) \leqslant C {  \frac{\sqrt{d}}{N} \eta^{3/2},}
		\end{equation*}
		or by Kantorovich-Rubinstein Theorem,
		\begin{equation*}
			\abs{\bbE h (\hatmomentum_{\eta}) - \bbE h (\dismomentum_{\eta})}
			\leqslant C {  \frac{\sqrt{d}}{N} \eta^{3/2} \norm{h}_{\Lip},}  
		\end{equation*}
		for any Lipschitz function $h$. Since $h$ can be pointwise approximated by a sequence of $h_i \in \mathcal{C}_b^2 (\bbR^d, \bbR)$, $i = 1, 2, \dotsc$, satisfying $\norm{\nabla h_i}_\infty \leqslant 2 \norm{h}_{\Lip}$, we just need to show that
		\begin{equation} \label{eq:W1gra_h}
			\abs{\bbE h (\hatmomentum_{\eta}) - \bbE h (\dismomentum_{\eta})}
			\leqslant C \frac{\sqrt{d}}{N} \eta^{3/2} \norm{\nabla h}_\infty, \quad \forall\ h \in \mathcal{C}_b^2 (\bbR^d).
		\end{equation}
		For any fixed $h \in \mathcal{C}_b^2 (\bbR^d)$, we define
		\begin{equation*}
			h_0 (\bfz) \defas \bbE \left[ h( \bfz + \beta \bfB_{\eta}) \right]
			= \int_{\bbR^d} h(\bfz + \bfy) \cdot ( 2\pi \beta^2 \eta )^{-{d}/{2}} \exp \left\{ -{\abs{\bfy}^2} / {(2\beta^2\eta)} \right\} \dd \bfy .
		\end{equation*}
		Straightforward calculations derive that
		\begin{equation*}
			\begin{split}
				\frac{\partial^2 h_0}{ \partial z_i \partial z_j } (\bfz) 
				&= \int_{\bbR^d} \frac{\partial^2 h }{\partial z_i\partial z_j} (\bfz+\bfy) \cdot ( 2\pi \beta^2 \eta )^{-{d}/{2}} \exp \left\{ -{\abs{\bfy}^2} / {(2\beta^2\eta)} \right\} \dd \bfy \\
				&= \int_{\bbR^d} \frac{\partial h }{\partial z_i } (\bfz+\bfy) \cdot \frac{y_j}{ \beta^2 \eta } \cdot ( 2\pi \beta^2 \eta )^{-{d}/{2}} \exp \left\{ -{\abs{\bfy}^2} / {(2\beta^2\eta)} \right\} \dd \bfy \\
				&= \frac{1}{\beta^2 \eta} \bbE\left[ \frac{\partial h}{\partial z_i} (\bfz+\beta \bfB_{\eta})\cdot \beta \bfB_{\eta}^{(j)} \right] ,
			\end{split}
		\end{equation*}
		where $\bfB_{\eta}^{(j)}$ is the $j$-th component of $\bfB_{\eta}$. That is 
		\begin{equation*}
			\nabla^2 h_0(\bfz) = \frac{1}{\beta \eta} \bbE \left[ {\nabla h(\bfz+\beta \bfB_{\eta})  \bfB_{\eta}^{\top} }\right],
		\end{equation*}
		which implies that
		\begin{equation} \label{eq:infty norm h0}
			\norm{ \nabla^2 h_0 }_\infty \leqslant \frac{\infnorm{ \nabla h }}{\beta \eta} \bbE \abs{\bfB_{\eta}} \leqslant \frac{\infnorm{ \nabla h }}{\beta } \sqrt{\frac{d}{\eta}}.
		\end{equation}
		Taylor’s expansion gives that
		\begin{align*}
			&\mathrel{\phantom{=}} \bbE h (\dismomentum_{\eta}) - \bbE h (\hatmomentum_{\eta}) \\
			&= \bbE \left[ h_0 \Big( (1 - \eta \gamma) \bfm - \eta \frac1N \sum_{i = 1}^N\nabla F(\bfx, \xi^i) \Big) \right] - \bbE \left[ h_0 \big((1 - \eta \gamma) \bfm - \eta \nabla f(\bfx) \big) \right] \\
			&= \eta \bbE \left[ \sca{\nabla h_0 \big( (1 - \eta \gamma) \bfm - \eta \nabla f(\bfx) \big), \frac1N \sum_{i = 1}^N \bm{\Lambda} (\bfx, \xi^i) } \right]\\
			&\mathrel{\phantom{=}} + \frac{1}{2} \eta^2 \bbE \left[ \inprod{ \frac1N \sum_{i = 1}^N \bm{\Lambda} (\bfx, \xi^i) , \nabla^2 h_0 (\bm{\zeta})  \frac1N \sum_{i = 1}^N \bm{\Lambda} (\bfx, \xi^i) } \right] \\
			&= \frac{1}{2} \eta^2 \bbE \left[ \inprod{ \frac1N \sum_{i = 1}^N \bm{\Lambda} (\bfx, \xi^i) , \nabla^2 h_0 (\bm{\zeta})  \frac1N \sum_{i = 1}^N \bm{\Lambda} (\bfx, \xi^i) } \right],
		\end{align*}
		holds for some random vector $\bm{\zeta}$, where $\bm{\Lambda} (\bfx, \xi^i)$, $i=1, \dotsc, N$ are defined in \eqref{e:gi} and we use the independence between $\bfB_{\eta}$ and $\xi$. Together with \eqref{eq:infty norm h0} and Lemma \ref{lem-B-1}, we have
		\begin{align*}
			\abs{\bbE h (\dismomentum_{\eta}) - \bbE h (\hatmomentum_{\eta})} 
			\leqslant \frac{1}{2} \eta^2 \norm{ \nabla^2 h_0 }_\infty \bbE_{\xi} \abs{ \frac1N \sum_{i = 1}^N \bm{\Lambda} (\bfx, \xi^i)}^2 
			\leqslant C { \frac{\sqrt{d}}{N} } \eta^{3/2} \norm{\nabla h}_\infty,
		\end{align*}
		which is \eqref{eq:W1gra_h}, and leads to \eqref{eq:W_1 of mu_tilde and mu_hat}.

		\noindent(ii) Notice that
		$(\conmomentum_t - \dismomentum_t, \conposition_t - \disposition_t)_{0 \leqslant  t \leqslant  \eta}$ satisfies
		\begin{equation*}
		\left\{
		\begin{aligned}
		\dd (\conmomentum_t - \dismomentum_t) &= -\gamma (\conmomentum_t - \bfm) \, \dd t - (\nabla f(\conposition_t) - \nabla f(\bfx)) \, \dd t - \frac{1}{N} \sum_{i = 1}^N \bm{\Lambda} (\bfx, \xi^i) \, \dd t, \\
		\dd (\conposition_t - \disposition_t) &= (\conmomentum_t - \bfm) \, \dd t,
		\end{aligned}
		\right.
		\end{equation*}
		with initial value $(\conmomentum_0 - \dismomentum_0, \conposition_0 - \disposition_0) = (\zero, \zero)$. Combining H\"older's inequality and Lemma \ref{lem-B-1} derives that
		\begin{equation*}
		\begin{split}
		\bbE \abs{\conmomentum_\eta - \dismomentum_\eta}^4
		&\leqslant  C \bbE \abs{\int_0^\eta (\conmomentum_s - \bfm) \, \dd s}^4 + C \bbE \abs{\int_0^\eta (\nabla f(\conposition_s) - \nabla f(\bfx)) \, \dd s}^4 \\
		&\quad + C \eta^4 \bbE \abs{\frac{1}{N} \sum_{i = 1}^N \bm{\Lambda} (\bfx, \xi^i) }^4 \\
		&\leqslant C\eta^3 \int_0^\eta \bbE \abs{\conmomentum_s - \bfm}^4 \dd s + C \eta^3 \int_0^\eta \bbE \abs{\conposition_s - \bfx}^4 \dd s + C { \frac{\eta^4}{N^2}, } \\
		\bbE \abs{\conposition_\eta - \disposition_\eta}^4
		&= \bbE \abs{\int_0^\eta (\conmomentum_s - \bfm) \, \dd s}^4
		\leqslant \eta^3 \int_0^\eta \bbE \abs{\conmomentum_s - \bfm}^4 \dd s.
		\end{split}
		\end{equation*}	
		Then Lemma \ref{lem3-2} with $p=2$ implies
		\begin{equation*}
		\bbE \abs{\conmomentum_\eta - \dismomentum_\eta}^4 + \bbE \abs{\conposition_{\eta} - \disposition_{\eta}}^4
		\leqslant C \eta^6 (\mathcal{V}(\bfm,\bfx)^2 + d^2 ) + C { \frac{\eta^4}{N^2}. }
		\end{equation*} 
		
		We complete the proof.
	\end{proof}

	\subsection{Proofs of Auxiliary Lemmas in the Total Variation Distance}
	\label{C3:total variation distance}
	Before proving Lemma \ref{lem3-7}, we shall introduce the well known Bismut formula in Malliavin calculus. To this end, let us briefly recall the Malliavin calculus by Bismut in our setting. Let $T>0$ be an arbitrary number, for any $t \in [0,T]$, we take $\bfY_t$ as a functional of Brownian motion, i.e., $\bfY_t(\bfB)$, let $\bfh$ be an adaptive stochastic process in $L^2([0,T] \times \Omega,\bbR^d)$, and define a general Malliavin derivative along 
${\bfh}$ as the following
  \begin{equation}
  	\label{eq:BisMal}
		\md^{\bfh} \bfY_t(\BM) = \lim_{\varepsilon \to 0} \frac{\bfY_t(\BM + \varepsilon {\bf H}) - \bfY_t( \BM ) }{\varepsilon},
	\end{equation}
where ${\bf H}(t)=\int_0^t \bfh(s) \dd s$ for $t \in [0,T]$, as long as the above limit exists in $L^2(\Omega)$. For $\phi \in \mathcal{C}^1_b(\bbR^{d},\bbR)$, by the chain rule, 
	\[
		\md^{\bfh} \phi( \bfY_t(\BM) ) =  \nabla \phi( \bfY_t(\BM) )  \md^{\bfh}  \bfY_t(\BM)   =  \sum_{i = 1}^d \partial_i \phi( \bfY_t(\BM) ) \md^{\bfh} F_t^i(\BM)
	\]
The following Bismut's formula, which can be taken as an integration by parts, is an important property in Malliavin calculus:
	\begin{equation}
		\label{eq:initial Bismut}
		\bbE\left[ \nabla \phi(\bfY_t) \md^{\bfh}\bfY_t \right] = \bbE \left[ \phi(\bfY_t) \int_0^t \inprod{{\bfh}(s),  \dd \bfB_s } \right].
	\end{equation}

	\begin{proof}[\textbf{Proof of Lemma \ref{lem3-7}}]
		
		Let $\bfY_t^{\bfy} = (\conmomentum_t^{\bfm} , \conposition_t^{\bfx})$ come from SDE  \eqref{SDE1}. Let  $\bfh$ be an adaptive stochastic process in $L^2([0,T] \times \Omega,\bbR^{2d})$ to be chosen later, by \cite{MR0942019}, the Malliavin derivative of $\bfY_t$ along $\bfh$ satisfies
			\[
			\left\{
			\begin{aligned}
			\md^{\bfh} \conmomentum_t^{\bfm} &= - \gamma \int_{0}^t \md^{\bfh} \conmomentum_s^{\bfm} \dd s - \int_{0}^t \nabla^2 f(\conposition_s^{\bfx}) \md^{\bfh} \conposition_s^{\bfx} \dd s + \beta \int_0^t \bfh(s) \dd s , \\
			\md^{\bfh} \conposition_t^{\bfx} &= \int_{0}^t \md^{\bfh} \conmomentum_s^{\bfm} \dd s .
			\end{aligned}
			\right.
			\]
			Besides, for any $\bfw = (\bfu, \bfv) \in \bbR^{2d}$, the Jacobian flow of $\bfY_t$ with respect to $\bfw$ is defined as
			\[
			\nabla \bfY_t^{\bfy} \cdot \bfw = \lim_{ \varepsilon \to 0 } \frac{\bfY_t^{\bfy + \varepsilon \bfw} - \bfY_t^{\bfy}}{\varepsilon} .
			\]
			More precisely, we have that
			\[
			\begin{split}
			\nabla_{\bfm} \bfY_t \cdot \bfu  &= 
			\begin{bmatrix}
			\bfu \\ \zero
			\end{bmatrix}
			+ \int_0^{t} 
			\begin{bmatrix}
			-\gamma \mathbf{I} & \nabla^2 f(\conposition_t) \\
			\mathbf{I} & \zero \\
			\end{bmatrix}
			\nabla_{\bfm} \bfY_t \cdot \bfu \dd t ,\\
			\nabla_{\bfx} \bfY_t \cdot \bfv  &= 
			\begin{bmatrix}
			\zero \\ \bfv
			\end{bmatrix}
			+ \int_0^{t} 
			\begin{bmatrix}
			-\gamma \mathbf{I} & \nabla^2 f(\conposition_t) \\
			\mathbf{I} & \zero \\
			\end{bmatrix}
			\nabla_{\bfx} \bfY_t \cdot \bfv \dd t .
			\end{split}
			\]			
			Observe that for any  $\phi \in \mathcal{C}_b^1(\bbR^{2d}, \bbR)$
			\begin{equation}\label{eq:grad phi}
			\nabla \bbE \phi(\bfY_T) \cdot \bfw = \bbE\left[ \nabla \phi(\bfY_T) \nabla_\bfm \bfY_T \cdot \bfu \right] + \bbE\left[ \nabla \phi(\bfY_T) \nabla_\bfx \bfY_T \cdot \bfv \right] .
			\end{equation}
			 We aim to find some special $\bfh$ and $\tilde{\bfh}$ such that 
			\[
			\nabla_{\bfm} \bfY_T \cdot \bfu = \md^{\bfh} \bfY_T \quad \text{and}  \quad 
			\nabla_{\bfx} \bfY_T \cdot \bfv = \md^{\widetilde{\bfh}} \bfY_T,
			\]
		and consequently
		\begin{equation*}
			\nabla \bbE \phi(\bfY_T) \cdot \bfw = \bbE\left[ \nabla \phi(\bfY_T)  \md^{\bfh} \bfY_T \right] + \bbE\left[ \nabla \phi(\bfY_T) \md^{\widetilde{\bfh}} \bfY_T \right],
			\end{equation*}
			which enables us to apply Bismut's formula \eqref{eq:initial Bismut}. 
			
			Define $\bfH(t) = \int_{0}^t \bfh(s) \dd s$ and $\widetilde{\bfH}(t) = \int_{0}^t \tilde{\bfh}(s) \dd s$, by Zhang \cite[Theorem 3.3]{MR2673982}, for any $T>0$ and $\bfu, \bfv \in \bbR^{d}$, we can choose $\bfH(t)$ and $\widetilde{\bfH}(t)$ specified as below.
		Firstly, for any $\bfu \in \bbR^d$, let
		\begin{equation*}
		\bfH(t) = -\frac{1}{\beta} \left( \bfu - \bfg(t) - \gamma \int_{0}^t \bfg(s) \dd s - \int_{0}^t \dd s \int_{0}^s \nabla^2 f(\conposition_s) \bfg(r) \, \dd r \right) ,
		\end{equation*}
		for $t\in [0,T]$, where $\bfg(t)$ is given by
		\[
		\bfg(t) = 
		\left\{
		\begin{aligned}
		\left(T - 3t \right) \bfu / T, & \quad 0\leqslant t < {T}/{2} ,\\
		\left(t- T \right)\bfu / T , & \quad {T}/{2} \leqslant t \leqslant T,
		\end{aligned}
		\right.
		\] 
		satisfying $\bfg(0) = \bfu$, $\bfg(T) = \zero$ and $\int_0^T \bfg(t) \dd t = \zero$. One can verify that 
		\begin{equation*}
			\left\{
				\begin{aligned}
					\sca{\nabla_{\bfm}\conmomentum_t(\bfm,\bfx), \bfu} &= \md^{\bfh}\conmomentum_t(\bfm,\bfx) + \bfg(t),  \\
					\sca{\nabla_{\bfm} \conposition_t(\bfm,\bfx), \bfu} &= \md^{\bfh}\conposition_t(\bfm,\bfx) + \int_{0}^t \bfg(s) \, \dd s ,
				\end{aligned}			
			\right.
		\end{equation*}
		and $\nabla_{\bfm} \bfY_T \cdot \bfu = \md^{\bfh} \bfY_T$. Then, applying Bismut formula to $\bbE\left[ \nabla \phi(\bfY_T)  \md^{\bfh} \bfY_T \right]$ yields that
		\[
		\sca{\nabla_{\bfm} \bbE \phi(\bfY_T), \bfu}
		=
		\bbE \left[ - \phi(\bfY_T) \frac{1}{\beta } \int_{0}^T  \left( \nabla^2 f(\conposition_s) \int_0^s \bfg(r) \, \dd r + \gamma \bfg(s) + \dot{\bfg}(s) \right) \dd \bfB_s \right] .
		\]
		Hence,  there exists a constant $C$ only depending on $L$, $\gamma$ and $\beta$ such that
		\begin{equation}\label{ineq:B1}
		\abs{\sca{\nabla_{\bfm} \bbE \phi(\bfY_T), \bfu}}
		\leqslant C \infnorm{\phi} \abs{\bfu} \left( T^{{3}/{2}} \lor T^{-{1}/{2}} \right).
		\end{equation}
		
		Analogously, for any $\bfv \in \bbR^d$, let $\widetilde{\bfH}(t)$ 
		\[
		\widetilde{\bfH}(t) = \frac1{\beta} \left[ -\widetilde{\bfg}(t) - \gamma \int_{0}^t \widetilde{\bfg}(s) \, \dd s - \int_0^t \nabla^2 f(\conposition_s) \left( \bfv+\int_{0}^s \widetilde{\bfg}(r) \dd r \right) \dd s \right],
		\]
		for $t \in [0,T]$, where $\widetilde{\bfg}(t)$ is given by
		\[	
		\widetilde{\bfg}(t) = \left\{
		\begin{aligned}
		-{4t}/{T^2} \bfv \quad, &\quad 0\leqslant t \leqslant {T}/{2}, \\
		{4(t-T)}/{T^2} \bfv, & \quad {T}/{2} \leqslant t \leqslant T, 
		\end{aligned}
		\right.
		\]
		satisfying $\widetilde{\bfg}(0) = \widetilde{\bfg}(T) = 0$ and $\int_{0}^T \widetilde{\bfg}(t) \dd t = -\bfv$. Then, it is easy to verify that
		\begin{equation*}
			\left\{
				\begin{aligned}
					\sca{\nabla_{\bfx}\conmomentum_t(\bfm,\bfx), \bfv} &= \md^{\widetilde{\bfh}}\conmomentum_t(\bfm,\bfx) + \widetilde{\bfg}(t),  \\
					\sca{\nabla_{\bfx} \conposition_t(\bfm,\bfx), \bfv} &= \md^{\widetilde{\bfh}}\conposition_t(\bfm,\bfx) + \bfv +  \int_{0}^t \widetilde{\bfg}(s) \, \dd s ,
				\end{aligned}
			\right.
		\end{equation*}
		and $\nabla_{\bfx} \bfY_T \cdot \bfv = \md^{\widetilde{\bfh}} \bfY_T$. Applying Bismut's formula to $\bbE\left[ \nabla \phi(\bfY_T) \md^{\widetilde{\bfh}} \bfY_T \right]$ yields
		\[
		\sca{\nabla_{\bfx} \bbE \phi(\bfY_T), \bfv}
		= \bbE \left\{ - \phi(\bfY_T) \frac{1}{\beta} \int_{0}^T  \left[ \nabla^2 f(\conposition_s) \left( \bfv + \int_0^s \widetilde{\bfg}(r) \, \dd r \right) + \gamma \widetilde{\bfg}(s) + \dot{\widetilde{\bfg}}(s) \right]  \dd \bfB_s \right\},
		\]
		which immediately implies that there exists a constant $C>0$ only depending on $L$, $\gamma$ and $\beta$ such that
		\begin{equation}\label{ineq:B2}
		\abs{\sca{\nabla_{\bfx} \bbE \phi(\bfY_T), \bfv}}
		\leqslant C  \infnorm{ \phi } \abs{\bfv} \left( T^{{1}/{2}} \lor T^{-{3}/{2}} \right).
		\end{equation}
		
		Combining \eqref{eq:grad phi}, \eqref{ineq:B1} and \eqref{ineq:B2}, we have
		\[
			\abs{ \nabla \bbE \phi(\bfY_T) \cdot \bfw } \leqslant C \infnorm{\phi} (\abs{\bfu} + \abs{\bfv}) \left( T^{{3}/{2}} \lor T^{-{3}/{2}} \right) \leqslant C \infnorm{\phi} \left( T^{{3}/{2}} \lor T^{-{3}/{2}} \right)  \abs{ \bfw}
		\]
		for any $\bfw = (\bfu, \bfv) \in \bbR^{2d}$, and this implies the desired. 
	\end{proof}
	

	\begin{proof}[\textbf{Proof of Lemma \ref{lem3-8} }]
		Firstly, the chain rule of Malliavin derivative yields that
		\[
		\inprod{ \md g(\bfF), \md F^j }_{\mathcal{H}} = \sum_{i=1}^d \partial_i g(\bfF) \inprod{ \md F^i , \md F^j }_{\mathcal{H}} = \sum_{i=1}^d \partial_i g(\bfF) ( \bm{\Gamma}(\bfF) )_{i, j} , \ j=1,\dotsc, d,
		\]
		which implies that
		\begin{equation}\label{eq:partial}
		\partial_i g(\bfF) = \sum_{j=1}^d \inprod{ \md g(\bfF), \md F^j }_{\mathcal{H}} \big( \bm{\Gamma}(\bfF)^{-1} \big)_{i,j}, \ i = 1, \dotsc, d.
		\end{equation}
		Then,  we have
		\begin{align*}
		\bbE \inprod{ \nabla g(\bfF) , \bfG } 
		&= \bbE \inprod{ \md g(\bfF) , \sum_{i=1}^d \sum_{j=1}^d \big( \bm{\Gamma}(\bfF) \big)^{-1}_{i,j} G^i \md F^j  }_{\mathcal{H}} \\
		&= \bbE \left[ g(\bfF) \delta \left( \sum_{j=1}^d \big( \bm{\Gamma}(\bfF)^{-1} \bfG \big)_j \md F^j \right) \right].
		\end{align*} 
		where the first equality is by \eqref{eq:partial}, and the last one is due to the duality \eqref{eq:dual} between derivative operator and divergence operator. 
		Then, the Jensen's inequality yields that
		\begin{equation} \label{eq:pr3.7.1}
		\abs{ \bbE \inprod{ \nabla g(\bfF) , \bfG } } \leqslant
		\infnorm{g} \sqrt{\bbE \left[ \bigg|{ \delta \bigg( \sum_{j=1}^d \big( \bm{\Gamma}(\bfF)^{-1} \bfG \big)_j \md F^j \bigg) }\bigg|^2 \right]}.
		\end{equation}
		Applying the continuity of divergence operator $\delta$ from $\mathbb{D}^{1,2}(\mathcal{H})$ into $L^2(\Omega)$ \cite[Propsition 1.5.7]{MR2200233}, there exists a constant $C>0$ such that
		\begin{equation} \label{eq:pr3.7.2}
			\bbE \left[\bigg|{ \delta \bigg( \sum_{j=1}^d \big( \bm{\Gamma}(\bfF)^{-1} \bfG \big)_j \md F^j \bigg) }\bigg|^2 \right] \leqslant C ( \mathcal{I}_1 + \mathcal{I}_2 ) ,
		\end{equation}
		where $\mathcal{I}_1$ and $\mathcal{I}_2$ are given by
		\begin{equation*}
			\mathcal{I}_1 = \bbE\left[  \bigg\Vert \sum_{j=1}^d \big( \bm{\Gamma}(\bfF)^{-1} \bfG \big)_j \md F^j \bigg\Vert^2_{\mathcal{H}} \right] , \quad
			\mathcal{I}_2 = \bbE \left[ \bigg\Vert \sum_{j=1}^d \md \left( \big( \bm{\Gamma}(\bfF)^{-1} \bfG \big)_j \md F^j \right) \bigg\Vert_{\mathcal{H}\otimes \mathcal{H}}^2 \right]
		\end{equation*}
		
		For $\mathcal{I}_1$, it follows from Cauchy-Schwarz inequality that
		\begin{equation} \label{eq:pr3.7.3}
		\begin{split}
		\mathcal{I}_1 
		&\leqslant \bbE \left[ \bigg|{ \sum_{j=1}^d \abs{ \big( \bm{\Gamma}(\bfF)^{-1} \bfG \big)_j } \norm{ \md F_j }_{\mathcal{H}} }\bigg|^2 \right] \\
		&\leqslant \bbE \left[ \bigg( \sum_{j=1}^d \big( \bm{\Gamma}(\bfF)^{-1} \bfG \big)_j^2 \bigg) \bigg( \sum_{j=1}^d \norm{ \md F^j }^2_{\mathcal{H}} \bigg) \right] \\
		&\leqslant \bbE \left\{ \left[ \opnorm{ \bm{\Gamma}(\bfF)^{-1} } \abs{ \bfG } \norm{ \md \bfF }_{\mathcal{H}} \right]^2 \right\}.
		\end{split}
		\end{equation}
		
		For $\mathcal{I}_2$, according to \cite[Lemma 2.1.6]{MR2200233}, it holds that
		\[
			\md \big( \bm{\Gamma}(\bfF)^{-1} \bfG \big)_j 
			= \sum_{i=1}^d \big( \bm{\Gamma}(\bfF)^{-1} \big)_{j,i} \md G^i 
			- \sum_{i,k,l=1}^d G^i  \big( \bm{\Gamma}(\bfF)^{-1} \big)_{j,k} \big( \bm{\Gamma}(\bfF)^{-1} \big)_{l,i} \md \big( \bm{\Gamma}(\bfF)_{k,l} \big).
		\]	
		This, together with representation $\md \big( \bm{\Gamma}(\bfF)_{k,l} \big) = \md \inprod{ \md F^k , \md F^l }_{\mathcal{H}}$ for any $1\leqslant k, l\leqslant d$, we have
		\begin{align} \label{e:EST-I2}
			\bigg\Vert \sum_{j=1}^d \md \left( \big( \bm{\Gamma}(\bfF)^{-1} \bfG \big)_j \md F^j \right) \bigg\Vert_{\mathcal{H}\otimes \mathcal{H}} \leqslant \mathcal{I}_{2,1} +  \mathcal{I}_{2,2} + \mathcal{I}_{2,3}, 
		\end{align}
		where $\mathcal{I}_{2,1}$,  $\mathcal{I}_{2,2}$ and $\mathcal{I}_{2,3}$ are given by
		\[
			\begin{split}
				\mathcal{I}_{2,1} &= \sum_{j=1}^d \abs{ \big( \bm{\Gamma}(\bfF)^{-1} \bfG \big)_j } \norm{ \md^2 F^j }_{ \mathcal{H}\otimes \mathcal{H} } \\
				\mathcal{I}_{2,2} &= \sum_{i,j=1}^d \abs{ \big( \bm{\Gamma}(\bfF)^{-1} \big)_{j,i} } \norm{ \md G^i }_{\mathcal{H}} \norm{ \md F^j }_{\mathcal{H}} \\
				\mathcal{I}_{2,3} &= 
				\sum_{i,j,k,l=1}^d \abs{ G^i \big( \bm{\Gamma}(\bfF)^{-1} \big)_{j,k} \big( \bm{\Gamma}(\bfF)^{-1} \big)_{l,i} }
				\norm{ \md \inprod{ \md F^k , \md F^l }_{\mathcal{H} } }_{\mathcal{H}} \norm{ \md F^j }_{\mathcal{H}} .
			\end{split}
		\]
		It is easy to estimate $\mathcal{I}_{2,1}$ and $\mathcal{I}_{2,2}$ which satisfy
		\begin{equation}
			\label{e:J1J2}
			\mathcal{I}_{2,1} \leqslant \opnorm{ \bm{\Gamma}(\bfF)^{-1} } \abs{\bfG} \norm{ \md^2 \bfF }_{\mathcal{H}\otimes \mathcal{H}} ,
			\qquad 
			\mathcal{I}_{2,2} \leqslant \hsnorm{ \bm{\Gamma}(\bfF)^{-1} }^{\phantom{2}} \norm{\md \bfG}_{\mathcal{H}} \norm{ \md \bfF }_{\mathcal{H}} .
		\end{equation}
		For $\mathcal{I}_{2,3}$, we claim that
		\begin{equation}
			\label{e:J3}
			\mathcal{I}_{2,3} \leqslant 2  \hsnorm{ \bm{\Gamma}(\bfF)^{-1} }^2  \abs{ \bfG } \norm{ \md \bfF }_{\mathcal{H}}^2 \norm{ \md^2 \bfF }_{\mathcal{H}\otimes \mathcal{H}} .
		\end{equation}
		Combining \eqref{e:EST-I2}, \eqref{e:J1J2} and \eqref{e:J3}, we can obtain that
		\begin{equation} \label{eq:pr3.7.5}
			\begin{split}
				\mathcal{I}_2 
				& = \bbE \left[ \bigg\Vert \sum_{j=1}^d \md \left( \big( \bm{\Gamma}(\bfF)^{-1} \bfG \big)_j \md F^j \right) \bigg\Vert_{\mathcal{H}\otimes \mathcal{H}}^2 \right]\\
				&\leqslant C \bigg\{ 
					\bbE \left[ \opnorm{ \bm{\Gamma}(\bfF)^{-1} }^2 \abs{\bfG}^2 \norm{ \md^2 \bfF }_{\mathcal{H}\otimes \mathcal{H}}^2 \right] \\
					&\qquad\qquad+ \bbE \left[ \hsnorm{ \bm{\Gamma}(\bfF)^{-1} }^2 \norm{\md \bfG}_{\mathcal{H}}^2 \norm{ \md \bfF }_{\mathcal{H}}^2 \right] \\
					&\qquad\qquad\qquad + \bbE \left[ \hsnorm{ \bm{\Gamma}(\bfF)^{-1} }^4 \abs{ \bfG }^2 \norm{ \md \bfF }_{\mathcal{H}}^4 \norm{ \md^2 \bfF }_{\mathcal{H}\otimes \mathcal{H}}^2 \right]
				\bigg\} .
			\end{split}			
		\end{equation}
		Combining \eqref{eq:pr3.7.2}, \eqref{eq:pr3.7.3} and \eqref{eq:pr3.7.5}, by H\"older's inequality, we can obtain that
		\[
			\begin{split}
				&\pheq
				\bbE \left[\bigg|{ \delta \bigg( \sum_{j=1}^d \big( \bm{\Gamma}(\bfF)^{-1} \bfG \big)_j \md F^j \bigg) }\bigg|^2 \right] \\
				&\leqslant C \norm{\bfG}_{1,4}^2 \left\{ \bbE\Big[ \big( 1+\norm{\md \bfF}_{\mathcal{H}}^8 \big) \big( 1+ \norm{\md^2 \bfF}_{\mathcal{H}\otimes \mathcal{H}}^4 \big) \big( 1 + \hsnorm{ \bm{\Gamma}(\bfF)^{-1} }^8 \big) \Big]   \right\}^{1/2} .
			\end{split}
		\]
		Substituting this into \eqref{eq:pr3.7.1}, we can obtain the desired.
		
		It remains to prove \eqref{e:J3}. Observe that for any $1\leqslant k, l \leqslant d$
		\[
			\norm{ \md \inprod{ \md F^k , \md F^l }_{\mathcal{H} } }_{\mathcal{H}} 
			\leqslant 
			\norm{ \md^2 F^k }_{\mathcal{H}\otimes \mathcal{H}} \norm{ \md F^l }_{\mathcal{H}} +  \norm{ \md F^k }_{\mathcal{H}}\norm{ \md^2 F^l }_{\mathcal{H}\otimes \mathcal{H}}.
		\]
		Then, for each $j = 1, \dotsc, d$, it holds that
		\begin{align*}
			&\pheq
			\sum_{i,k,l=1}^d \abs{ G^i \big( \bm{\Gamma}(\bfF)^{-1} \big)_{j,k} \big( \bm{\Gamma}(\bfF)^{-1} \big)_{l,i} }
			\norm{ \md \inprod{ \md F^k , \md F^l }_{\mathcal{H} } }_{\mathcal{H}}  \\
			&\leqslant \sum_{k=1}^d \abs{ \big( \bm{\Gamma}(\bfF)^{-1} \big)_{j,k} } \norm{ \md^2 F^k }_{\mathcal{H}\otimes \mathcal{H}} \sum_{i,l=1}^d \abs{ G^i  \big( \bm{\Gamma}(\bfF)^{-1} \big)_{l,i} } 
			 \norm{ \md F^l }_{\mathcal{H}} \\
			&\pheq + \sum_{k=1}^d \abs{ \big( \bm{\Gamma}(\bfF)^{-1} \big)_{j,k} } \norm{ \md F^k }_{\mathcal{H}} \sum_{i,l=1}^d \abs{ G^i \big( \bm{\Gamma}(\bfF)^{-1} \big)_{l,i} } 
			 \norm{ \md^2 F^l }_{\mathcal{H}\otimes \mathcal{H}} .
		\end{align*}
		The Cauchy-Schwarz inequality implies that
		\begin{align*}
		\left\{ \sum_{i,l=1}^d \abs{ G^i \big( \bm{\Gamma}(\bfF)^{-1} \big)_{l,i}  \norm{ \md F^l }_{\mathcal{H}} } \right\}^2 
		&\leqslant
		\left[{ \sum_{i=1}^d (G^i)^2 }\right] 
		\left[{ \sum_{i,l=1}^d \big( \bm{\Gamma}(\bfF)^{-1} \big)_{l,i}^2 }\right] 
		\left[{ \sum_{l=1}^d \norm{ \md F^l }_{\mathcal{H}}^2 }\right] \\
		&\leqslant 
		\abs{ \bfG }^2 \hsnorm{ \bm{\Gamma}(\bfF)^{-1} }^2 \norm{ \md \bfF }_{\mathcal{H}}^2 ,
		\end{align*}			
		and similarly
		\[
			\left\{ \sum_{i,l=1}^d \abs{ G^i \big( \bm{\Gamma}(\bfF)^{-1} \big)_{l,i} } 
			\norm{ \md^2 F^l }_{\mathcal{H}\otimes \mathcal{H}} \right\}^2
			\leqslant 
			\abs{ \bfG }^2 \hsnorm{ \bm{\Gamma}(\bfF)^{-1} }^2 \norm{ \md^2 \bfF }_{\mathcal{H}}^2 .
		\]
		Combining inequalities above, we can obtain 
		\begin{equation*}
		\begin{split}
			\mathcal{I}_{2,3} &\leqslant 
			\abs{ \bfG } \hsnorm{ \bm{\Gamma}(\bfF)^{-1} } \norm{ \md \bfF }_{\mathcal{H}} \sum_{j, k = 1}^d \abs{ \big( \bm{\Gamma}(\bfF)^{-1} \big)_{j,k} } \norm{ \md^2 F^k }_{\mathcal{H}\otimes \mathcal{H}} \norm{ \md F^j }_{\mathcal{H}} \\
			&\pheq + \abs{ \bfG } \hsnorm{ \bm{\Gamma}(\bfF)^{-1} } \norm{ \md^2 \bfF }_{\mathcal{H}\otimes \mathcal{H}} \sum_{j, k = 1}^d \abs{ \big( \bm{\Gamma}(\bfF)^{-1} \big)_{j,k} } \norm{ \md F^k }_{\mathcal{H}} \norm{ \md F^j }_{\mathcal{H}} \\
			&\leqslant 2  \hsnorm{ \bm{\Gamma}(\bfF)^{-1} }^2  \abs{ \bfG } \norm{ \md \bfF }_{\mathcal{H}}^2 \norm{ \md^2 \bfF }_{\mathcal{H}\otimes \mathcal{H}} .
		\end{split}
		\end{equation*}

		The proof is complete.
	\end{proof}


	\begin{proof}[\textbf{Proof of Lemma \ref{lem3-9}}]
		According to Markov's inequality, it holds
		\begin{align*}
		\bbP (E_{j,k}^1)
		&\leqslant \bbP \left( \eta_j \sum_{i = j + 1}^k (\varphi (\bm{\xi}_i) - \bbE \varphi (\bm{\xi}_i)) > (k - j) \eta_k \right) \\
		&\leqslant \left( \frac{\eta_j}{\eta_k} \right)^4 \bbE \abs{\frac{1}{k-j} \sum_{i = j + 1}^k \big(\varphi (\bm{\xi}_i) - \bbE \varphi (\bm{\xi}_i) \big)}^4,
		\end{align*}
		where in the first inequality we use $t_k - t_j \geqslant (k-j) \eta_k$ due to $\eta_i \leqslant \eta_j$ for any $i > j$. Assumption \ref{Assump.3} implies that
		\begin{equation*}
		\frac{\eta_j}{\eta_k}
		= \prod_{i = j + 1}^k \frac{\eta_{i-1}}{\eta_i}
		\leqslant \prod_{i = j + 1}^k (1 + \omega \eta_i)
		\leqslant \prod_{i = j + 1}^k \rme^{\omega \eta_i}
		= \rme^{\omega (t_k - t_j)},
		\end{equation*}
		where the first inequality follows from the condition $\eta_{i-1} \leqslant \eta_i (1 + \omega \eta_i)$. Assumption \ref{Assump.4} guarantees that $\bbE \abs{\varphi (\bm{\xi}_i) - \bbE [\varphi (\bm{\xi}_i)]}^4 <  \infty$, then together with Lemma \ref{lem-B-1}, the following holds for constant $C>0$,
		\begin{equation*}
		\bbP (E_{j,k}^1)
		\leqslant \frac{C \rme^{4 \omega (t_k - t_j)}}{(k-j)^2}
		\leqslant \frac{C \rme^{4 \omega (t_k - t_j)}}{(t_k - t_j)^2} \eta_j^2
		\leqslant \frac{C \rme^{6 \omega (t_k - t_j)}}{(t_k - t_j)^2} \eta_k^2.
		\end{equation*}
		
		Since Assumption \ref{Assump.4} implies
		$\bbE \abs{\varphi (\bm{\xi}_i)^2 - \bbE [\varphi (\bm{\xi}_i)^2]}^4 < + \infty$, $\bbP (E_{j,k}^2)$ can be analogously estimated only with $\varphi (\bm{\xi}_i)$ replaced by $\varphi (\bm{\xi}_i)^2$.
	\end{proof}
	
	\vskip 2mm
	
 Recall SDE \eqref{SDE2}, we have 
	\begin{equation*}
	\left\{
	\begin{aligned}
	\dismomentum_{t_j} &= (1-\gamma \eta_j) \dismomentum_{t_{j-1}} - \frac1N \sum_{r = 1}^N \nabla F (\disposition_{t_{j-1}}, \xi_j^r) \eta_j + \beta ( \bfB_{t_j} - \bfB_{t_{j-1}} ), \\
	\disposition_{t_j} &= \eta_j \dismomentum_{t_{j-1}} + \disposition_{t_{j-1}} ,
	\end{aligned}
	\right.
	\end{equation*}	
	for $j \geqslant 1$ and initial value $(\dismomentum_0 , \disposition_0) = (\bfm , \bfx)$. By \cite[Section 2.2.2]{MR2200233}, the Malliavin derivative satisfies
	\begin{equation}
	\label{eq:Malliavin derivative of EM-dis.}
	\begin{bmatrix}
	\md_s\dismomentum_{t_j} \\
	\md_s\disposition_{t_j}
	\end{bmatrix}
	=
	\bfE_j^{\top}
	\begin{bmatrix}
	\md_s\dismomentum_{t_{j-1}} \\
	\md_s\disposition_{t_{j-1}}
	\end{bmatrix}
	+
	\beta
	\begin{bmatrix}
	\mathbf{I}_d \\
	\mathbf{0}
	\end{bmatrix}
	\mathbf{1}_{[t_{j-1}, t_j) } (s) ,
	\end{equation}
	where the matrix $\bfE_j$ is defined for simplicity as 
	\begin{equation} \label{def:Ej}
	\bfE_j \equiv 
	\begin{bmatrix}
	(1-\gamma \eta_{j}) \mathbf{I}_d & \eta_{j} \mathbf{I}_d \\
	- {N}^{-1} \sum_{r = 1}^N \nabla^2 F (\disposition_{t_{j-1}}, \xi_j^r) \eta_{j} & \mathbf{I}_d
	\end{bmatrix}.
	\end{equation}
	Besides, define $\bfS_j$ as
	\begin{equation*}
	\bfS_k \equiv \mathbf{I}_{2d}, \qquad
	\bfS_j \equiv \bfE_{j+1} \dotsc \bfE_k, \quad j = 1, \dots, k-1.
	\end{equation*}
	Given $(\dismomentum_{t_\ell}, \disposition_{t_\ell}) = \bfz$, the Malliavin derivative of $\widetilde{\bfY}_{t_\ell , t_k}^{\bfz} = (\dismomentum_{t_k}, \disposition_{t_k})$ has the following recursion: 
	\begin{equation}
	\label{eq:recursion of Malliavin}
	\begin{bmatrix}
	\md_s \dismomentum_{t_{k}} \\
	\md_s \disposition_{t_{k}}
	\end{bmatrix}
	=
	\beta \sum_{j = \ell + 1}^k \bfS_j^{\top} 
	\begin{bmatrix}
	\mathbf{I}_d \\ \mathbf{0}
	\end{bmatrix}
	\mathbf{1}_{[t_{j-1}, t_j)} (s),
	\end{equation}
	and the corresponding Malliavin matrix is given by
	\begin{equation}
	\label{eq:Malliavin Matrix**}
	\bm{\Gamma} \big( \widetilde{\bfY}_{t_\ell , t_k}^{\bfz} \big) = \int_{t_\ell}^{t_k} 
	\begin{bmatrix}
	\md_s \dismomentum_{t_{k}} \\
	\md_s \disposition_{t_{k}}
	\end{bmatrix}
	\begin{bmatrix}
	\md_s \dismomentum_{t_{k}} \\
	\md_s \disposition_{t_{k}}
	\end{bmatrix}^{\top} 
	\dd s
	= \beta^2 \sum_{j = \ell + 1}^k \eta_j \bfS_j^{\top} 
	\begin{bmatrix}
	\mathbf{I}_d \\ \mathbf{0}
	\end{bmatrix}
	\begin{bmatrix}
	\mathbf{I}_d & \mathbf{0}
	\end{bmatrix} \bfS_j.
	\end{equation}
	Before estimating the lower bound of eigenvalues of $\bm{\Gamma} ( \widetilde{\bfY}_{t_\ell , t_k}^{\bfz} )$, let us derive an estimate of $\opnorm{\bfS_j-\mathbf{I}_{2d}}$ first.
	
	\begin{lemma}
		\label{lemma:operator norm of S_j}
		For any fixed $k\geqslant1$, the following holds for each $1\leqslant j \leqslant k$, 
		\begin{equation*}
		\opnorm{\bfS_j-\mathbf{I}_{2d}} \leqslant \exp \left\{ \sum_{i = j+1}^k \eta_i \big(\varphi (\bm{\xi}_i) + \gamma + 1 \big) \right\} - 1,
		\end{equation*}
		where $\varphi (\bm{\xi}_i) = N^{-1} \sum_{r=1}^N \sup_{ \bfx\in\bbR^d } \opnorm{\nabla^2 F(\bfx, \xi_i^r)}$, $i = j+1, \dotsc, k$, are in \eqref{e:Phi}.
	\end{lemma}
	
	\begin{proof}
		We prove this by a backward induction. By definition, $\bfS_k = \mathbf{I}_{2d}$, so this result holds for $j=k$. Assume that
		\begin{equation*}
		\opnorm{\bfS_\ell-\mathbf{I}_{2d}} \leqslant \exp \left\{\sum_{i = \ell +1}^k \eta_i (\varphi (\bm{\xi}_i) + \gamma + 1) \right\} - 1,
		\end{equation*}
		holds for some $\ell \geqslant 2$, then
		\begin{equation*}
		\opnorm{\bfS_{\ell-1} - \mathbf{I}_{2d} } = \opnorm{ \bfE_\ell \bfS_\ell - \bfE_\ell + \bfE_\ell - \mathbf{I}_{2d} } \leqslant \opnorm{ \bfE_\ell } \opnorm{ \bfS_\ell - \mathbf{I}_{2d} } + \opnorm{ \bfE_\ell - \mathbf{I}_{2d} }.
		\end{equation*}
		For any matrix $ \mathbf{K} = \begin{bmatrix}
			\bfA & \bfB \\ \bfC & \zero
		\end{bmatrix} \in \bbR^{2d \times 2d}$ and $\bfw = (\bfu, \bfv) \in \bbR^d$, it holds
		\[
			\abs{\mathbf{K} \bfw }= \abs{ \bfA \bfu + \bfB \bfv + \bfC \bfu }\leqslant [\opnorm{\bfA} + \opnorm{\bfB} + \opnorm{C}] \abs{\bfw} ,
		\]
		which implies that $\opnorm{\mathbf{K}} \leqslant \opnorm{\bfA} + \opnorm{\bfB} + \opnorm{C}$.Then, we have 
		\begin{equation*}
		\opnorm{ \bfE_\ell - \mathbf{I}_{2d}  }
		= \eta_\ell \opnorm{ \begin{bmatrix}
			-\gamma \mathbf{I}_d & \mathbf{I}_d \\
			-\frac{1}{N} \sum_{r=1}^N \nabla^2 F (\disposition_{t_{\ell-1}}, \xi_\ell^r) & \mathbf{0}
			\end{bmatrix} }
		\leqslant \eta_\ell \big(\varphi (\bm{\xi}_\ell) + \gamma + 1 \big).
		\end{equation*}
		Hence, we obtain the desired result by
			\begin{align*}
			\opnorm{\bfS_{\ell-1} - \mathbf{I}_{2d} }
			&\leqslant \big[ \opnorm{ \bfE_\ell - \mathbf{I}_{2d} } + 1 \big] \opnorm{ \bfS_\ell - \mathbf{I}_{2d} } - 1 \\
			&\leqslant \exp\left\{ \sum_{i = \ell}^k \eta_i \big(\varphi (\bm{\xi}_i) + \gamma + 1 \big) \right\} - 1.  
			\end{align*}	
		The proof is complete.	
	\end{proof}
	Next, we prove Lemmas \ref{lem3-10} -- \ref{lem3-13}. To make notations simple, we define
	\begin{equation}\label{def:b & tilde b}
		\bfb (\bfy) = 
		\begin{bmatrix}
			-\gamma \bfm - \nabla f(\bfx) \\
			\bfm
		\end{bmatrix}, \qquad
		\widetilde{\bfb} (\bfy, \bm{\xi}) = 
		\begin{bmatrix}
			-\gamma \bfm - {N}^{-1} \sum_{i = 1}^N \nabla F (\bfx, \xi^i) \\
			\bfm
		\end{bmatrix},
	\end{equation}
	for $\bfy = (\bfm, \bfx)$ and $\bm{\xi} = (\xi^1, \dots, \xi^N)$.
	
	\begin{proof}[\textbf{Proof of Lemma \ref{lem3-10}}]
		(i) By the definition and \eqref{eq:Malliavin Matrix**}, 
		\begin{equation*}
		\norm{ \md \widetilde{\bfY}_{t_\ell , t_k}^{\bfz} }_{\mathcal{H}}^2
		= \tr \left( \bm{\Gamma} \big( \widetilde{\bfY}_{t_\ell , t_k}^{\bfz} \big) \right)
		= \beta^2 \sum_{j = \ell + 1}^k \eta_j \tr \left( \bfS_j^{\top} 
		\begin{bmatrix}
		\mathbf{I}_d \\ \mathbf{0}
		\end{bmatrix}
		\begin{bmatrix}
		\mathbf{I}_d & \mathbf{0}
		\end{bmatrix} \bfS_j \right)
		\leqslant 2 d \beta^2 \sum_{j = \ell + 1}^k \eta_j \opnorm{\bfS_j}^2.
		\end{equation*}
		Together with H\"older's inequality, we have
		\begin{equation} \label{eq:lem3-10pr1}
		\norm{ \md \widetilde{\bfY}_{t_\ell , t_k}^{\bfz} }_{\mathcal{H}}^8
		\leqslant C d^4 (t_k - t_\ell)^3 \sum_{j = \ell + 1}^k \eta_j \opnorm{\bfS_j}^8.
		\end{equation}
		Notice that $\bfS_j = \bfE_{j+1} \dots \bfE_k$ with $\bfE_j$ defined in \eqref{def:Ej} implies
		\begin{equation*}
		\bbE \opnorm{\bfS_j}^8
		\leqslant \bbE \left\{ \bbE \left[ \left( 1 + \opnorm{\bfE_k - \mathbf{I}_{2d}} \right)^8 \middle\vert (\disposition_{t_{i-1}}, \bm{\xi}_i)_{j+1 \leqslant i \leqslant k-1} \right] \opnorm{\bfE_{j+1} \dots \bfE_{k-1}}^8 \right\},
		\end{equation*}
		and 
		\begin{align*}
		\left( 1 + \opnorm{\bfE_k - \mathbf{I}_{2d}} \right)^8
		&\leqslant \left[ 1 + \eta_k \left( \frac{1}{N} \sum_{r = 1}^N \sup_{ \bfx \in \bbR^d } \opnorm{\nabla^2 F (\bfx, \xi_k^r)} + \gamma + 1 \right) \right]^8 \\
		&\leqslant \frac{1}{1 - 7 \eta_k} \left[ 1 + \eta_k \left( \frac{1}{N} \sum_{r = 1}^N \sup_{ \bfx \in \bbR^d } \opnorm{\nabla^2 F (\bfx, \xi_k^r)} + \gamma + 1 \right)^8 \right] \\
		&\leqslant \left( 1 + \frac{7 \eta_k}{1 - 7 \eta_1} \right) \left[ 1 + C \eta_k \left( \frac{1}{N} \sum_{r = 1}^N \sup_{ \bfx \in \bbR^d } \opnorm{\nabla^2 F (\bfx, \xi_k^r)}^8 + 1 \right) \right],
		\end{align*}
		where in the second inequality we use $(1 + \eta_k x)^8 \leqslant 1 + \eta_k x^8 + 7 \eta_k (1 + \eta_k x)^8$ for $x \geqslant 0$. According to Assumption \ref{Assump.4} , we have
		\begin{equation*}
		\bbE \opnorm{\bfS_j}^8
		\leqslant (1 + C \eta_k) \bbE \opnorm{\bfE_{j+1} \dots \bfE_{k-1}}^8
		\leqslant \rme^{C \eta_k} \bbE \opnorm{\bfE_{j+1} \dots \bfE_{k-1}}^8 .
		\end{equation*}
		Then, applying this method recursively shows that $\bbE \opnorm{\bfS_j}^8 \leqslant \rme^{C (t_k - t_j)}$ for $j = \ell + 1, \dots, k$, so \eqref{eq:lem3-10pr1} and condition $t_k - t_{\ell} \leqslant 1/[5(B_1 + \gamma + 2)]$ imply that
			\begin{equation*}
			\bbE \norm{ \md \widetilde{\bfY}_{t_\ell , t_k}^{\bfz} }_{\mathcal{H}}^8 \leqslant C d^4 (t_k - t_{\ell})^3 \sum_{j = \ell + 1}^k \eta_j \rme^{C(t_k - t_{\ell})} \leqslant C d^4 .
			\end{equation*}
		Analogously, it can be shown that $\bbE \lVert \md \widetilde{\bfY}_{t_\ell , t_{k-1}}^{\bfz} \rVert_{\mathcal{H}}^8 \leqslant C d^4$. 
		
		For any $t \in [t_{k-1}, t_k]$ and $\bfh \in \mathcal{H}$ with $\norm{\bfh}_{ \mathcal{H} } \leqslant 1$, by \cite{MR0942019}
		\begin{equation}
			\label{eq:first Mal}
			\begin{split}
		\md^{\bfh} \bfY_{t_{k-1}, t} ( \widetilde{\bfY}_{t_\ell, t_{k-1}}^{\bfz} ) 
		&=  \md^{\bfh} \widetilde{\bfY}_{t_\ell, t_{k-1}}^{\bfz} 
		+ \beta \int_{t_{k-1}}^t
		\begin{bmatrix}
		\mathbf{I}_d \\ \mathbf{0}
		\end{bmatrix}
		\bfh_s \, \dd s  \\
		&\quad + \int_{t_{k-1}}^t \nabla \bfb ( \bfY_{ t_{k-1} , u } ( \widetilde{\bfY}_{t_\ell, t_{k-1}}^{\bfz} ) ) \md^{\bfh} \bfY_{ t_{k-1} , u } ( \widetilde{\bfY}_{t_\ell, t_{k-1}}^{\bfz} ) \, \dd u.
		\end{split}
		\end{equation}
		Since $\lVert \nabla \bfb (\bfy) \rVert_{\operatorname{op}} \leqslant L + \gamma + 1$ where $\bfb(\bfy)$ is in \eqref{def:b & tilde b}, it follows that
		\begin{equation*}
		\norm{ \md \bfY_{t_{k-1}, t} ( \widetilde{\bfY}_{t_\ell, t_{k-1}}^{\bfz} ) }_{\mathcal{H}} 
		\leqslant \norm{ \md \widetilde{\bfY}_{t_\ell, t_{k-1}}^{\bfz} }_{\mathcal{H}} + \beta \sqrt{\eta_k} + C \int_{t_{k-1}}^t \norm{ \md \bfY_{ t_{k-1} , u } ( \widetilde{\bfY}_{t_\ell, t_{k-1}}^{\bfz} ) }_{\mathcal{H}} \dd u.
		\end{equation*}
		Then Gr\"onwall's inequality implies
		\begin{equation*}
		\bbE \norm{ \md \bfY_{t_{k-1}, t} ( \widetilde{\bfY}_{t_\ell, t_{k-1}}^{\bfz} ) }_{\mathcal{H}}^8
		\leqslant C \left( \bbE \norm{ \md \widetilde{\bfY}_{t_\ell, t_{k-1}}^{\bfz} }_{\mathcal{H}}^8 + 1 \right)
		\leqslant C d^4, \quad t \in [t_{k-1}, t_k].
		\end{equation*}
		
		\noindent (ii) Given event $(E_{\ell, k}^1)^c$, i.e., $\sum_{i = \ell + 1}^k \eta_i ( \varphi (\bm{\xi}_i) - \bbE \varphi (\bm{\xi}_i) ) \leqslant t_k - t_\ell$, Lemma \ref{lemma:operator norm of S_j} and Assumption \ref{Assump.4} derive the following equation for some constant $C>0$, 
		\begin{equation*}
		\opnorm{\bfS_j}
		\leqslant \opnorm{\bfS_j-\mathbf{I}_{2d}} + 1
		\leqslant \rme^{\sum_{i = j+1}^k \eta_i (\varphi (\bm{\xi}_i) + \gamma + 1)}
		\leqslant \rme^{(t_k - t_\ell) (\bbE \varphi (\bm{\xi}_i) + \gamma + 2)} \leqslant C,
		\end{equation*}
		for $j = \ell + 1, \dots, k$. According to \eqref{eq:lem3-10pr1}, the following holds ,
		\begin{equation*}
		\norm{ \md \widetilde{\bfY}_{t_\ell , t_k}^{\bfz} }_{\mathcal{H}}^8
		\leqslant C d^4 (t_k - t_\ell)^3 \sum_{j = \ell + 1}^k \eta_j \opnorm{\bfS_j}^8
		\leqslant C d^4 
		\quad \text{on}\ \ ( E_{\ell, k}^1 )^c .
		\end{equation*}
		
		The proof is complete.
	\end{proof}
	
	\begin{proof}[\textbf{Proof of Lemma \ref{lem3-11}}]
		For any $\bfh_1, \bfh_2 \in \mathcal{H}$ with $\norm{\bfh_1}_{ \mathcal{H} }, \norm{\bfh_2}_{ \mathcal{H} } \leqslant 1$, by \cite{MR0942019}, we have
		\begin{align*}
		\md^{\bfh_1} \md^{\bfh_2} \widetilde{\bfY}_{t_{\ell}, t_k}^{\bfz}
		&= \md^{\bfh_1} \md^{\bfh_2} \widetilde{\bfY}_{t_{\ell}, t_{k-1}}^{\bfz} 
		+ \eta_k \nabla \widetilde{\bfb}\big( \widetilde{\bfY}_{t_{\ell}, t_{k-1}}^{\bfz} , \bm{\xi}_k \big) \md^{\bfh_1} \md^{\bfh_2} \widetilde{\bfY}_{t_\ell, t_{k-1}}^{\bfz} \\
		& \quad + \eta_k \nabla^2 \widetilde{\bfb}\big( \widetilde{\bfY}_{t_\ell, t_{k-1}}^{\bfz} , \bm{\xi}_k \big) \md^{\bfh_1} \widetilde{\bfY}_{t_\ell, t_{k-1}}^{\bfz} \md^{\bfh_2} \widetilde{\bfY}_{t_\ell, t_{k-1}}^{\bfz} ,
		\end{align*}
		where $\widetilde{\bfb}(\bfy,\bm{\xi})$ is defined in \eqref{def:b & tilde b}. Hence,
		\begin{align*}
		\abs{ \md^{\bfh_1} \md^{\bfh_2} \widetilde{\bfY}_{t_\ell, t_k}^{\bfz} }
		&\leqslant \abs{ \md^{\bfh_1} \md^{\bfh_2} \widetilde{\bfY}_{t_\ell, t_{k-1}}^{\bfz} }
		+ \eta_k \opnorm{ \nabla \widetilde{\bfb}\big( \widetilde{\bfY}_{t_\ell, t_{k-1}}^{\bfz} , \bm{\xi}_k \big) } \abs{ \md^{\bfh_1} \md^{\bfh_2} \widetilde{\bfY}_{t_\ell, t_{k-1}}^{\bfz} } \\
		&\quad + \eta_k \opnorm{ \nabla^2 \widetilde{\bfb}\big( \widetilde{\bfY}_{t_\ell, t_{k-1}}^{\bfz} , \bm{\xi}_k \big) } \abs{ \md^{\bfh_1} \widetilde{\bfY}_{t_\ell,t_{k-1}}^{\bfz} } \abs{ \md^{\bfh_2} \widetilde{\bfY}_{t_\ell, t_{k-1}}^{\bfz} },
		\end{align*}
		which implies that
		\begin{align*}
		\norm{ \md^2 \widetilde{\bfY}_{t_\ell, t_k}^{\bfz} }_{\mathcal{H}\otimes \mathcal{H}}
		&\leqslant \norm{ \md^2 \widetilde{\bfY}_{t_\ell, t_{k-1}}^{\bfz} }_{\mathcal{H}\otimes \mathcal{H}}
		+ \eta_k \opnorm{ \nabla \widetilde{\bfb}\big( \widetilde{\bfY}_{t_\ell, t_{k-1}}^{\bfz} , \bm{\xi}_k \big) } \norm{ \md^2 \widetilde{\bfY}_{t_\ell, t_{k-1}}^{\bfz} }_{\mathcal{H}\otimes \mathcal{H}} \\
		&\quad + \eta_k \opnorm{ \nabla^2 \widetilde{\bfb}\big( \widetilde{\bfY}_{t_\ell, t_{k-1}}^{\bfz} , \bm{\xi}_k \big) } \norm{ \md \widetilde{\bfY}_{t_\ell,t_{k-1}}^{\bfz} }_{\mathcal{H}}^2.
		\end{align*}
		Notice that $\eta_k\leqslant \eta_1$ and $(x+\eta_k y)^4 \leqslant x^4+\eta_k y^4 + 3\eta_k (x+\eta_k y)^4$ for $x, y \geqslant 0$, so
		\begin{equation}
			\label{ineq:B30}
			\begin{split}
		\bbE \norm{ \md^2 \widetilde{\bfY}_{t_\ell, t_k}^{\bfz} }_{\mathcal{H}\otimes \mathcal{H}}^4
		&\leqslant \left( 1+\frac{3\eta_k}{1-3\eta_1} \right) \bbE \norm{ \md^2 \widetilde{\bfY}_{t_\ell, t_{k-1}}^{\bfz} }_{\mathcal{H}\otimes \mathcal{H}}^4 \\
		&\quad + \frac{8\eta_k}{1-3\eta_1} \left\{ \bbE \left[ \opnorm{ \nabla \widetilde{\bfb}\big( \widetilde{\bfY}_{t_\ell, t_{k-1}}^{\bfz} , \bm{\xi}_k \big) }^4 \norm{ \md^2 \widetilde{\bfY}_{t_\ell, t_{k-1}}^{\bfz} }_{\mathcal{H}\otimes \mathcal{H}}^4 \right] \right. \\
		&\quad \left. + \bbE \left[ \opnorm{ \nabla^2 \widetilde{\bfb}\big( \widetilde{\bfY}_{t_\ell, t_{k-1}}^{\bfz} , \bm{\xi}_k \big) }^4 \norm{ \md \widetilde{\bfY}_{t_\ell,t_{k-1}}^{\bfz} }_{\mathcal{H}}^8 \right] \right\}.
		\end{split}
		\end{equation}
		Combining Assumption \ref{Assump.4}, Lemma \ref{lem3-10}, and the independence between $\bm{\xi}_k$ and other random variables on the right-hand side, we have
		\begin{gather} \label{e:B31}
			\begin{split}
				&\pheq
				\bbE \left[ \opnorm{ \nabla^2 \widetilde{\bfb}\big( \widetilde{\bfY}_{t_\ell, t_{k-1}}^{\bfz} , \bm{\xi}_k \big) }^4 \norm{ \md \widetilde{\bfY}_{t_\ell,t_{k-1}}^{\bfz} }_{\mathcal{H}}^8 \right] \\
				&= \bbE \left\{ \bbE_{\bm{\xi}_k} \left[ \opnorm{ \nabla^2 \widetilde{\bfb}\big( \widetilde{\bfY}_{t_\ell, t_{k-1}}^{\bfz} , \bm{\xi}_k \big) }^4 \right] \norm{ \md \widetilde{\bfY}_{t_\ell,t_{k-1}}^{\bfz} }_{\mathcal{H}}^8 \right\} \\
				&\leqslant C \bbE \norm{ \md \widetilde{\bfY}_{t_\ell,t_{k-1}}^{\bfz} }_{\mathcal{H}}^8 
				\leqslant C d^4 .
			\end{split}
		\end{gather}
		Similarly, by Assumption \ref{Assump.4}, and the independence between $\bm{\xi}_k$ and other random variables, we also have
		\begin{equation} \label{e:B32}
		\begin{split}
		\bbE \left[ \opnorm{ \nabla \widetilde{\bfb}\big( \widetilde{\bfY}_{t_\ell, t_{k-1}}^{\bfz} , \bm{\xi}_k \big) }^4 \norm{ \md^2 \widetilde{\bfY}_{t_\ell, t_{k-1}}^{\bfz} }_{\mathcal{H}\otimes \mathcal{H}}^4 \right]
		& \leqslant C \bbE \norm{ \md^2 \widetilde{\bfY}_{t_\ell, t_{k-1}}^{\bfz} }_{\mathcal{H}\otimes \mathcal{H}}^4.
		\end{split}
		\end{equation}
		Combining \eqref{ineq:B30}, \eqref{e:B31} and \eqref{e:B32}, since $\md^2 \widetilde{\bfY}_{t_\ell, t_\ell}^{\bfz} = \zero$, we have by an induction 
		\begin{align*}
			\bbE \norm{ \md^2 \widetilde{\bfY}_{t_\ell, t_k}^{\bfz} }_{\mathcal{H}\otimes \mathcal{H}}^4
			&\leqslant (1 + C \eta_k) \bbE \norm{ \md^2 \widetilde{\bfY}_{t_\ell, t_{k-1}}^{\bfz} }_{\mathcal{H}\otimes \mathcal{H}}^4 + C d^4 \eta_k \\
			&\leqslant C d^4 \sum_{i = \ell + 1}^k \eta_i \prod_{j=i+1}^k ( 1+C \eta_j)
			\leqslant C d^4 (t_k-t_\ell) \rme^{C (t_k-t_\ell)}.
		\end{align*}
		Due to $t_k - t_{\ell} \leqslant 1/[5(B_1 + \gamma +1)]$, we can obtain the desired.
		
		For $\bfY_{t_{k-1}, t}( \widetilde{\bfY}_{t_\ell, t_{k-1}}^{\bfz} )$, $t \in [t_{k-1}, t_k]$ and any $\bfh_i \in \mathcal{H}$ with $\norm{\bfh_i}_{ \mathcal{H} } \leqslant 1$, $i =1, 2$, by \cite{MR0942019}, the second order Malliavin derivation satisfies
		\begin{multline*}
		\md^{\bfh_1} \md^{\bfh_2} \bfY_{t_{k-1}, t} ( \widetilde{\bfY}_{t_\ell, t_{k-1}}^{\bfz} )
		= \md^{\bfh_1} \md^{\bfh_2} \widetilde{\bfY}_{t_\ell, 	t_{k-1}}^{\bfz} \\
		+ \int_{t_{k-1}}^t \nabla \bfb\big( \bfY_{t_{k-1}, u} ( \widetilde{\bfY}_{t_\ell, t_{k-1}}^{\bfz} ) \big) \md^{\bfh_1} \md^{\bfh_2} \bfY_{t_{k-1}, u} ( \widetilde{\bfY}_{t_\ell, t_{k-1}}^{\bfz} ) \, \dd u \\
		+ \int_{t_{k-1}}^t \nabla^2 b \big( \bfY_{t_{k-1}, u} ( \widetilde{\bfY}_{t_\ell, t_{k-1}}^{\bfz} ) \big) \md^{\bfh_1} \bfY_{t_{k-1}, u} ( \widetilde{\bfY}_{t_\ell, t_{k-1}}^{\bfz} )  \md^{\bfh_2} \bfY_{t_{k-1}, u} ( \widetilde{\bfY}_{t_\ell, t_{k-1}}^{\bfz} ) \, \dd u.
		\end{multline*}
		Since $\opnorm{\nabla \bfb (\bfy)} \leqslant L + \gamma + 1$ and $\opnorm{\nabla^2 b (\bfy)} \leqslant B_2$ due to Assumption \ref{Assump.4}, we have
		\begin{align*}
		\norm{\md^2 \bfY_{t_{k-1}, t} ( \widetilde{\bfY}_{t_\ell, t_{k-1}}^{\bfz} )}_{\mathcal{H}\otimes \mathcal{H}}
		&\leqslant \norm{\md^2 \widetilde{\bfY}_{t_\ell, 	t_{k-1}}^{\bfz}}_{\mathcal{H}\otimes \mathcal{H}}
		+ C \int_{t_{k-1}}^t \norm{\md^2 \bfY_{t_{k-1}, u} ( \widetilde{\bfY}_{t_\ell, t_{k-1}}^{\bfz} )}_{\mathcal{H}\otimes \mathcal{H}} \dd u \\
		&\quad + C \int_{t_{k-1}}^{t_k} \norm{\md \bfY_{t_{k-1}, u} ( \widetilde{\bfY}_{t_\ell, t_{k-1}}^{\bfz} )}_{\mathcal{H}}^2 \dd u.
		\end{align*}
		By Gr\"onwall's inequality,
		\begin{equation*}
		\norm{\md^2 \bfY_{t_{k-1}, t_k} ( \widetilde{\bfY}_{t_\ell, t_{k-1}}^{\bfz} )}_{\mathcal{H}\otimes \mathcal{H}}
		\leqslant C \norm{\md^2 \widetilde{\bfY}_{t_\ell, 	t_{k-1}}^{\bfz}}_{\mathcal{H}\otimes \mathcal{H}}
		+ C \int_{t_{k-1}}^{t_k} \norm{\md \bfY_{t_{k-1}, u} ( \widetilde{\bfY}_{t_\ell, t_{k-1}}^{\bfz} )}_{\mathcal{H}}^2 \dd u.
		\end{equation*}
		We have shown that $\bbE \lVert \md^2 \widetilde{\bfY}_{t_\ell, t_{k-1}}^{\bfz} \rVert_{\mathcal{H}\otimes \mathcal{H}}^4 \leqslant C d^4$, together with Lemma \ref{lem3-10} and H\"older's inequality, we get
		\begin{equation*}
		\bbE \norm{\md^2 \bfY_{t_{k-1}, t_k} ( \widetilde{\bfY}_{t_\ell, t_{k-1}}^{\bfz} )}_{\mathcal{H}\otimes \mathcal{H}}^4 \leqslant C d^4. 
		\end{equation*}
		
		The proof is complete.
	\end{proof}
	
	\begin{proof}[\textbf{Proof of Lemma \ref{lem3-12}}]
		Recall the SDEs  \eqref{SDE1} and \eqref{SDE2}, and define 
		\begin{equation*}
		\bm{\Xi}_{t_\ell, t_k}
		:= \bfY_{t_{k-1}, t_k} ( \widetilde{\bfY}_{t_\ell , t_{k-1}}^{\bfz} ) - \widetilde{\bfY}_{t_\ell , t_k}^{\bfz}
		= \int_{t_{k-1}}^{t_k} \left( \bfb (\bfY_{t_{k-1}, t} ( \widetilde{\bfY}_{t_\ell , t_{k-1}}^{\bfz} )) - \widetilde{\bfb} (\widetilde{\bfY}_{t_\ell , t_{k-1}}^{\bfz}, \bm{\xi}_k) \right) \dd t, 
		\end{equation*}
		where functions $\bfb$ and $\widetilde{\bfb}$ are defined in \eqref{def:b & tilde b}. 
		For any $\bfh \in \mathcal{H}$ with $\norm{\bfh}_{ \mathcal{H} } \leqslant 1 $, by \cite{MR0942019}, we have
		\begin{equation} \label{eq:lemma3.12pr1}
		\begin{split}
			\md^{\bfh} \bm{\Xi}_{t_\ell, t_k}
			&= \int_{t_{k-1}}^{t_k} \left( \nabla \bfb (\bfY_{t_{k-1}, t} ( \widetilde{\bfY}_{t_\ell , t_{k-1}}^{\bfz} )) \md^{\bfh} \bfY_{t_{k-1}, t} ( \widetilde{\bfY}_{t_\ell , t_{k-1}}^{\bfz} ) - \nabla \widetilde{\bfb} (\widetilde{\bfY}_{t_\ell , t_{k-1}}^{\bfz}, \bm{\xi}_k) \md^{\bfh} \widetilde{\bfY}_{t_\ell , t_{k-1}}^{\bfz} \right) \dd t .
		\end{split}
		\end{equation}
		
		\noindent (i) Observe that $\md^{\bfh} \bm{\Xi}_{t_\ell, t_k} = \mathcal{I}_1 + \mathcal{I}_2 + \mathcal{I}_3$, where 
		\begin{equation}
			\label{eq:lemma3.12pr2}
			\begin{split}
				\mathcal{I}_1 &= \int_{t_{k-1}}^{t_k} \nabla \bfb (\bfY_{t_{k-1}, t} ( \widetilde{\bfY}_{t_\ell , t_{k-1}}^{\bfz} )) \left( \md^{\bfh} \bfY_{t_{k-1}, t} ( \widetilde{\bfY}_{t_\ell , t_{k-1}}^{\bfz} ) - \md^{\bfh} \widetilde{\bfY}_{t_\ell , t_{k-1}}^{\bfz} \right) \dd t , \\
				\mathcal{I}_2 &= \int_{t_{k-1}}^{t_k} \left( \nabla \bfb (\bfY_{t_{k-1}, t} ( \widetilde{\bfY}_{t_\ell , t_{k-1}}^{\bfz} )) - \nabla \bfb (\widetilde{\bfY}_{t_\ell , t_{k-1}}^{\bfz}) \right) \md^{\bfh} \widetilde{\bfY}_{t_\ell , t_{k-1}}^{\bfz} \, \dd t ,\\
				\mathcal{I}_3 &= \eta_k \left( \nabla \bfb (\widetilde{\bfY}_{t_\ell , t_{k-1}}^{\bfz}) - \nabla \tilde{b} (\widetilde{\bfY}_{t_\ell , t_{k-1}}^{\bfz}, \bm{\xi}_k) \right) \md^{\bfh} \widetilde{\bfY}_{t_\ell , t_{k-1}}^{\bfz}.
			\end{split}
		\end{equation}
		We estimate them one by one.
		
		For $\mathcal{I}_1$, 
		one has 
		\begin{equation*}
			\begin{split}
				&\pheq
				\md^{\bfh} \bfY_{t_{k-1}, t} ( \widetilde{\bfY}_{t_\ell, t_{k-1}}^{\bfz} )  - \md^{\bfh} \widetilde{\bfY}_{t_\ell, t_{k-1}}^{\bfz} \\
				&=  
				 \beta \int_{t_{k-1}}^t
				\begin{bmatrix}
					\mathbf{I}_d \\ \mathbf{0}
				\end{bmatrix}
				\bfh_s \, \dd s  
				+
				\int_{t_{k-1}}^t \nabla \bfb ( \bfY_{ t_{k-1} , u } ( \widetilde{\bfY}_{t_\ell, t_{k-1}}^{\bfz} ) ) \md^{\bfh} \widetilde{\bfY}_{t_\ell, t_{k-1}}^{\bfz} \, \dd u
				\\
				&\quad + \int_{t_{k-1}}^t \nabla \bfb ( \bfY_{ t_{k-1} , u } ( \widetilde{\bfY}_{t_\ell, t_{k-1}}^{\bfz} ) )\Big[  \md^{\bfh} \bfY_{ t_{k-1} , u } ( \widetilde{\bfY}_{t_\ell, t_{k-1}}^{\bfz} ) - \md^{\bfh} \widetilde{\bfY}_{t_\ell, t_{k-1}}^{\bfz} \Big]\, \dd u.
			\end{split}
		\end{equation*}
		 Together with $\opnorm{\nabla \bfb (\bfy)} \leqslant L + \gamma + 1$, the following holds,
		\begin{align*}
		\norm{ \md \bfY_{t_{k-1}, t} ( \widetilde{\bfY}_{t_\ell , t_{k-1}}^{\bfz} ) - \md \widetilde{\bfY}_{t_\ell , t_{k-1}}^{\bfz} }_{ \mathcal{H} }
		&\leqslant C \int_{t_{k-1}}^t \norm{ \md \bfY_{t_{k-1}, u} ( \widetilde{\bfY}_{t_\ell , t_{k-1}}^{\bfz} ) - \md \widetilde{\bfY}_{t_\ell , t_{k-1}}^{\bfz} }_{ \mathcal{H} } \dd u \\
		&\quad + C \sqrt{\eta_k} \left( \norm{\md \widetilde{\bfY}_{t_\ell , t_{k-1}}^{\bfz}}_{ \mathcal{H} } + 1 \right),
		\end{align*}
		for all $t \in [t_{k-1}, t_k]$. And Gr\"onwall's inequality derives
		\begin{equation} \label{eq:B36}
		\norm{ \md \bfY_{t_{k-1}, t} ( \widetilde{\bfY}_{t_\ell , t_{k-1}}^{\bfz} ) - \md \widetilde{\bfY}_{t_\ell , t_{k-1}}^{\bfz} }_{ \mathcal{H} }
		\leqslant C \sqrt{\eta_k} \left( \norm{\md \widetilde{\bfY}_{t_\ell , t_{k-1}}^{\bfz}}_{ \mathcal{H} } + 1 \right).
		\end{equation}
		which leads to
		\begin{equation}\label{eq:lemma3.12pr3}
			\abs{ \mathcal{I}_1 }^4 \leqslant C \eta_k^6 \left( \norm{\md \widetilde{\bfY}_{t_\ell , t_{k-1}}^{\bfz}}_{ \mathcal{H} }^4 + 1 \right) \norm{\bfh}_{\mathcal{H}}^4.
		\end{equation}
		
		For $\mathcal{I}_2$, we let $\bm{\Delta} = \bfY_{t_{k-1}, t} ( \widetilde{\bfY}_{t_\ell , t_{k-1}}^{\bfz} ) - \widetilde{\bfY}_{t_\ell , t_{k-1}}^{\bfz} $, then
		\[
			\mathcal{I}_2 = \md^{\bfh} \widetilde{\bfY}_{t_\ell , t_{k-1}}^{\bfz} \int_{t_{k-1}}^{t_k} \int_0^1 \nabla^2 \bfb \big( r \bm{\Delta} +  \widetilde{\bfY}_{t_\ell , t_{k-1}}^{\bfz}  \big) \cdot \bm{\Delta} \, \dd r \dd t .
		\]
		By $\opnorm{\nabla^2 \bfb} \leqslant B_2$ due to Assumption \ref{Assump.4}, we have 
		\begin{align}\label{eq:est I2}
			\abs{\mathcal{I}_2}^4 &\leqslant C \norm{ \md \widetilde{\bfY}_{t_\ell , t_{k-1}}^{\bfz} }_{\mathcal{H}}^4 \eta_k^3 \int_{t_{k-1}}^{t_k}  \big| \bfY_{t_{k-1}, t} ( \widetilde{\bfY}_{t_\ell , t_{k-1}}^{\bfz} ) - \widetilde{\bfY}_{t_\ell , t_{k-1}}^{\bfz}  \big|^4  \dd t \norm{\bfh}_{\mathcal{H}}^4 .
		\end{align}
		Besides, by Lemmas \ref{lem3-2} and \ref{lem3-4}, we have for any $t\in[t_{k-1}, t_k]$,
		\begin{equation*}
			\begin{split}
				\bbE \left[ \big| \bfY_{t_{k-1}, t} ( \widetilde{\bfY}_{t_\ell , t_{k-1}}^{\bfz} ) - \widetilde{\bfY}_{t_\ell , t_{k-1}}^{\bfz}  \big|^4 \right]
			&\leqslant C (t-t_{k-1})^2 \Big( \bbE \big[ \mathcal{V}( \widetilde{\bfY}_{t_\ell , t_{k-1}}^{\bfz} )^2 \big] + d^2 \Big) \\
			&\leqslant C \eta_k^2 \big( \mathcal{V}(\bfz)^2 + d^2 \big) .
			\end{split}
		\end{equation*}
		Then, by Markov property, it holds that
		\begin{equation}
			\label{eq:expectation}
			\bbE \abs{\mathcal{I}_2}^4 \leqslant C \eta_k^6 \big( \mathcal{V}(\bfz)^2 + d^2 \big)  \norm{\bfh}_{\mathcal{H}}^4 \bbE \norm{ \md \widetilde{\bfY}_{t_\ell , t_{k-1}}^{\bfz} }_{\mathcal{H}}^4 .
		\end{equation}
		
		For $\mathcal{I}_3$, since $\opnorm{\bfA} \leqslant \norm{\bfA}_{\rm HS}  \leqslant \sqrt{d} \opnorm{\bfA}$ for any matrix $\bfA$, it follows from Lemma \ref{lem-B-1} that
		{ \begin{align*}
			&\pheq
			\bbE_{\bm{\xi}} \opnorm{\nabla \bfb (\bfy) - \nabla \widetilde{\bfb} (\bfy, \bm{\xi})}^4
			= \bbE_{\bm{\xi}} \opnorm{\frac{1}{N} \sum_{r = 1}^N (\nabla^2 F (\bfx, \xi^r) - \nabla^2 f (\bfx))}^4 \\
			&\leqslant \frac{C}{N^2} \bbE_\xi \norm{\nabla^2 F (\bfx, \xi^r) - \nabla^2 f (\bfx)}_{\rm HS}^4
			\leqslant \frac{C d^2}{N^2}. 
		\end{align*}}
		Due $\bm{\xi}$ is independent of other random variables, this immediately leads to
		\begin{equation}
			\label{eq:est I3}
			{ \bbE \abs{\mathcal{I}_3}^4 \leqslant C \eta_k^4  \norm{ \bfh }_{\mathcal{H}}^4 \frac{ d^2}{N^2} \bbE \norm{\md \widetilde{\bfY}_{t_\ell , t_{k-1}}^{\bfz}}_{ \mathcal{H} }^4.}
		\end{equation}
		
		Since  $\bbE \lVert \md \widetilde{\bfY}_{t_\ell , t_{k-1}}^{\bfz} \rVert_{ \mathcal{H} }^4 \leqslant C d^2$ by Lemma \ref{lem3-10}, then
		Combining \eqref{eq:lemma3.12pr2}, \eqref{eq:lemma3.12pr3}, \eqref{eq:expectation} and \eqref{eq:est I3},  we get
		\begin{equation*}
			{ \bbE \norm{\md \bm{\Xi}_{t_\ell, t_k}}_{ \mathcal{H} }^4
			\leqslant C d^2 \eta_k^6 ( \mathcal{V} (\bfz)^2 + d^2 ) + C \frac{d^4 \eta_k^4}{N^2} . }
		\end{equation*}
		
		\noindent (ii)  According to \eqref{eq:lemma3.12pr1}, we have that
		\begin{equation}
			\label{eq:312II}
			\begin{split}
				\norm{\md \bm{\Xi}_{t_\ell, t_k}}_{ \mathcal{H} }
			&\leqslant \int_{t_{k-1}}^{t_k} \opnorm{ \nabla \bfb} \norm{\md \bfY_{t_{k-1}, t} ( \widetilde{\bfY}_{t_\ell , t_{k-1}}^{\bfz} )}_{ \mathcal{H} }  \dd t \\
			&\pheq 
			+
			 \eta_k \opnorm{\nabla \widetilde{\bfb} (\widetilde{\bfY}_{t_\ell , t_{k-1}}^{\bfz}, \bm{\xi}_k)} \norm{\md \widetilde{\bfY}_{t_\ell , t_{k-1}}^{\bfz}}_{ \mathcal{H} }
			\end{split}
		\end{equation}		
		\eqref{eq:B36} together with a triangle inequality, it holds
		\[
			\norm{\md \bfY_{t_{k-1}, t} ( \widetilde{\bfY}_{t_\ell , t_{k-1}}^{\bfz} )}_{ \mathcal{H} } - \norm{\md \widetilde{\bfY}_{t_\ell , t_{k-1}}^{\bfz}}_{ \mathcal{H} } 
			\leqslant C \sqrt{\eta_k} \left( \norm{\md \widetilde{\bfY}_{t_\ell , t_{k-1}}^{\bfz}}_{ \mathcal{H} } + 1 \right)
		\]
		 for all $t \in [t_{k-1}, t_k]$. Besides, by Lemma \ref{lem3-10},   $ \big\Vert{ \md \widetilde{\bfY}_{t_\ell , t_{k-1}}^{\bfz}} \big\Vert_{ \mathcal{H} } \leqslant C \sqrt{d}$ on the event $(E_{\ell, k}^1)^c$. Then, we have
		\begin{equation} \label{e:p1}
		\norm{\md \bfY_{t_{k-1}, t} ( \widetilde{\bfY}_{t_\ell , t_{k-1}}^{\bfz} )}_{ \mathcal{H} }
		\leqslant C \sqrt{d},
		 \quad \text{on the event } \quad (E_{\ell, k}^1)^c.
		\end{equation}		
		On the other hand, given $(E_{\ell, k}^2)^c$, i.e., $\sum_{i = \ell + 1}^k \eta_i ( \varphi (\bm{\xi}_i)^2 - \bbE [\varphi (\bm{\xi}_i)^2] ) \leqslant t_k - t_\ell$ with $\varphi (\bm{\xi}_i)$ defined by \eqref{e:Phi}, we have 
		\begin{equation} \label{e:p2}
		\eta_k \varphi (\bm{\xi}_k)^2
		\leqslant \sum_{i = \ell + 1}^k \eta_i \varphi (\bm{\xi}_i)^2
		\leqslant (t_k - t_\ell) (\bbE [\varphi (\bm{\xi}_k)^2] + 1)
		\leqslant C,
		\end{equation}
		where the last inequality is due to Assumption \ref{Assump.4}. 
		
		Combining $\opnorm{ \nabla b} \leqslant L + \gamma + 1$, \eqref{eq:312II}, \eqref{e:p1} and \eqref{e:p2}, we have
		\begin{equation*}
			\begin{split}
				\norm{\md \bm{\Xi}_{t_\ell, t_k}}_{ \mathcal{H} }
			&\leqslant C  \sqrt{d} \eta_k \left( \opnorm{\nabla \widetilde{\bfb} (\widetilde{\bfY}_{t_\ell , t_{k-1}}^{\bfz}, \bm{\xi}_k)} + 1 \right) \\
			&\leqslant C \sqrt{d} \eta_k ( \varphi (\bm{\xi}_k) + \gamma + 2 )
			\leqslant C \sqrt{d \eta_k}, \quad \text{on the event} \quad \big( E_{\ell, k}^1 \cup E_{\ell, k}^2 \big)^c.
			\end{split}
		\end{equation*}		
		The proof is complete.		
	\end{proof}
	
	\begin{proof}[\textbf{Proof of Lemma \ref{lem3-13}}]
		For any vector $\bfw \in \bbR^{2d}$ with $\abs{\bfw} = 1$, the relation \eqref{eq:Malliavin Matrix**} yields that
		\begin{equation} \label{eq:Malmatrix}
		\bfw^{\top} \bm{\Gamma} \big( \widetilde{\bfY}_{t_\ell , t_k}^{\bfz} \big) \bfw = \beta^2 \sum_{j = \ell + 1}^k \eta_j \abs{ \begin{bmatrix}
			\mathbf{I}_d & \mathbf{0}
			\end{bmatrix} \bfS_j \bfw }^2.
		\end{equation}
		Inspired by \cite{MR3256873}, we divide $\bfw$ and $\bfS_j$ into
		\begin{equation*}
		\bfw = \begin{bmatrix}
		\bfu \\
		\bfv
		\end{bmatrix}, \qquad
		\bfS_j =
		\begin{bmatrix}
		\bfA_j & \bfB_j \\
		\bfC_j & \bfD_j
		\end{bmatrix},
		\end{equation*}
		with vectors $\bfu$, $\bfv \in \bbR^d$ and matrices $\bfA_j$, $\bfB_j$, $\bfC_j$, $\bfD_j \in \bbR^{d\times d}$.
		
		We split the proof into two cases, one being $\abs{\bfu} > (t_k - t_\ell) / 4$ and the other $\abs{\bfu} \leqslant (t_k - t_\ell) / 4$. 
		
		
		\noindent (i) The case $\abs{\bfu} > (t_k - t_\ell) / 4$. Observe that
		\begin{equation} \label{e:case1}
		\begin{split}
			\abs{ \begin{bmatrix}
			\mathbf{I}_d & \mathbf{0}
			\end{bmatrix} \bfS_j \bfw }^2
			&\geqslant \frac{1}{2} \abs{ \begin{bmatrix}
			\mathbf{I}_d & \mathbf{0}
			\end{bmatrix}\bfw}^2 - \abs{ \begin{bmatrix}
			\mathbf{I}_d & \mathbf{0}
			\end{bmatrix} (\bfS_j - \mathbf{I}_{2d}) \bfw }^2 \\
			&\geqslant \frac{1}{2} \abs{\bfu}^2 - \opnorm{\bfS_j - \mathbf{I}_{2d} }^2.
		\end{split}
		\end{equation}
		 Assumption \ref{Assump.4} and the definition of $m$ yield that
		\begin{equation} \label{e:Tk-Tm}
			( t_k - t_m ) ( \bbE \varphi(\bm{\xi}_k) + \gamma + 2 ) \leqslant \frac{t_k - t_\ell}{10}  .
		\end{equation}
		Given $(E_{m,k}^1)^c$, i.e., $\sum_{i = m+1}^k \eta_i ( \varphi(\bm{\xi}_i) - \bbE \varphi(\bm{\xi}_i) )\leqslant t_k - t_m$, we have
		\begin{equation}
			\label{e:EmkC}
			\sum_{i = m+1}^k \eta_i (\varphi (\bm{\xi}_i) + \gamma + 1)
			\leqslant 
			(t_k - t_m) (\bbE \varphi (\bm{\xi}_k) + \gamma + 2) .
		\end{equation}
		Combining \eqref{e:Tk-Tm}, \eqref{e:EmkC} and Lemma \ref{lemma:operator norm of S_j}, we have that for $j = m, \dots, k$, 
		\begin{equation*}
		\opnorm{\bfS_j - \mathbf{I}_{2d} }
		\leqslant \rme^{ \sum_{i = m+1}^k \eta_i (\varphi (\bm{\xi}_i) + \gamma + 1) } - 1
		\leqslant \rme^{(t_k - t_\ell) / 10} - 1
		\leqslant \frac{t_k - t_\ell}{8},
		\end{equation*}
		where the last inequality follows from $\rme^{x / 10} - 1 \leqslant x / 8$ for $x \in (0, 1 /10)$. Substituting this into \eqref{e:case1} and $\abs{\bfu} > (t_k - t_\ell)/4$ imply 
		\[
			\abs{ \begin{bmatrix}
					\mathbf{I}_d & \mathbf{0}
				\end{bmatrix} \bfS_j \bfw }^2
			\geqslant
			\frac{(t_k - t_{\ell})^2}{64}  .
		\] 
		Then, due to $t_k - t_{m-1} \geqslant (t_k - t_\ell) / [10(\gamma + B_1 + 2)]$ by the definition of $m$, we obtain 
		\begin{align} \label{eq:est1}
		\bfw^{\top} \bm{\Gamma} \big( \widetilde{\bfY}_{t_\ell , t_k}^{\bfz} \big) \bfw
		\geqslant \beta^2 \sum_{j = m}^k \eta_j \abs{ \begin{bmatrix}
			\mathbf{I}_d & \mathbf{0}
			\end{bmatrix} \bfS_j \bfw }^2
		\geqslant C (t_k - t_\ell)^3, \ \text{on}\ ( E_{m,k}^1 )^c.
		\end{align}
		
		\noindent (ii) The case $\abs{\bfu} \leqslant (t_k - t_\ell) / 4$. In this case we clearly have $\abs{\bfv} = ({1 - \abs{\bfu}^2})^{1/2} \geqslant 3 / 4$. Notice that $\bfS_{i-1} = \bfE_i \bfS_i$ implies
		\begin{equation*}
		\begin{bmatrix}
		\bfA_{i-1} & \bfB_{i-1}
		\end{bmatrix}
		=
		\begin{bmatrix}
		(1-\gamma \eta_i) \mathbf{I}_d & \eta_i \mathbf{I}_d
		\end{bmatrix}
		\begin{bmatrix}
		\bfA_i & \bfB_i \\
		\bfC_i & \bfD_i
		\end{bmatrix},
		\end{equation*}
		i.e.
		\begin{equation*}
		\begin{bmatrix}
		\bfA_{i-1} & \bfB_{i-1}
		\end{bmatrix}
		-
		\begin{bmatrix}
		\bfA_i & \bfB_i
		\end{bmatrix}
		= -\gamma{\eta_i} 
		\begin{bmatrix}
		\bfA_i & \bfB_i
		\end{bmatrix}
		+\eta_i
		\begin{bmatrix}
		\bfC_i & \bfD_i
		\end{bmatrix}.
		\end{equation*}
		Then, summing this over $i$, we can obtain that
		\begin{equation*}
		\begin{bmatrix}
		\bfA_j & \bfB_j
		\end{bmatrix}
		-
		\begin{bmatrix}
		\bfA_{k} & \bfB_{k}
		\end{bmatrix}
		= -\gamma \sum_{i=j+1}^k \eta_i 
		\begin{bmatrix}
		\bfA_i & \bfB_i
		\end{bmatrix}
		+ \sum_{i=j+1}^k \eta_i 
		\begin{bmatrix}
		\bfC_i & \bfD_i
		\end{bmatrix}.
		\end{equation*}
		Due to $\bfA_k = \mathbf{I}_d$ and $\bfB_k = \mathbf{0}$, we get the following relation, 
		\begin{equation}
		\label{eq:relationship of [A_j,B_j]}
		\bfL_j
		:=
		\begin{bmatrix}
		\bfA_j & \bfB_j
		\end{bmatrix}
		+ \gamma 
		\sum_{i=j+1}^k 
		\eta_i 
		\begin{bmatrix}
		\bfA_i & \bfB_i
		\end{bmatrix}
		=
		\sum_{i=j+1}^k 
		\eta_i 
		\begin{bmatrix}
		\bfC_i & \bfD_i
		\end{bmatrix}
		+
		\begin{bmatrix}
		\mathbf{I}_d & \mathbf{0}
		\end{bmatrix}
		=:
		\bfR_j,
		\end{equation}
		where we denote the matrices on the left-hand side and right-hand side by $\bfL_j$ and $\bfR_j$ respectively. 
		
		For $\bfL_j$, observe that $[\bfA_j \quad \bfB_j] = [\mathbf{I}_d \quad \zero] \bfS_j$, then  we have
		\begin{align*}
		\abs{\bfL_j\bfw}^2
		&\leqslant 2 \abs{ \begin{bmatrix}
			\bfA_j & \bfB_j
			\end{bmatrix} \bfw }^2
		+ 2 \gamma^2 \left( \sum_{i = \ell + 1}^k \eta_i \abs{ \begin{bmatrix}
			\bfA_i & \bfB_i
			\end{bmatrix} \bfw } \right)^2 \\
		&\leqslant 2 \abs{ \begin{bmatrix}
			\mathbf{I}_d & \mathbf{0}
			\end{bmatrix} \bfS_j \bfw }^2
		+ 2 \gamma^2 (t_k - t_\ell) \sum_{i = \ell + 1}^k \eta_i \abs{ \begin{bmatrix}
			\mathbf{I}_d & \mathbf{0}
			\end{bmatrix} \bfS_i \bfw }^2,
		\end{align*}
		which immediately implies that
		\begin{equation}
		\label{eq:Upper bound of Lw}
		\sum_{j = \ell + 1}^k \eta_j \abs{\bfL_j\bfw}^2 \leqslant 2 [1 + \gamma^2 (t_k - t_\ell)^2] \sum_{j = \ell + 1}^k \eta_j 
		\abs{ \begin{bmatrix}
			\mathbf{I}_d & \mathbf{0}
			\end{bmatrix} \bfS_j \bfw }^2.
		\end{equation}
		
		For $\bfR_j$, observe that $[\bfC_j \quad \bfD_j] = [\zero \quad \mathbf{I}_d] \bfS_j$, then
		\begin{align*}
		\bfR_j \bfw
		&= \sum_{i=j+1}^k \eta_i
		\begin{bmatrix}
		\mathbf{0} & \mathbf{I}_d
		\end{bmatrix} \bfw
		+ \sum_{i=j+1}^k \eta_i \left(
		\begin{bmatrix}
		\bfC_i & \bfD_i
		\end{bmatrix}
		-
		\begin{bmatrix}
		\mathbf{0} & \mathbf{I}_d
		\end{bmatrix} \right) \bfw
		+
		\begin{bmatrix}
		\mathbf{I}_d & \mathbf{0}
		\end{bmatrix} \bfw \\
		&= (t_k - t_j) \bfv
		+ \sum_{i=j+1}^k \eta_i \left(
		\begin{bmatrix}
		\bfC_i & \bfD_i
		\end{bmatrix}
		-
		\begin{bmatrix}
		\mathbf{0} & \mathbf{I}_d
		\end{bmatrix} \right) \bfw
		+ \bfu.
		\end{align*}
		Given $(E_{\ell, k}^1)^c$, together with Lemma \ref{lemma:operator norm of S_j} implies
		\begin{align*}
		\opnorm{\begin{bmatrix} \bfC_i & \bfD_i \end{bmatrix} - \begin{bmatrix} \mathbf{0} & \mathbf{I}_d \end{bmatrix}}
		&\leqslant \opnorm{\bfS_i-\mathbf{I}_{2d}}
		\leqslant \rme^{ \sum_{i = \ell + 1}^k \eta_i (\varphi (\bm{\xi}_i) + \gamma + 1) } - 1 \\
		&\leqslant \rme^{(t_k - t_\ell) (\bbE \varphi (\bm{\xi}_k) + \gamma + 2)} - 1
		\leqslant \rme^{1/5} - 1
		\leqslant \frac{1}{4}.
		\end{align*}
		Together with $\abs{\bfu} \leqslant (t_k - t_\ell) / 4$ and $\abs{\bfv} \geqslant 3 / 4$, we have
		\begin{align*}
		\abs{\bfR_j \bfw}
		&\geqslant (t_k - t_j) \abs{\bfv}
		- \sum_{i=j+1}^k \eta_i \abs{ \left(
			\begin{bmatrix}
			\bfC_i & \bfD_i
			\end{bmatrix}
			-
			\begin{bmatrix}
			\mathbf{0} & \mathbf{I}_d
			\end{bmatrix} \right) \bfw }
		- \abs{\bfu} \\
		&\geqslant \frac{3}{4} (t_k - t_j) - \sum_{i=j+1}^k \eta_i \opnorm{\begin{bmatrix} \bfC_i & \bfD_i \end{bmatrix} - \begin{bmatrix} \mathbf{0} & \mathbf{I}_d \end{bmatrix}} - \frac{1}{4} (t_k - t_\ell) \\
		&\geqslant \frac{1}{2} (t_k - t_j) - \frac{1}{4} (t_k - t_\ell).
		\end{align*}
		Take $m'$ such that $t_k - t_{m'} \geqslant 2 (t_k - t_\ell) / 3$ and $t_{m'} - t_\ell \geqslant (t_k - t_\ell) / 4$, then
		\begin{equation}
		\label{eq:Lower bound of Rw}
		\begin{split}
			\sum_{j = \ell + 1}^k \eta_j \abs{\bfR_j\bfw}^2
			&\geqslant \sum_{j = \ell + 1}^{m'} \eta_j \abs{\bfR_j\bfw}^2 \\
			&\geqslant \left[ \frac{1}{2}(t_k - t_{m^\prime}) - \frac14(t_k-t_\ell) \right] ^2 \sum_{j = \ell + 1}^{m'} \eta_j
			\geqslant \frac{(t_k - t_\ell)^3}{576}.
		\end{split}
		\end{equation}
		Combining \eqref{eq:relationship of [A_j,B_j]}, \eqref{eq:Upper bound of Lw} and \eqref{eq:Lower bound of Rw} leads to
		\[
			2 [1 + \gamma^2 (t_k - t_\ell)^2] \sum_{j = \ell + 1}^k \eta_j 
			\abs{ \begin{bmatrix}
					\mathbf{I}_d & \mathbf{0}
				\end{bmatrix} \bfS_j \bfw }^2
			\geqslant 
			\frac{(t_k - t_\ell)^3}{576} .
		\]
		This, together with \eqref{eq:Malmatrix} and $t_k - t_\ell \leqslant 1/[5(B_1 + \gamma + 2)]$, immediately implies
		\begin{equation} \label{eq:est2}
		\bfw^{\top} \bm{\Gamma} \big( \widetilde{\bfY}_{t_\ell , t_k}^{\bfz} \big) \bfw
		= \beta^2 \sum_{j=1}^k \eta_j \abs{ \begin{bmatrix}
			\mathbf{I}_d & \mathbf{0}
			\end{bmatrix} \bfS_j \bfw }^2
		\geqslant C (t_k - t_\ell)^3 , \quad \text{on} \ (E_{\ell, k}^1)^c.
		\end{equation}
		
		Consequently, \eqref{eq:est1} and \eqref{eq:est2} yield the desired,  
		\begin{equation*}
		\lambda_{\min} \left( \bm{\Gamma} \big( \widetilde{\bfY}_{t_\ell , t_k}^{\bfz} \big) \right)
		= \min_{\abs{\bfw} = 1} \bfw^{\top} \bm{\Gamma} \big( \widetilde{\bfY}_{t_\ell , t_k}^{\bfz} \big) \bfw
		\geqslant C (t_k - t_\ell)^3, \quad \text{on}\ \big( E_{\ell, k}^1 \cup E_{m,k}^1 \big)^c.
		\end{equation*}
		The proof is complete.
	\end{proof}
	
\end{document}